\documentclass{article}

\PassOptionsToPackage{numbers, compress}{natbib}

 \usepackage[main,final]{neurips_2026}
\usepackage{wrapfig}
\usepackage{graphicx}
\usepackage{subcaption}
\usepackage{booktabs} 
\usepackage{amsmath}
\usepackage{amssymb}
\usepackage{mathtools}
\usepackage{needspace}
\usepackage{amsthm}
\usepackage{tikz}
\usepackage{color}
\usepackage{tabularray}
\usepackage{caption}
\usepackage{bbm}
\usepackage{filecontents}
\usepackage{pifont}
\usepackage{wrapfig}
\usepackage{float}
\usepackage{multirow}
\usepackage{tabularx}
\usepackage{makecell}
\usepackage{balance}
\usepackage{tabu}
\usepackage[flushleft]{threeparttable}

\newcommand{\hmark}{{\ding{51}}\textsuperscript{\kern-0.52em\scriptsize\ding{55}}}
\newcommand{\system}{\textsc{praetorian}}

\theoremstyle{plain}
\newtheorem{theorem}{Theorem}[section]
\newtheorem{proposition}[theorem]{Proposition}
\newtheorem{lemma}[theorem]{Lemma}
\newtheorem{corollary}[theorem]{Corollary}
\theoremstyle{definition}

\newtheorem{assumption}[theorem]{Assumption}
\theoremstyle{remark}

\usepackage[textsize=tiny]{todonotes}
\usepackage[utf8]{inputenc} 
\usepackage[T1]{fontenc}    
\usepackage[hidelinks]{hyperref} 
\usepackage[capitalize,noabbrev]{cleveref}
\usepackage{url}            
\usepackage{amsfonts}       
\usepackage{nicefrac}       
\usepackage{microtype}      
\usepackage{xcolor}         
\usepackage{colortbl}
\definecolor{softblue}{RGB}{225,238,248}
\title{Trapping Attacker in Dilemma: Examining Internal Correlations and External Influences of Trigger for Defending GNN Backdoors}

\author{%
  Fan Yang \\
  The Chinese University of Hong Kong\\
  \texttt{yf.020@ie.cuhk.edu.hk} 
  \And
    Binyan Xu \\
  The Chinese University of Hong Kong\\
  \texttt{binyxu@ie.cuhk.edu.hk} 
    \And
    Di Tang \thanks{Corresponding Author}\\
  Sun Yat-Sen University\\
  \texttt{tangd9@mail.sysu.edu.cn} 
    \And
   Kehuan Zhang \thanks{Corresponding Author} \\
  The Chinese University of Hong Kong\\
  \texttt{khzhang@ie.cuhk.edu.hk} \\
}

\begin{document}

\maketitle

\begin{abstract}

GNNs have become a standard tool for learning on relational data, yet they remain highly vulnerable to backdoor attacks. Prior defenses often depend on inspecting specific subgraph patterns or node features, and thus can be circumvented by adaptive attackers. We propose \system{}, a new defense that targets intrinsic requirements of effective GNN backdoors rather than surface-level cues. Our key observation is that flipping a victim node's prediction requires substantial influence on the victim: attackers tend to either inject many trigger nodes or rely on a small set of highly influential ones. Building on this observation, \system{} (i) analyzes internal correlations within potential trigger subgraphs to detect abnormally large injected structures, and (ii) quantifies external node influence to identify triggers with disproportionate impact. Across our evaluations, \system{} reduces the average attack success rate (ASR) to 0.55\% with only a 0.62 percentage-point drop in clean accuracy (CA), whereas state-of-the-art defenses still yield an average ASR of $\geq$20\% and a CA drop of at least 3 percentage points under the same conditions. Moreover, \system{} remains effective against a range of adaptive attacks: adversaries inject many trigger nodes to achieve high ASR ($\geq 80\%$) under random disruption on Cora, with conspicuous changes to the graph and usually lower CA. The tested clean-label variants reach at most 18.1\% ASR. Overall, \system{} constrains attackers to an unfavorable trade-off between efficacy and detectability.

\end{abstract}
\section{Introduction}
\label{sec:introduction}
\vspace{-5pt}

Graph Neural Networks (GNNs) have become a standard tool for learning on relational data, delivering strong performance on node classification~\cite{gcn,gat,graphsage}, link prediction~\cite{linkprediction2018}, and graph classification~\cite{DGCGN,gin}, and enabling deployments in high-impact applications such as social networks~\cite{fan2019graphneuralnetworkssocial}, recommendation~\cite{recommendatationgnn}, and biological analysis~\cite{biologicalgnn,moleculargnn}. However, recent work shows that GNNs are vulnerable to \emph{backdoor attacks}~\cite{DPGBA,UGBA,SBA,GTA}, where an adversary attaches a trigger subgraph to victim nodes to induce targeted misclassification, raising serious concerns in safety-critical scenarios (e.g., healthcare)~\cite{gnn-healthcare,disease-diagnosis,finance-gnn,healthycate-gnn}.

\textbf{Limitations of existing defenses.}
To mitigate this threat, several defenses have been suggested. Prior defenses largely fall into two categories.
\emph{Robust learning} methods reduce sensitivity to malicious perturbations via adversarial training~\cite{gosch2023adversarialtraininggraphneural,li2022spectraladversarialtrainingrobust} or noise injection~\cite{RandomSmooth,qian2023robusttraininggraphneural,dai2021nrgnnlearninglabelnoiseresistant}, but they incur substantial overhead~\cite{Adversarial-Attack-and-Defense} and often degrade clean accuracy~\cite{GNNGuard,RobustGCN}; moreover, they can be ineffective against complex or adaptive triggers~\cite{homophilynecessitygraphneural}.
\emph{Trigger identification} methods attempt to flag suspicious nodes.
For example, DPGBA~\cite{DPGBA} treats trigger nodes (e.g., in GTA~\cite{GTA}) as outliers and isolates them with outlier detection.
However, such defenses typically rely on non-fundamental trigger signatures and can be bypassed by \emph{in-distribution} triggers crafted via adversarial learning, as demonstrated by DPGBA itself.
A different method, RIGBD~\cite{RIGBD}, drops edges to find nodes with large prediction impact, but it often fails to generalize across trigger designs and attack strengths.
Overall, existing defenses remain brittle under adaptive attacks.

\textbf{Our method:}
Rather than chasing trigger-specific signatures, we propose \system{}, a defense built on a \emph{necessary condition} for effective GNN backdoors.
To reliably flip a victim’s prediction, an attacker must \emph{inject influence} into the victim’s local computation---i.e., the trigger must contribute enough message-passing signal to override benign evidence.
This requirement is inherent to GNN inference and broadly applies across different trigger encodings.

\begin{wrapfigure}{r}{0.52\linewidth}
    \centering
    \vspace{-10pt}
    \includegraphics[width=1\linewidth]{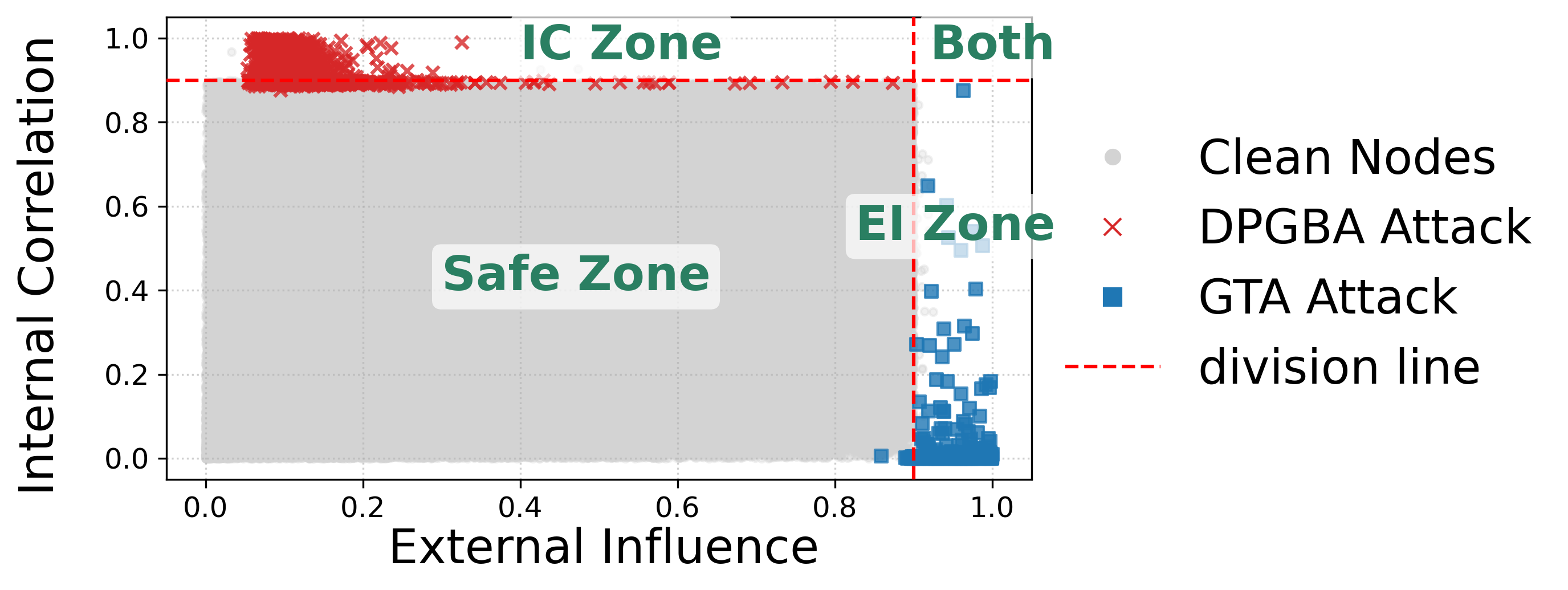}
    \vspace{-18pt}
    \caption{Attackers trapped in dilemma: detected by either high internal correlation (DPGBA) or external influence (GTA).}
    \vspace{-10pt}
    \label{fig:combined_histograms}
\end{wrapfigure}

Crucially, our theory (Appendix~\ref{sec:theoretical_framework}) makes this influence requirement \emph{operational}. We prove a two-view decomposition: (i) \emph{synergistic influence}, captured by \emph{internal correlation} (IC), and (ii) \emph{per-node influence}, summed in \emph{external influence} (EI).
Attackers can spread influence across many low-impact trigger nodes, concentrate it on a few nodes, or use a strategic combination of both. Therefore, any high-ASR backdoor satisfying the non-trivial influence premise must exhibit non-negligible influence in at least one of the two theoretical components, EI or IC. This motivates PRAETORIAN to jointly audit practical proxies for external influence and internal correlation (Appendix~\ref{apd:structured_additions}; Section~\ref{sec:node_level}). Figure~\ref{fig:combined_histograms} shows how these signals separate DPGBA and GTA triggers from benign nodes on OGB-arxiv (Appendix~\ref{apd:motivation}).

Specifically, we propose \system{}, which (i) fuses two views into a stable metric, (ii) filters trigger and victim nodes from ranked scores, and (iii) neutralizes backdoors while preserving clean accuracy (CA).
Across multiple datasets and attacks, \system{} detects victim nodes with 98.52\% precision / 94.58\% recall and trigger nodes with 98.64\% precision / 91.20\% recall, and reduces average ASR to 0.55\% with only a 0.62 percentage-point CA drop.
In contrast, prior defenses yield at least 35\% mean ASR with a CA drop of at least 3 percentage points under identical settings.
We further evaluate adaptive attackers: achieving high ASR ($\geq 80\%$) under random disruption on Cora involves many noisy trigger nodes and usually lower CA, while the tested clean-label variants reach at most 18.1\% ASR.

\noindent\textbf{Contributions.} Our contributions are threefold: \\
\noindent\textbf{(1) A new detection paradigm for GNN backdoor defense.} We introduce a two-view framework that jointly examines internal correlations and external influences, providing a new perspective for identifying node-level backdoor triggers beyond surface-level structural or feature cues.

\vspace{-5pt}
\noindent \textbf{(2) New defense against GNN backdoors.} We introduce \system{}\footnote{Code will be released at \url{https://github.com/Joney-Yf/Graph_backdoor-Defense-Praetorian}.}, a new defense based on our proposed new paradigm. \system{} not only effectively identifies trigger and victim nodes but also neutralizes backdoors while maintaining performance on clean data.

\vspace{-5pt} 
\noindent\textbf{(3) Superior Effectiveness and Robustness.} We evaluate \system{} on 4 datasets against 9 backdoor attacks, showcasing significant improvements compared to existing defenses. In these experiments, \system{} forces attackers into a dilemma: either accept a low ASR or incur conspicuous structural distortions, often accompanied by lower CA.

\vspace{-8pt}
\section{Background \& Related Work}
\vspace{-5pt}

\subsection{Graph Backdoor Attacks}
\vspace{-5pt}
As GNNs are increasingly deployed in safety- and security-critical settings, protecting them from adversarial manipulation is essential. A prominent threat is the \emph{backdoor attack}, where an adversary implants hidden malicious behavior during training: the compromised model behaves normally on clean inputs but produces attacker-chosen outputs when a specific \emph{trigger} is present~\cite{li2022backdoorlearningsurvey}. This differs from test-time adversarial attacks~\cite{adversarial-survey}, which perturb inputs at inference without modifying the model.

In graph learning, backdoors can target graph-, edge-, and node-level tasks. \textbf{1. Whole-graph attacks}~\cite{explainabilitybased,SBA,GTA} manipulate global structures or features so that any test graph containing a trigger subgraph is predicted as a target label. \textbf{2. Edge-level attacks}~\cite{link-backdoor,dyn-link} insert or alter edges to induce targeted link predictions by embedding a hidden edge-pattern trigger. \textbf{3. Node-level attacks} attach a trigger subgraph to selected victim nodes to misclassify them into a target label. Early work such as SBA~\cite{SBA} proposes universal subgraph triggers but can be less effective, while later attacks improve stealth and success through adaptive perturbations (GTA~\cite{GTA}), budgeted node selection and feature matching (UGBA~\cite{UGBA}), and adversarial learning to craft \emph{in-distribution} triggers (DPGBA~\cite{DPGBA}).

\textit{Node-level backdoors} are particularly concerning in real deployments, such as social networks, because localized trigger attachments can preserve global topology and keep distributions largely benign. \textit{Accordingly, our goal is to defend against node-level backdoors by detecting and neutralizing malicious trigger subgraphs while preserving clean performance.}

\vspace{-5pt}
\subsection{Difficulty in Graph Backdoor Defenses}
\vspace{-5pt}
Existing GNN backdoor defenses mainly rely on \textit{robust training} or \textit{poison node detection}.

\textit{Robust training.}
Robust training methods~\cite{RobustGCN,gosch2023adversarialtraininggraphneural,nt2019learninggraphneuralnetworks,survey-adversarial-learning-graph,robust-graph_learning,li2022spectraladversarialtrainingrobust} reduce sensitivity to malicious perturbations by simulating attack scenarios during optimization or injecting noise into the input graph. In practice, they often require generating many adversarial variants, frequently through iterative procedures, leading to substantial computational overhead; they can also degrade clean accuracy~\cite{GNNGuard} due to the robustness--accuracy trade-off.

\textit{Poison node detection.}
Detection-based defenses~\cite{DPGBA,RIGBD,UGBA} identify trigger nodes by searching for structural or feature anomalies. This is challenging on graphs because structural heterogeneity and local variation can lead to false positives on benign nodes or false negatives on subtle triggers. Moreover, detectors based on heuristic signatures or fixed anomaly patterns can be evaded by adaptive attackers that craft \emph{in-distribution} triggers mimicking benign neighborhoods~\cite{UGBA,DPGBA}.

Finally, several image-domain backdoor defenses have been adapted to graphs~\cite{RIGBD,abl}, but such transfer is non-trivial and often ineffective because relational dependencies and message passing create attack and defense behaviors that differ from those in images (Appendix~\ref{apd:challenges}).
\vspace{-8pt}
\section{Preliminary}
\vspace{-8pt}
\subsection{Notations}
\vspace{-5pt}
A summary of symbols is provided in Appendix~\ref{apd:symbol}.

\noindent\textbf{Attributed graph.}
We consider an attributed graph $\mathcal{G}=(\mathcal{V},\mathcal{E},\mathcal{X})$ with $|\mathcal{V}|=N$ nodes, edges $\mathcal{E}\subseteq \mathcal{V}\times\mathcal{V}$, and node features $\mathcal{X}=\{x_i\}_{i=1}^N$.
Let $A\in\mathbb{R}^{N\times N}$ be the adjacency matrix where $A_{ij}=1$ iff $(v_i,v_j)\in\mathcal{E}$ (else $0$).
We denote the neighbors of $v_i$ by $\mathcal{N}(i)$.

\noindent\textbf{Training graph and victim nodes.}
We are given a backdoored training graph $\mathcal{G}_T=(\mathcal{V}_T,\mathcal{E}_T,\mathcal{X}_T)$.
The training nodes are partitioned into a clean set $\mathcal{V}_C$ and a victim set $\mathcal{V}_B$ with $\mathcal{V}_C\cap\mathcal{V}_B=\varnothing$; nodes in $\mathcal{V}_C$ have clean labels $y_i$, while nodes in $\mathcal{V}_B$ are assigned the attacker-chosen target label $y_t$.
We call $|\mathcal{V}_B|$ the \emph{victim size (VS)}.
Nodes in $\mathcal{V}_T\setminus(\mathcal{V}_C\cup\mathcal{V}_B)$ serve as unlabeled neighbors during training but are not evaluated for classification.

\noindent\textbf{Triggers nodes.}
To mount a node-level backdoor, the attacker attaches to each $v_i\in\mathcal{V}_B$ a trigger subgraph $\mathcal{G}^{\mathrm{Tri}}_i=(\mathcal{V}^{\mathrm{Tri}}_i,\mathcal{E}^{\mathrm{Tri}}_i,\mathcal{X}^{\mathrm{Tri}}_i)$ via an attachment operator $a(v_i,\mathcal{G}^{\mathrm{Tri}}_i)$ (details in Appendix~\ref{apd:attach}).
A trigger may contain multiple connected components (or isolated nodes). We define the \emph{trigger size (TS)} as $|\mathcal{V}^{\mathrm{Tri}}_i|$.
We denote the set of \emph{poisoned nodes} as $\mathcal{V}_P=\mathcal{V}_B\cup\big(\cup_i \mathcal{V}^{\mathrm{Tri}}_i\big)$, which are the abnormal nodes our defense aims to identify.

\noindent\textbf{Inference graph.}
At test time, we evaluate on an unseen graph $\mathcal{G}_U=(\mathcal{V}_U,\mathcal{E}_U,\mathcal{X}_U)$ with $\mathcal{V}_U=\mathcal{V}_{UC}\cup\mathcal{V}_{UB}$, where $\mathcal{V}_{UC}$ are clean nodes and $\mathcal{V}_{UB}$ are victim nodes that would be misclassified when a trigger is attached.
We assume an inductive split: $\mathcal{V}_U\cap\mathcal{V}_T=\varnothing$.

\vspace{-5pt}
\subsection{Threat Model}
\vspace{-5pt}
\label{sec:threat_model}
\noindent\textbf{Attacker.}
We consider node-level backdoor attackers who poison the training graph by attaching trigger subgraphs to a subset of victim nodes $\mathcal{V}_B$ and assigning them an attacker-chosen target label $y_t$.
Specifically, for each $v_i\in\mathcal{V}_B$, the attacker attaches a trigger $\mathcal{G}^{\mathrm{Tri}}_i$ via an attachment operator $a(v_i,\mathcal{G}^{\mathrm{Tri}}_i)$ (Appendix~\ref{apd:attach}).
At inference, the attacker applies the same attachment operator $a(\cdot)$ to victim nodes in the unseen test graph $\mathcal{G}_U$.

\noindent\textbf{Defender's Knowledge.}
The defender observes the poisoned training graph $\mathcal{G}_T$ and trains a node classifier, \textbf{but does not know the victim set $\mathcal{V}_B$, the trigger nodes, or the target label $y_t$.}
At test time, the defender is evaluated on an unseen graph $\mathcal{G}_U$ that may contain triggered victims.

\noindent\textbf{Defender's Goals.}
Given a poisoned training graph $\mathcal{G}_T$, the defender aims to (i) identify the target label $y_t$, the victim set $\mathcal{V}_B$, and the trigger node set $\mathcal{V}^{\mathrm{Tri}}$, and (ii) neutralize the backdoor so that the resulting classifier $f$ achieves low ASR on an unseen backdoored graph $\mathcal{G}_U$ while preserving high clean accuracy (CA). \textit{In our pipeline, $y_t$ is inferred from the observed labels of detected victim nodes.}

The goal on the test graph is to predict clean labels on $\mathcal{V}_{UC}$ while suppressing target-label predictions on nodes in $\mathcal{V}_{UB}$:
$\min_{f}\sum_{v_i \in \mathcal{V}_{UC}}\ell(f(v_i),y_i)+\sum_{v_j \in \mathcal{V}_{UB}}\ell_{\mathrm{anti}}(f(v_j),y_t)$,
where $\ell$ is the standard classification loss and $\ell_{\mathrm{anti}}$ penalizes predicting the target label. We evaluate $f$ on $\mathcal{G}_U$ using CA and ASR.

\vspace{-5pt}
\section{Methodology}
\vspace{-5pt}

\begin{figure*}[h]
    \centering
    \includegraphics[width=\textwidth]{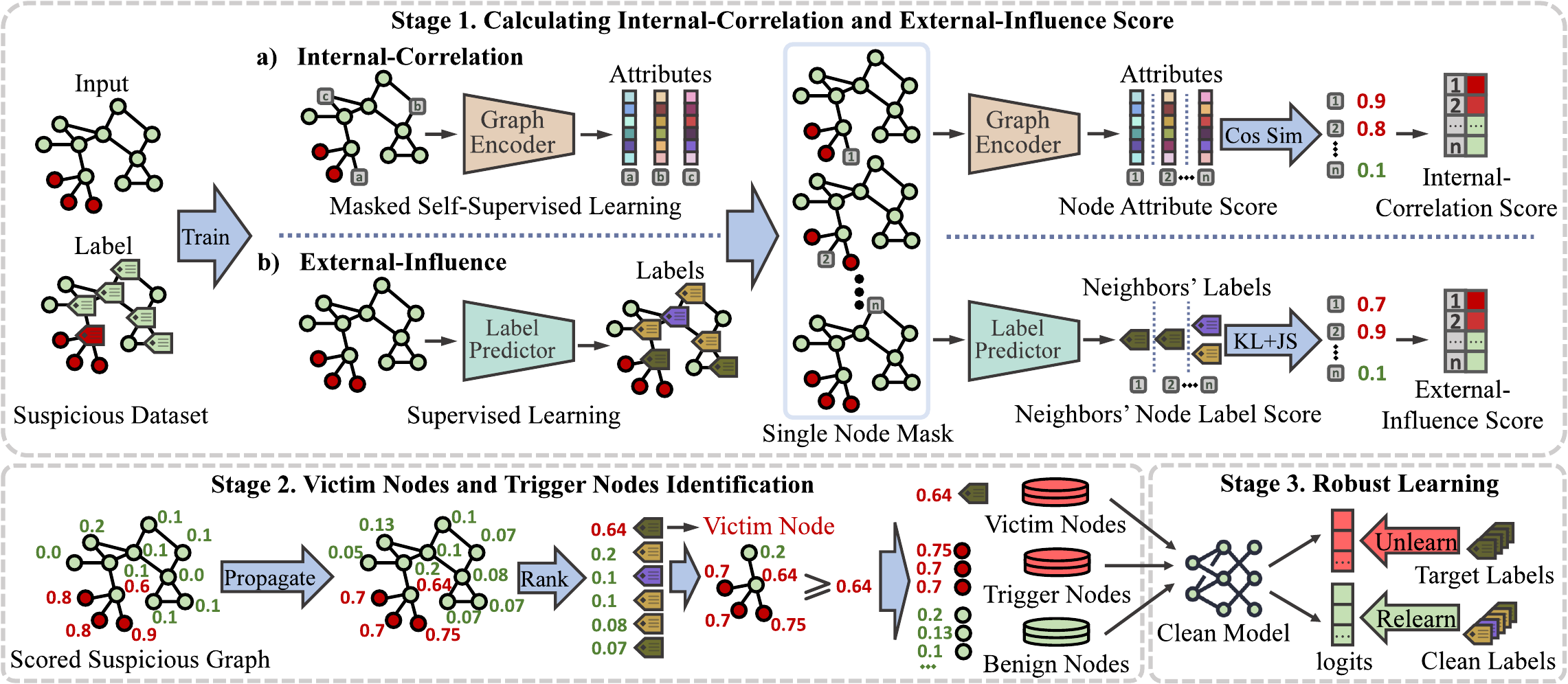}
    \vspace{-10pt}
    \caption{Overview of \system}
    \vspace{-6pt}
    \label{fig:overview}
\end{figure*}

Section~\ref{sec:introduction} motivates two complementary trigger signals: (i) trigger subgraphs exhibit unusually strong \emph{internal correlation}, and (ii) trigger nodes have outsized \emph{external influence}, as removing them induces larger prediction shifts in their neighborhoods than removing benign nodes.

Building on this, we propose \system{}, a defense framework with three objectives: (i) quantify both signals reliably, (ii) fuse them to localize victim and trigger nodes, and (iii) neutralize backdoor effects at test time while preserving clean accuracy.

Concretely, \system{} computes per-node internal-correlation and external-influence scores via masking-based probes, and aggregates them into a unified suspicion score. Using this score, we identify candidate victim/trigger nodes and then perform backdoor neutralization by (a) filtering detected trigger nodes and (b) suppressing target-label behavior on the remaining suspicious nodes.
Figure~\ref{fig:overview} summarizes the pipeline in three stages: \textit{Calculating Internal-Correlation and External-Influence Scores}, \textit{Victim and Trigger Node Identification}, and \textit{Robust Learning Based on Identified Nodes}. Next, we detail each one.

\vspace{-5pt}
\subsection{Calculating Two Scores}
\vspace{-5pt}

\noindent\textbf{Internal-correlation score.}
Given the poisoned training graph $\mathcal{G}_T=(\mathcal{V}_T,\mathcal{E}_T,\mathcal{X}_T)$, we quantify a node’s \emph{internal correlation} using a masked graph autoencoder~\cite{hou2022graphmae}. Intuitively, nodes that participate in highly unique (easily-memorized) trigger patterns are easier to reconstruct from their neighborhoods.

We sample a mask set $\widehat{\mathcal{V}}\subset\mathcal{V}_T$ and replace their features with a learnable mask token $x_M\in\mathbb{R}^d$: \(\widehat{x}_i = 
\begin{cases} 
x_i & v_i \notin \widehat{\mathcal{V}} \\ 
x_M & v_i \in \widehat{\mathcal{V}}
\end{cases}\).

An encoder--decoder autoencoder maps the partially observed graph
$\widehat{\mathcal{G}}_T=(\mathcal{V}_T,\mathcal{E}_T,\widehat{\mathcal{X}}_T)$
to representations $H$ and reconstructs features $\mathcal{X}'$:
\begin{equation}
H=f_E(\widehat{\mathcal{G}}_T),\qquad \mathcal{X}'=f_D(H,\mathcal{E}_T).
\end{equation}
We train the autoencoder by minimizing the cosine reconstruction loss over masked nodes:
\begin{equation}
\mathcal{L}_{\mathrm{int}}
=\frac{1}{|\widehat{\mathcal{V}}|}
\sum_{v_i\in\widehat{\mathcal{V}}}
\Big(1-\cos(x_i,x_i')\Big)^{\gamma},\ \gamma>1.
\end{equation}
\vspace{-5pt}

\textit{Internal-Correlation Score Calculation:}
After training, we mask a candidate node $v_i$ and compute its reconstruction loss $\mathcal{L}_{\mathrm{int}}(i)$; we define the internal-correlation score as:
\begin{equation}
    \mathcal{S}_{\mathrm{int}}(i)=\frac{1}{\mathcal{L}_{\mathrm{int}}(i)}
    \label{euq:sint}
\end{equation}
\textit{Higher $\mathcal{S}_{\mathrm{int}}(i)$ indicates stronger statistical correlation and a higher probability of the node being a trigger node.
}

\noindent\textbf{External-influence score.}
To quantify each node’s \emph{external influence}, we train a standard node classifier $f'$ on the poisoned training graph (without masking). Let
\vspace{0pt}
\begin{equation}
    H^\prime = f^\prime_E(\mathcal{V}_T, \mathcal{E}_T, \mathcal{X}_T), \quad \mathcal{Y}^\prime = f^\prime_D(\mathcal{V}_T, \mathcal{E}_T, H^\prime)
\end{equation}
where $\mathcal{Y}'$ are the predicted class probabilities; we optimize $f'$ with cross-entropy on the available labels.

After training, for each node $v_i$ we compute a counterfactual prediction by masking its features and re-running $f'$, obtaining class probabilities $\overline{\mathcal{Y}}$.
We measure how much this intervention changes the predictions of its 1-hop neighbors:
\vspace{0pt}
\begin{equation}
\mathcal{S}_{\mathrm{ext}}(i)
=
\frac{1}{|\mathcal{N}(i)|}
\sum_{j\in\mathcal{N}(i)}
\Big(
\mathrm{KL}\!\left(\mathcal{Y}'_j\,\|\,\overline{\mathcal{Y}}_j\right)
+
\mathrm{JS}\!\left(\mathcal{Y}'_j\,\|\,\overline{\mathcal{Y}}_j\right)
\Big),
\end{equation}
where $\mathrm{KL}$ Divergence~\cite{Joyce2011} and $\mathrm{JS}$ Divergence~\cite{MENENDEZ1997307} capture distributional shifts (KL is sensitive; JS is symmetric and bounded).
\textit{A higher value of $\mathcal{S}_{ext}$ indicates a greater impact on the neighbors' predictions, suggesting a higher probability of being a trigger node.}

\vspace{-5pt}
\subsection{Victim and Trigger Node Identification}
\vspace{-5pt}
\noindent\textbf{Score fusion via dual signals.}
Internal correlation and external influence are informative under different attack settings (Table~\ref{tab:intravsexternal}). We therefore combine them using an adaptively weighted sum, with the exact weights and an equal-weight ablation provided in Appendix~\ref{apd:ablation_details}.

\noindent\fbox{\begin{minipage}{\dimexpr\linewidth-2\fboxsep-2\fboxrule\relax}
\textbf{Algorithm 1. Recover victim--trigger groups from fused scores}\\
\textbf{1. Rank.} Propagate the fused scores for $P$ hops so high-scoring triggers transfer evidence to attached training victims, then rank the labeled training nodes by decreasing propagated score $\tilde s$. Within the search range (the first $K$ nodes), find the largest adjacent gap:
$k^\star=\arg\max_{1\le k<K}(\tilde s_{(k)}-\tilde s_{(k+1)})$. Retain the first $k^\star$ nodes.\\
\textbf{2. Group.} Partition these nodes by observed label and keep groups of size at least $\tau$. If none remains, skip localization and unlearning. The cutoff in Step 1 uses scores alone.\\
\textbf{3. Refine.} For one supported label, scan the ranking from the top, collecting nodes with that label until $S$ consecutive labels differ from it. For multiple labels, keep their seed groups separate.\\
\textbf{4. Recover.} In each group's $P$-hop neighborhood on the original graph, exclude identified victims and retain candidates whose propagated score is at least the group's minimum victim score. Pass both node sets and their label to the repair step.
\end{minipage}}

\noindent\begin{minipage}{\linewidth}
\textbf{Parameter choices for Algorithm 1 and its input scores.}\par
\smallskip
\begin{tabularx}{\linewidth}{@{}l l X@{}}
\toprule
Quantity & Source & Role in scoring or recovery\\
\midrule
Fusion weights, $k^\star$ & Inferred & View-specific support; position of the score gap.\\
$P$ & Inherited & Model depth for propagation and trigger recovery.\\
$\gamma=3$, $p=10\%$ & Fixed & Reconstruction contrast; fraction searched for the gap.\\
$\tau=5$, $S=3$ & Fixed & Minimum label support; single-label stopping count.\\
\bottomrule
\end{tabularx}
\end{minipage}

Appendix~\ref{apd:score_to_node_details} specifies propagation and recovery; Appendix~\ref{app:additional_robustness} reports sensitivity to the fixed choices.

\vspace{-5pt}
\subsection{Robust Learning Based on Identified Nodes}
\label{sec:repair_objective}
\vspace{-5pt}
Localization is effective but not perfect: some poisoned nodes may evade detection, and removing only detected triggers can be insufficient for stealthy \emph{in-distribution} triggers. We therefore train one robust classifier using the identified victims and triggers.

\textbf{Filtered training + unlearning.}
Let $\mathcal{C}_{\mathrm{sus}}$ be the retained suspicious label groups, with identified victims $\mathcal{V}^{I\!-\!Vic}_c$ and associated triggers $\mathcal{V}^{I\!-\!Tri}_c$ for each $c\in\mathcal{C}_{\mathrm{sus}}$.
Let $\mathcal{V}^{I\!-\!Vic}$ and $\mathcal{V}^{I\!-\!Tri}$ denote their unions over all groups.
We filter them from the training graph, $\mathcal{V}_F=\mathcal{V}_T\setminus\big(\mathcal{V}^{I\!-\!Vic}\cup \mathcal{V}^{I\!-\!Tri}\big)$, and optimize:
\vspace{-5pt}
\begin{equation}
\begin{aligned}
\mathcal{L}_{\mathrm{clean}}
&=\sum_{v_i\in\mathcal{V}_F\cap V_{\mathrm{train}}}\ell\!\left(f(\mathcal{G}_F,v_i),\,y_i\right),&
\mathcal{L}_{\mathrm{unlearn}}
=
-\sum_{c\in\mathcal{C}_{\mathrm{sus}}}
\sum_{v_i\in \mathcal{V}^{I\!-\!Vic}_c\cup\mathcal{V}^{I\!-\!Tri}_c}
\ell\!\left(f(\mathcal{G}_T,v_i),\,c\right)
\end{aligned}
\end{equation}
where $V_{\mathrm{train}}$ denotes the labeled training nodes, $\ell$ is cross-entropy, and $\mathcal{L}=\mathcal{L}_{\mathrm{clean}}+\mathcal{L}_{\mathrm{unlearn}}$. This removes identified backdoor structures and suppresses each group's trigger-target association; the same objective applies whether the converter returns one group or several groups.

\vspace{-5pt}
\section{Evaluation}

\vspace{-5pt}
\subsection{Experimental settings.}
\vspace{-5pt}
\noindent\textbf{Configuration.}
\system{} uses a 2-layer GAT masked autoencoder with learning rate $1\times10^{-4}$, sensitivity parameter $\gamma=3$, and mask rate $0.1$. 
The final defended classifier is a 2-layer GCN trained with Adam at learning rate $0.01$ for 400 iterations. We evaluate the final iterate.

\noindent\textbf{Datasets.}
We evaluate \system{} on four node-classification benchmarks: Cora~\cite{cora}, PubMed~\cite{pubmed}, Flickr~\cite{flickr}, and OGB-arxiv~\cite{ogb-arxiv}. 
Appendix~\ref{apd:dataset}-Table~\ref{tab:dataset} shows dataset statistics.

\noindent\textbf{Attack methods.}
We consider four representative graph backdoor attacks: GTA~\cite{GTA}, SBA~\cite{SBA}, UGBA~\cite{UGBA}, and DPGBA~\cite{DPGBA}. 
We tune attack hyperparameters to obtain strong attack performance before applying defenses. 
Attack details are provided in Appendix~\ref{apd:attack_details}. 
We further evaluate five adaptive attacks in Section~\ref{sec:adaptive} and Appendix~\ref{apd:additional_security_analysis}.

\noindent\textbf{Defense methods.}
We compare \system{} with SP~\cite{UGBA}, OD~\cite{DPGBA}, RIGBD~\cite{RIGBD}, and GNNGuard~\cite{GNNGuard}. 
All methods are evaluated under the same training-time defense protocol; OD and SP are therefore applied only during training, although they can also be used at inference. 
Configurations and tuning protocols are given in Appendix~\ref{apd:defense_details}. 
We exclude graph-level defenses~\cite{graphprot,yang2024distributedbackdoorattacksfederated}, as they address poisoned graphs rather than node-level victims or triggers.

\noindent\textbf{Evaluation protocol.}
Following prior work~\cite{DPGBA, UGBA, RIGBD}, we split each dataset into disjoint training ($\mathcal{G}_{train}$, 80\%) and testing ($\mathcal{G}_{test}$, 20\%) graphs. 
The attacker poisons $\mathcal{G}_{train}$ by selecting victim nodes $\mathcal{V}_B$ and attaching triggers with default size 3, producing $\mathcal{G}_T$. 
The victim size $|\mathcal{V}_B|$ is set to 10, 40, 565, and 160 for Cora, PubMed, OGB-arxiv, and Flickr, respectively. 
For evaluation, we split $\mathcal{G}_{test}$ into two equal subsets: one is triggered to compute attack success rate (\textbf{ASR}), and the other remains clean to compute clean accuracy (\textbf{CA}).

\noindent\textbf{Evaluation metrics.}
We report Recall (R) and Precision (P) for our method in identifying victim and trigger nodes, alongside ASR and CA. Recall (R) represents the proportion of correctly identified victim or trigger nodes relative to all valid nodes of that type in each run. Precision (P) indicates the percentage of correctly identified nodes among the candidates. Each experiment was conducted \textit{five times}, with average results reported. See Appendix~\ref{apd:asr_ca_calculations} for ASR and CA calculations. All values are percentages.

\noindent\textbf{Runtime.} Across the three timed datasets, scoring accounts for over 96\% of runtime. On OGB-arxiv, scoring, selection, and relearning take 1,560.11, 40.02, and 15.83 seconds, respectively, with 10.7\,GB peak memory. \system{}-Faster recomputes only the affected $L$-hop neighborhood and scores 25\% fewer candidate nodes on OGB-arxiv, reducing detection time to 5\% of the original (Appendix~\ref{apd:discussion}).

\vspace{-6pt}
\begin{table*}[h]
\caption{Backdoor defense results. Bold indicates the best result within each attack--dataset setting.}
\vspace{-5pt}
\setlength{\tabcolsep}{2.5pt}
\renewcommand{\arraystretch}{1}
\centering
\resizebox{\textwidth}{!}{
\begin{tabular}{c|c|c|cc|cc|cc|cc|cc|cc} 
\hline
\multirow{3}{*}{Attack} & \multirow{3}{*}{Dataset} & \multirow{2}{*}{Clean Graph} & \multicolumn{12}{c}{Defense} \\ 
\cline{4-15}
                        &                          &                              & \multicolumn{2}{c|}{No Defense} & \multicolumn{2}{c|}{GNNGuard} & \multicolumn{2}{c|}{OD} & \multicolumn{2}{c|}{SP} & \multicolumn{2}{c|}{RIGBD} & \multicolumn{2}{c}{PRAETORIAN} \\ 
\cline{3-15}
                        &                          & CA$\uparrow$ 
                        & ASR$\downarrow$ & CA$\uparrow$ 
                        & ASR$\downarrow$ & CA$\uparrow$ 
                        & ASR$\downarrow$ & CA$\uparrow$ 
                        & ASR$\downarrow$ & CA$\uparrow$ 
                        & ASR$\downarrow$ & CA$\uparrow$ 
                        & ASR$\downarrow$ & CA$\uparrow$ \\ 
\hline

\multirow{4}{*}{SBA}    
                        & Cora      & 83.19 & 49.78 & 83.11 & 32.09 & 78.74 & 45.07 & 83.33 & 19.29 & 82.96 & 22.84 & 83.11 & \textbf{0.09} & \textbf{84.30} \\
                        & PubMed    & 85.11 & 28.57 & 85.13 & 16.35 & 80.94 & 17.99 & \textbf{85.34} & 3.35 & 84.99 & 11.51 & 78.67 & \textbf{0.91} & 85.21 \\
                        & OGB-arxiv & 65.17 & 47.32 & 66.05 & 24.60 & 62.69 & 24.16 & 63.22 & 0.50 & 62.05 & \textbf{0.00} & 66.13 & 0.03 & \textbf{66.34} \\
                        & Flickr    & 45.29 & 90.02 & 45.22 & 10.68 & 41.38 & 88.09 & 41.84 & 89.39 & 40.67 & 70.63 & 40.17 & \textbf{0.00} & \textbf{45.83} \\
\hline

\multirow{4}{*}{GTA}    
                        & Cora      & 83.56 & 97.87 & 75.85 & 27.77 & \textbf{78.89} & 98.49 & 72.37 & 58.40 & 72.07 & 57.69 & 77.41 & \textbf{0.44} & 76.89 \\
                        & PubMed    & 85.23 & 94.45 & 82.09 & 8.84 & 80.35 & 94.85 & 78.92 & 74.38 & 81.47 & 77.85 & 79.71 & \textbf{0.27} & \textbf{83.18} \\
                        & OGB-arxiv & 65.07 & 93.68 & 62.66 & 1.10 & \textbf{66.36} & \textbf{0.00} & 38.90 & \textbf{0.00} & 38.47 & 94.64 & 62.13 & \textbf{0.00} & 65.68 \\
                        & Flickr    & 42.87 & 100.00 & 42.55 & 12.82 & 42.78 & \textbf{0.00} & 40.79 & \textbf{0.00} & 39.98 & 44.66 & 40.25 & \textbf{0.00} & \textbf{43.00} \\
\hline

\multirow{4}{*}{UGBA}   
                        & Cora      & 83.63 & 94.13 & 79.19 & 13.16 & 74.22 & 57.42 & 81.41 & 94.31 & 80.30 & 92.53 & 79.56 & \textbf{0.89} & \textbf{81.93} \\
                        & PubMed    & 85.21 & 91.78 & 82.29 & 8.26 & 77.48 & 52.33 & 83.51 & 91.77 & 82.41 & 95.48 & 78.06 & \textbf{3.67} & \textbf{84.47} \\
                        & OGB-arxiv & 65.08 & 99.03 & \textbf{65.97} & 97.71 & 64.79 & 1.72 & 64.63 & 99.37 & 63.43 & \textbf{0.00} & 65.73 & 0.05 & 65.91 \\
                        & Flickr    & 44.62 & 99.96 & 42.67 & 97.83 & \textbf{44.57} & \textbf{0.00} & 42.53 & 99.90 & 42.32 & \textbf{0.00} & 41.54 & \textbf{0.00} & 42.92 \\
\hline

\multirow{4}{*}{DPGBA}  
                        & Cora      & 83.70 & 94.35 & 80.67 & 18.09 & 78.90 & 94.17 & 80.67 & 91.74 & 81.04 & 75.48 & \textbf{81.85} & \textbf{0.17} & \textbf{81.85} \\
                        & PubMed    & 85.20 & 95.24 & 83.36 & 12.87 & 81.46 & 92.79 & 83.11 & 93.17 & 82.96 & 95.88 & 78.64 & \textbf{2.24} & \textbf{84.06} \\
                        & OGB-arxiv & 64.86 & 92.26 & 65.54 & 88.52 & 64.04 & 93.97 & 63.36 & 90.86 & 64.77 & 61.48 & 65.52 & \textbf{0.00} & \textbf{66.14} \\
                        & Flickr    & 43.45 & 91.70 & 43.84 & 90.14 & \textbf{44.33} & 90.06 & 43.18 & 90.93 & 44.14 & 85.66 & 43.77 & \textbf{0.00} & 43.62 \\
\hline
\hline
\rowcolor{softblue}
Average                 &           & 69.45 & 85.01 & 67.89 & 35.05 & 66.37 & 53.19 & 65.44 & 62.33 & 65.25 & 55.40 & 66.39 & \textbf{0.55} & \textbf{68.83} \\
\hline
\end{tabular}
}
\vspace{-10pt}
\label{tab:results}
\end{table*}
\subsection{Defense Performance} 
\vspace{-5pt} 
\label{sec:defense_performance}

We evaluate \system{} against baseline defenses across four datasets. Table~\ref{tab:results} summarizes the ASR and CA for each method. \textbf{(1) Effectiveness:} \system{} achieves the lowest average ASR (0.55\%, compared with 35.05\% for GNNGuard). \textbf{(2) Clean Accuracy Preservation:} \system{} maintains a CA comparable to the clean GCN, with an average drop of 0.62 percentage points. This suggests that PRAETORIAN can suppress backdoor behavior while largely preserving utility on clean nodes.

Compared to baselines like GNNGuard, OD, SP, and RIGBD, \system{} identifies both victim and trigger nodes and uses them for robust learning. We provide a detailed comparative analysis in Appendix~\ref{app:detailed_analysis}.

\vspace{-5pt}
\subsection{Victim and Trigger Nodes Detection}
\vspace{-5pt}
\label{sec:node_level}

\begin{wraptable}{r}{0.5\linewidth}
\vspace{-10pt}
\centering
\scriptsize
\setlength{\tabcolsep}{2.5pt}
\renewcommand{\arraystretch}{0.82}
\caption{Victim and trigger node detection.}
\label{tab:node_level}
\vspace{-6pt}

\resizebox{\linewidth}{!}{
\begin{tabular}{c|c|cc|cc} 
\hline
\multirow{3}{*}{Attack} & \multirow{3}{*}{Dataset} & \multicolumn{4}{c}{Victim and Trigger Node Detection} \\ 
\cline{3-6}
                        &                          & \multicolumn{2}{c|}{Victim Nodes} & \multicolumn{2}{c}{Trigger Nodes}  \\
                        &                          & P~$\uparrow$ & R~$\uparrow$ & P~$\uparrow$ & R~$\uparrow$ \\
\hline
\multirow{4}{*}{SBA}    
                        & Cora      & 100.00 & 100.00 & 100.00 & 83.72 \\
                        & PubMed    & 100.00 & 91.50 & 97.60  & 70.33 \\
                        & OGB-arxiv & 100.00 & 99.91 & 100.00 & 83.27 \\
                        & Flickr    & 96.97  & 100.00 & 96.97 & 100.00 \\
\hline
\multirow{4}{*}{GTA}    
                        & Cora      & 98.00  & 86.00 & 94.67  & 86.00 \\
                        & PubMed    & 96.79  & 75.55 & 96.68  & 75.55 \\
                        & OGB-arxiv & 100.00 & 88.85 & 100.00 & 88.85 \\
                        & Flickr    & 99.37  & 100.00 & 99.37  & 100.00 \\
\hline
\multirow{4}{*}{UGBA}   
                        & Cora      & 98.18  & 100.00 & 93.33  & 100.00 \\
                        & PubMed    & 100.00 & 94.50  & 100.00 & 94.50 \\
                        & OGB-arxiv & 100.00 & 90.97  & 100.00 & 90.97 \\
                        & Flickr    & 100.00 & 99.38  & 100.00 & 99.38 \\
\hline
\multirow{4}{*}{DPGBA}  
                        & Cora      & 91.21  & 100.00 & 100.00 & 100.00 \\
                        & PubMed    & 97.10  & 100.00 & 100.00 & 100.00 \\
                        & OGB-arxiv & 100.00 & 86.55  & 99.84  & 86.55 \\
                        & Flickr    & 98.76  & 100.00 & 99.79  & 100.00 \\
\hline
\hline
\rowcolor{softblue}
\multicolumn{2}{c|}{Average} & 98.52 & 94.58 & 98.64 & 91.20 \\
\hline
\end{tabular}
}
\vspace{-15pt}
\end{wraptable}
Table~\ref{tab:node_level} shows the precision and recall of \system{} in detecting victim and trigger nodes.

From the table, we observe that \system{} consistently achieves high precision, exceeding 90\% across all datasets and attack methods, with an average precision of 98\%. The high precision indicates that trigger nodes often exhibit clear internal-correlation and/or external-influence anomalies, which \system{} captures effectively.

On the other hand, \system{} achieved an average recall of 91\%. Notably, for trigger nodes generated by SBA and GTA on the PubMed dataset, the recall falls below 80\%. This lower recall mainly arises because SBA/GTA on PubMed tend to produce \emph{disconnected} triggers, where only a \textit{single} trigger node interacts with the victim, which \emph{weakens} the internal-correlation signal. In addition, in some cases the victim’s original label \emph{matches the attack’s target label}, so masking such a trigger causes only \emph{modest} shifts in the victim’s prediction distribution, thereby \emph{weakening} the external-influence signal. Together, these factors \emph{reduce detection sensitivity}, leading to lower recall in these cases. In Section~\ref{sec:adaptive}, we further examine this case and evaluate \system{}’s robustness to one-node clean-label attacks.

Importantly, even without perfectly identifying all trigger and victim nodes, \system{} remains highly effective in reducing the ASR of attacks, as shown by the results of SBA and GTA on PubMed dataset (see Table~\ref{tab:results}).

\noindent\textbf{Score correspondence.}
With scores and protocol fixed before oracle use, an exact-coalition audit covers four attacks, Cora/PubMed, and five seeds (40 runs; 1,000 victim-trigger groups). Run-level Spearman correlations are $0.66$ [95\% CI: $0.59,0.72$] for external score versus singleton influence and $0.61$ [$0.53,0.67$] for group-mean internal score versus exact group IC. Both associations are positive in all 40 runs (Appendix~\ref{apd:reported_coalition_audit}).

\noindent\textbf{False positives on attack-free graphs.}
We ran the full defense five times on each original graph without injecting triggers, using the same settings as in the poisoned-graph experiments. Table~\ref{tab:reported_clean_fp} reports node-level false-positive rates (FPR) and the change in clean accuracy relative to training without a defense.

\begin{table}[!htb]
\centering
\caption{False positives on attack-free graphs over five runs per dataset.}
\label{tab:reported_clean_fp}
\begin{tabular}{lrrr}
\toprule
Dataset & Flagged nodes (mean / median) & Node FPR & $\Delta$CA (pp)\\
\midrule
Cora & 0.00 / 0 & 0.000\% & +0.00\\
PubMed & 2.20 / 0 & 0.011\% & $-0.04$\\
Flickr & 0.00 / 0 & 0.000\% & +0.00\\
OGB-arxiv & 3.20 / 0 & 0.002\% & $-0.22$\\
\bottomrule
\end{tabular}
\end{table}

Of the eleven false positives on PubMed, ten appeared only in the internal view and one only in the external view. Appendix~\ref{apd:reported_clean_audit} examines hubs, class imbalance and community shifts, and gives the calibration protocol.

\vspace{-5pt}
\subsection{Ablation Study}
\vspace{-5pt}
\label{sec:ablation_study}

We ablate \system{}’s core components and robust learning module in the main text. Additional results (i.e., static 1:1 weighting and Top-$K$ selection) are deferred to Appendix~\ref{apd:ablation_details}.

\begin{table*}[t!]
\caption{Internal correlation vs. external influence.}
\centering
\renewcommand{\arraystretch}{1}
\resizebox{\textwidth}{!}{
\begin{tabular}{c|c|cc|cc|cc|cc|cc|cc} 
\hline
\multirow{3}{*}{Attack} 
& \multirow{3}{*}{Dataset} 
& \multicolumn{6}{c|}{Internal Correlation Only} 
& \multicolumn{6}{c}{External Influence Only} \\ 
\cline{3-14}
& 
& \multicolumn{2}{c|}{Defense} 
& \multicolumn{2}{c|}{Victim Detection} 
& \multicolumn{2}{c|}{Trigger Detection} 
& \multicolumn{2}{c|}{Defense} 
& \multicolumn{2}{c|}{Victim Detection} 
& \multicolumn{2}{c}{Trigger Detection} \\
& 
& ASR~$\downarrow$ & CA~$\uparrow$ 
& P~$\uparrow$ & R~$\uparrow$ 
& P~$\uparrow$ & R~$\uparrow$ 
& ASR~$\downarrow$ & CA~$\uparrow$ 
& P~$\uparrow$ & R~$\uparrow$ 
& P~$\uparrow$ & R~$\uparrow$ \\
\hline

\multirow{4}{*}{SBA}    
& Cora      & $11.82$ & $84.05$ & $71.21$ & $80.00$ & $80.00$ & $66.67$ & $5.24$  & $84.30$ & $81.06$ & $66.00$ & $97.65$ & $42.33$ \\
& PubMed    & $10.77$ & $85.40$ & $80.00$ & $80.00$ & $80.00$ & $66.67$ & $10.18$ & $85.18$ & $80.00$ & $14.50$ & $65.60$ & $5.58$  \\
& OGB-arxiv & $0.03$  & $66.33$ & $100.00$ & $99.91$ & $100.00$ & $83.27$ & $48.64$ & $66.17$ & $0.00$  & $0.00$  & $0.00$  & $0.00$  \\
& Flickr    & $0.00$  & $46.33$ & $98.77$ & $100.00$ & $99.72$ & $88.89$ & $0.12$  & $42.54$ & $79.89$ & $66.25$ & $94.67$ & $54.93$ \\
\hline

\multirow{4}{*}{GTA}    
& Cora      & $20.71$ & $75.33$ & $69.29$ & $46.00$ & $80.00$ & $46.00$ & $0.00$ & $74.52$ & $91.74$ & $86.00$ & $84.50$ & $86.00$ \\
& PubMed    & $1.73$  & $83.16$ & $100.00$ & $15.50$ & $100.00$ & $15.50$ & $0.31$ & $83.39$ & $96.68$ & $73.00$ & $95.40$ & $73.00$ \\
& OGB-arxiv & $95.99$ & $60.69$ & $0.00$  & $0.00$  & $0.00$  & $0.00$  & $0.01$ & $65.76$ & $100.00$ & $88.67$ & $100.00$ & $88.67$ \\
& Flickr    & $60.00$ & $41.81$ & $38.53$ & $39.13$ & $38.53$ & $39.13$ & $0.00$ & $41.86$ & $97.67$ & $71.50$ & $92.66$ & $71.50$ \\
\hline

\multirow{4}{*}{UGBA}   
& Cora      & $1.87$  & $82.30$ & $86.06$ & $94.00$ & $100.00$ & $94.00$ & $3.20$  & $81.78$ & $90.18$ & $92.00$ & $92.67$ & $90.00$ \\
& PubMed    & $54.40$ & $83.49$ & $60.00$ & $19.00$ & $60.00$  & $19.00$ & $4.08$  & $84.47$ & $100.00$ & $92.00$ & $100.00$ & $92.00$ \\
& OGB-arxiv & $0.16$  & $65.96$ & $100.00$ & $90.97$ & $100.00$ & $90.97$ & $99.06$ & $66.09$ & $0.00$  & $0.00$  & $0.00$  & $0.00$ \\
& Flickr    & $0.00$  & $42.86$ & $100.00$ & $96.88$ & $100.00$ & $96.88$ & $0.00$  & $42.87$ & $100.00$ & $99.38$ & $100.00$ & $99.38$ \\
\hline

\multirow{4}{*}{DPGBA}  
& Cora      & $0.17$  & $82.30$ & $92.70$ & $100.00$ & $100.00$ & $100.00$ & $19.13$ & $81.93$ & $80.00$ & $80.00$ & $80.00$ & $77.33$ \\
& PubMed    & $3.72$  & $83.88$ & $97.10$ & $100.00$ & $100.00$ & $100.00$ & $88.13$ & $83.48$ & $53.33$ & $2.00$  & $27.00$ & $1.00$ \\
& OGB-arxiv & $0.02$  & $66.19$ & $100.00$ & $86.55$ & $99.84$  & $86.55$  & $92.65$ & $65.87$ & $0.00$  & $0.00$  & $0.00$  & $0.00$ \\
& Flickr    & $0.08$  & $41.60$ & $98.77$ & $100.00$ & $99.65$  & $100.00$ & $31.70$ & $40.40$ & $65.04$ & $66.85$ & $63.79$ & $66.46$ \\
\hline

\end{tabular}
}
\label{tab:intravsexternal}
\end{table*}

\noindent\textbf{Internal correlation and external influence.}
To assess the contribution of the two auditing views, we evaluate two variants: \system$\backslash$I, which removes internal correlation, and \system$\backslash$E, which removes external influence. 
Table~\ref{tab:intravsexternal} shows that the two views capture different attack patterns. 
\system$\backslash$I is less effective against dense and structurally robust triggers such as DPGBA, where backdoor patterns can persist after masking individual nodes. 
In contrast, \system$\backslash$E underperforms on sparse or single-node triggers such as GTA and UGBA, where internal subgraph evidence is limited. 
Combining both views allows \system{} to cover a broader spectrum of trigger structures, from sparse to dense patterns. 
See detailed failure-mode analysis in Appendix~\ref{app:ablation_ic_ex_analysis}.

\noindent\textbf{Effectiveness of robust learning.}
To evaluate the defense beyond detection, we compare \system{} with two variants: \system$\backslash$R, which removes detected nodes and retrains normally, and \system$\backslash$RV, which unlearns only on victim nodes. 
Figure~\ref{fig:ablation_study_3} reports the average ASR and CA across attacks and datasets. Localization is held fixed across repair variants (Appendix~\ref{apd:reported_repair}).

\Needspace*{9\baselineskip}
\begin{wrapfigure}{r}{0.5\linewidth}
    \vspace{-10pt} 
    \centering
    \includegraphics[width=\linewidth]{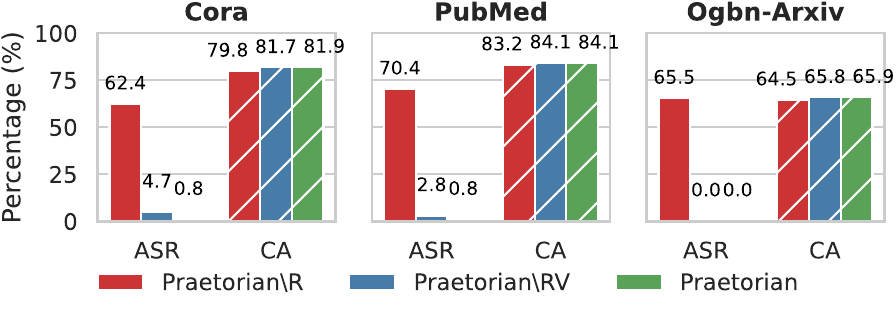}
    \vspace{-20pt}
    \caption{The Role of Robust Learning}
    \vspace{-12pt}
    \label{fig:ablation_study_3}
\end{wrapfigure}

\system$\backslash$R still yields a high average ASR of $66.07\%$, showing that simple node removal is insufficient to eliminate residual backdoor effects. 
\system$\backslash$RV improves robustness but remains less effective than \system{}, suggesting that victim-only unlearning does not fully cover the adversarial distribution. 
By unlearning on both victim and trigger nodes, \system{} provides more complete adversarial supervision and reduces ASR to nearly $0\%$ while preserving clean accuracy.

\vspace{-6pt}
\subsection{Robustness Against Varying Settings}
\label{sec:varying_setting}
\vspace{-5pt}
We evaluate the resilience of \system{} against varying attack intensities, specifically examining the impact of Trigger Size (TS) and Victim Size (VS). Additional evaluations regarding \textit{model architecture and stop criterion} are detailed in Appendix~\ref{app:additional_robustness}.

\noindent\textbf{Impact of Trigger Size (TS).}
We vary the trigger size from 1 to 5 on the Cora dataset under the DPGBA attack. As shown in Table~\ref{tab:ts_vs}, \system{} remains stable across these settings, maintaining precision $>90\%$ and recall $>85\%$ across all settings. This indicates that our dual-scoring mechanism (internal and external) effectively captures adversarial anomalies across different trigger sizes. Consequently, \system{} consistently neutralizes the attack (ASR $\approx 0\%$) while preserving a Clean Accuracy (CA) higher than that of the undefended backdoored model.

\Needspace{13\baselineskip}

\begin{wraptable}{r}{0.53\linewidth}
\vspace{-10pt}
\centering
\renewcommand{\arraystretch}{0.9}
\caption{Different Trigger Size and Victim Size}
\label{tab:ts_vs}
\vspace{-6pt}

\resizebox{\linewidth}{!}{
\begin{tabular}{c|cc|cc|cc|cc} 
\hline
\multirow{3}{*}{Setting} 
& \multicolumn{2}{c|}{\multirow{2}{*}{No Defense}} 
& \multicolumn{2}{c|}{\multirow{2}{*}{Defense}} 
& \multicolumn{4}{c}{Node Detection} \\ 
\cline{6-9}
& \multicolumn{2}{c|}{} 
& \multicolumn{2}{c|}{} 
& \multicolumn{2}{c|}{Victim Nodes} 
& \multicolumn{2}{c}{Trigger Nodes} \\ 
\cline{2-9}
& ASR & CA 
& ASR & CA 
& P & R 
& P & R \\
\hline
\hline
TS=1. VS=10  & $74.81$ & $82.96$ & $0.09$ & $83.26$ & $98.00$  & $88.00$  & $98.00$  & $88.00$  \\
TS=2. VS=10  & $86.17$ & $82.67$ & $2.00$ & $83.11$ & $94.36$  & $98.00$  & $100.00$ & $96.00$  \\
TS=3. VS=10  & $92.44$ & $82.00$ & $0.61$ & $83.41$ & $91.21$  & $100.00$ & $100.00$ & $100.00$ \\
TS=4. VS=10  & $95.48$ & $80.47$ & $0.70$ & $83.04$ & $96.36$  & $100.00$ & $99.51$  & $100.00$ \\
TS=5. VS=10  & $97.30$ & $79.26$ & $0.00$ & $80.37$ & $96.36$  & $100.00$ & $100.00$ & $100.00$ \\
\hline
TS=3. VS=10  & $92.44$ & $82.00$ & $0.61$ & $83.41$ & $91.21$  & $100.00$ & $100.00$ & $100.00$ \\
TS=3. VS=20  & $96.61$ & $80.74$ & $0.00$ & $81.56$ & $100.00$ & $100.00$ & $100.00$ & $96.67$  \\
TS=3. VS=40  & $94.70$ & $80.67$ & $0.09$ & $81.78$ & $99.51$  & $100.00$ & $99.67$  & $99.67$  \\
TS=3. VS=80  & $97.39$ & $80.00$ & $0.00$ & $80.44$ & $100.00$ & $100.00$ & $99.83$  & $99.25$  \\
TS=3. VS=160 & $98.70$ & $78.12$ & $0.00$ & $80.67$ & $99.88$  & $99.00$  & $99.00$  & $98.79$  \\
\hline
\end{tabular}
}
\vspace{-12pt}
\end{wraptable}
\noindent\textbf{Impact of Victim Size (VS).}
To assess scalability, we test victim sizes ranging from 10 to 160. Table~\ref{tab:ts_vs} reveals that \system{} remains robust to the scale of the attack, with precision consistently exceeding $90\%$ and recall surpassing $95\%$. Unlike methods that may degrade as the attack becomes more pervasive, our approach accurately isolates both victim and trigger nodes irrespective of the injection volume. This results in consistent defense performance (near-zero ASR) without compromising model utility.

\vspace{-6pt}
\section{Adaptive Attacks}
\vspace{-6pt}
\label{sec:adaptive}

We evaluate \system{} against adaptive adversaries targeting its two-view auditing pipeline. 
Because \system{} relies on internal-correlation (IC) and external-influence (EI) scores followed by score-to-node (S2N) localization, attackers can mainly adapt by weakening IC, diluting EI, or challenging S2N. 
Table~\ref{tab:adaptive_coverage} summarizes the covered adaptive surfaces and outcomes. 
We highlight two representative regimes, \textbf{One-node Clean-Label} and \textbf{Random Disruption}, in the main text, with additional \textit{white-box}, \textit{asymmetric}, and \textit{multi-target} settings reported in Appendix~\ref{apd:additional_security_analysis}.

\begin{table}[h!]
\centering
\vspace{-5pt}
\caption{Coverage and outcomes of adaptive attacks against \system{}. IC: internal-correlation score; EI: external-influence score; S2N: score-to-node localization.}
\label{tab:adaptive_coverage}
\footnotesize
\setlength{\tabcolsep}{3pt}
\renewcommand{\arraystretch}{1.0}
\begin{tabularx}{\linewidth}{l l l X l}
\toprule
\textbf{Setting} & \textbf{Adaptive lever} & \textbf{Target} & \textbf{Main outcome} & \textbf{Evidence} \\
\midrule
One-node clean-label &
TS$=1$ + clean-label victims &
IC, EI &
Post-defense ASR is at most $18.1\%$ under varying VS/TS. &
Table~\ref{tab:clean_vs} \\
\midrule
Random disruption &
Noisy triggers + large VS &
IC, EI &
Large noisy insertions can yield high residual ASR; CA cost varies by configuration. &
Fig.~\ref{fig:random_noise_asr_ca} \\
\midrule
White-box optimization &
Defense-aware trigger search &
IC, EI, S2N &
Partial/Full ASR is 2.96\%/0.00\%; Full trigger recall is 86\%. &
App.~\ref{sec:baseline} \\
\midrule
Asymmetric insertion &
Weak train / strong test triggers &
train/test gap &
ASR is $0\%$; mean victim/trigger P/R is at least 97\%/88\%. &
App.~\ref{sec:asymmetric} \\
\midrule
Multi-target attack &
Multiple target labels &
S2N grouping &
ASR drops to $0.0\%$ for DPGBA and $2.2\%$ for UGBA; victim and trigger P/R exceed $92\%$. &
App.~\ref{apd:multi-target} \\
\bottomrule
\end{tabularx}
\vspace{-10pt}
\end{table}

\noindent\textbf{One-node clean-label.}
This attack jointly weakens both auditing views. 
By constraining the trigger to a single node (TS$=1$), it removes intra-trigger consistency cues; by selecting clean-label victims whose ground-truth labels already match the target class, it reduces observable external influence after trigger removal. 
The remaining backdoor effect is therefore subtle and localized, challenging both the external-influence view and S2N localization.
We instantiate this setting with three UGBA and DPGBA variants on Cora and PubMed.

\begin{table}[htbp]
\centering
\caption{Varying VS of One-node Clean Label Attack}
\label{tab:clean_vs}
\normalsize
\renewcommand{\arraystretch}{0.82}
\setlength{\tabcolsep}{2.6pt}
\begin{tabular*}{\linewidth}{@{}l@{\extracolsep{\fill}}rrrrrrrrr@{}}
\toprule
& & \multicolumn{2}{c}{No defense} & \multicolumn{2}{c}{Defense} &
\multicolumn{2}{c}{Victim} & \multicolumn{2}{c}{Trigger}\\
\cmidrule(lr){3-4}\cmidrule(lr){5-6}\cmidrule(lr){7-8}\cmidrule(l){9-10}
Attack & VS & ASR & CA & ASR & CA & P & R & P & R\\
\midrule
DPGBA$_{CO}$ & 10 & 53.74 & 82.82 & 18.09 & 83.11 & 50.83 & 28.00 & 80.00 & 26.00 \\
 & 20 & 70.87 & 83.48 & 15.83 & 82.74 & 84.11 & 33.00 & 100.00 & 31.00 \\
 & 40 & 66.96 & 83.19 & 0.44 & 80.96 & 95.63 & 60.50 & 98.72 & 59.50 \\
 & 80 & 61.57 & 82.89 & 0.00 & 80.30 & 98.13 & 81.00 & 100.00 & 80.50 \\
 & 160 & 81.22 & 82.52 & 0.00 & 82.37 & 99.35 & 89.88 & 100.00 & 89.50 \\
\midrule
UGBA$_{CO}$ & 10 & 92.18 & 81.48 & 6.76 & 82.89 & 84.00 & 84.00 & 100.00 & 84.00 \\
 & 20 & 90.31 & 82.30 & 4.18 & 82.44 & 93.61 & 85.00 & 100.00 & 85.00 \\
 & 40 & 93.69 & 81.11 & 0.44 & 80.00 & 95.10 & 97.00 & 100.00 & 96.50 \\
 & 80 & 92.98 & 80.52 & 0.00 & 79.85 & 98.72 & 95.00 & 100.00 & 95.00 \\
 & 160 & 93.07 & 81.70 & 0.27 & 80.82 & 99.65 & 97.50 & 99.45 & 97.50 \\
\midrule
DPGBA$_{CO}$ + UGBA$_{CO}$ & 10 & 21.57 & 83.33 & 13.91 & 83.56 & 48.33 & 20.00 & 80.00 & 18.00 \\
 & 20 & 38.44 & 83.33 & 9.22 & 83.33 & 76.17 & 25.00 & 100.00 & 23.00 \\
 & 40 & 31.74 & 83.56 & 6.87 & 82.74 & 93.06 & 54.00 & 99.33 & 53.00 \\
 & 80 & 47.48 & 83.26 & 2.00 & 81.63 & 97.58 & 72.57 & 100.00 & 71.57 \\
 & 160 & 49.04 & 83.11 & 0.61 & 81.33 & 99.15 & 88.00 & 100.00 & 87.25 \\
\bottomrule
\end{tabular*}
\end{table}

As shown in Table~\ref{tab:clean_vs}, \system{} substantially reduces ASR, e.g., from $92.18\%$ to $6.76\%$ for $UGBA_{CO}$ with VS=10. 
Increasing the victim size makes the residual trigger influence harder to hide: when VS exceeds 40, \system{} identifies at least 60\% of triggers. 
Across victim-size variations and trigger-size variations (Table~\ref{tab:clean_ts}), the post-defense ASR never exceeds 18.1\%. The~hardest one-node $DPGBA_{CO}$ case retains 18.09\% ASR (from 53.74\%) with victim/trigger recall of only 28\%/26\%.

\Needspace{14\baselineskip}
\begin{wrapfigure}[15]{r}{0.50\linewidth}
    \centering
    \includegraphics[width=\linewidth]{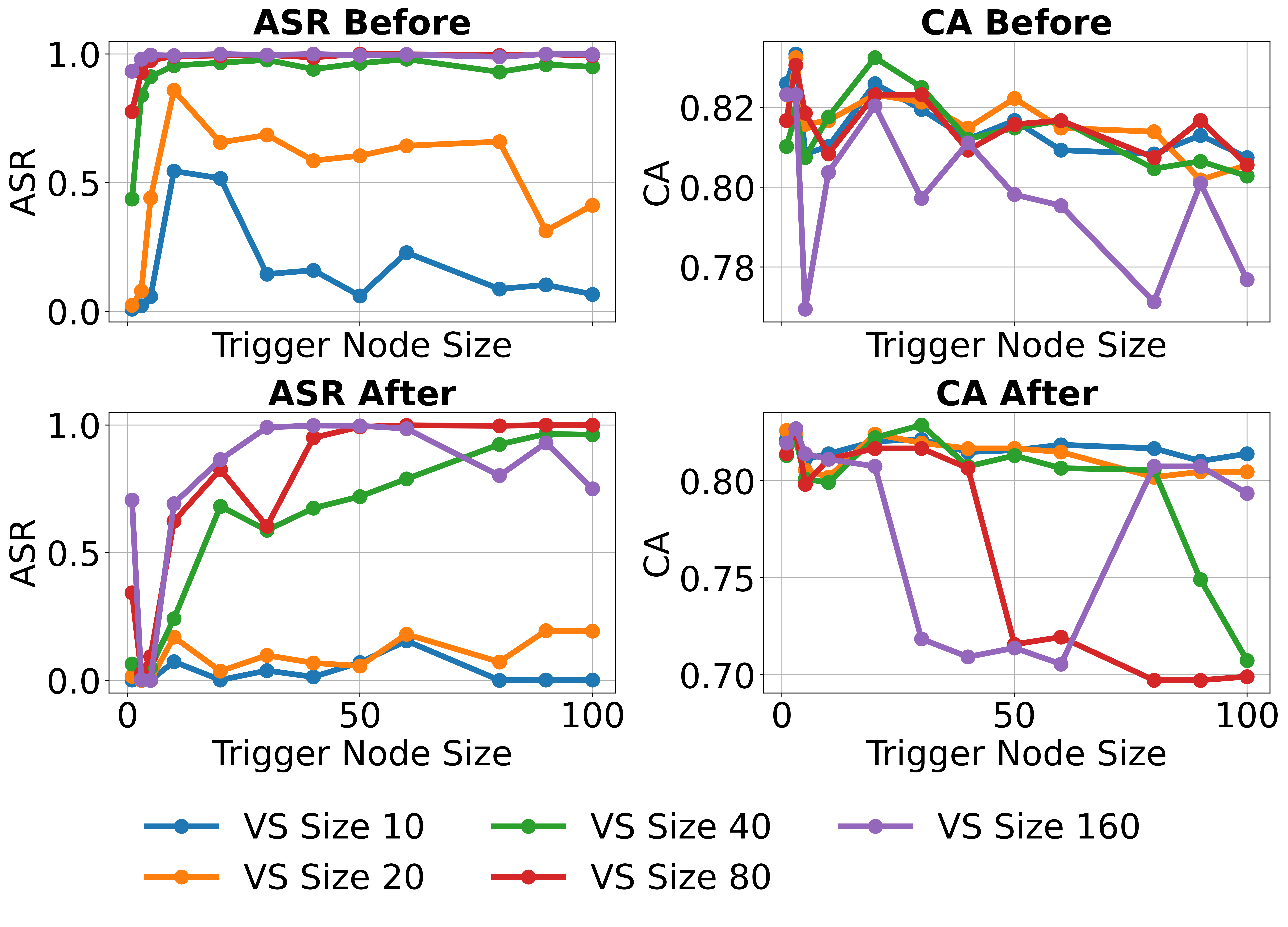}
    \caption{Random disruption on Cora.}
    \label{fig:random_noise_asr_ca}
\end{wrapfigure}
\noindent\textbf{Random disruption.}
This attack targets both views from the opposite direction. 
It injects Gaussian noise into multi-node trigger features to weaken internal-correlation cues, while increasing the victim size to distribute the backdoor association and reduce reliance on any single influential trigger node. 
We implement this strategy based on SBA on Cora and PubMed.

Figure~\ref{fig:random_noise_asr_ca} shows that larger triggers or victim sets can increase ASR under random disruption. Some attacks reach high ASR with little loss in clean accuracy, but inject many noisy nodes. Appendix~\ref{sec:random_noise_add} gives the results on Cora and PubMed.

\vspace{-5pt}
\section{Conclusion}
\vspace{-5pt}
In this paper, we introduced \system{}, a new defense paradigm against graph backdoor attacks that integrates both \emph{internal-correlation} and \emph{external-influence} analyses to accurately identify trigger and victim nodes. 
Experiments across multiple datasets and attack scenarios illustrate the attacker's dilemma: the clean-label variants achieve only a low ASR after defense, while random disruption can sustain high ASR by injecting larger noisy triggers or poisoning more victims.

\clearpage
\begingroup
\renewcommand{\acksection}{%
  \section*{Acknowledgments and Disclosures}%
  \addcontentsline{toc}{section}{Acknowledgments and Disclosures}}
\begin{ack}
\noindent\textbf{Funding.} In the 36 months preceding submission, none of the authors received third-party funding or in-kind support for this work.

\noindent\textbf{Competing interests.} During the same period, none of the authors had financial relationships outside the submitted work that could reasonably be perceived as influencing it.
\end{ack}
\endgroup

\bibliographystyle{unsrtnat}
\bibliography{ref}

@String{Computing = "Computing" }

@String{Computer = "{IEEE} Computer" }

@String{Springer = "Springer-Verlag" }

@BOOK{test,
   author = "Donald E. Knuth",
   title = "Seminumerical Algorithms",
   volume = 2,
   series = "The Art of Computer Programming",
   publisher = "Addison-Wesley",
   address = "Reading, MA",
   edition = "2nd",
   month = "10~" # jan,
   year = "1981",
}

@ArtifactSoftware{R,
    title = {R: A Language and Environment for Statistical Computing},
    author = {{R Core Team}},
    organization = {R Foundation for Statistical Computing},
    address = {Vienna, Austria},
    year = {2019},
    url = {https://www.R-project.org/},
}

@article{cora,
  title={Automating the construction of Internet portals with machine learning},
  author={McCallum, A. K. and Nigam, K. and Rennie, J. and et al.},
  journal={Information Retrieval},
  volume={3},
  number={2},
  pages={127--163},
  year={2000},
  doi={10.1023/A:1009953814988},
}

@article{pubmed, 
title={Collective Classification in Network Data}, volume={29}, 
DOI={10.1609/aimag.v29i3.2157}, 
number={3}, 
journal={AI Magazine}, 
author={Sen, Prithviraj and Namata, Galileo and Bilgic, Mustafa and Getoor, Lise and Galligher, Brian and Eliassi-Rad, Tina}, 
year={2008},
month={Sep.}, 
pages={93} }

@misc{OGB-arxiv,
      title={Open Graph Benchmark: Datasets for Machine Learning on Graphs}, 
      author={Weihua Hu and Matthias Fey and Marinka Zitnik and Yuxiao Dong and Hongyu Ren and Bowen Liu and Michele Catasta and Jure Leskovec},
      year={2021} 
}

@inproceedings {GTA,
author = {Zhaohan Xi and Ren Pang and Shouling Ji and Ting Wang},
title = {Graph Backdoor},
booktitle = {30th USENIX Security Symposium (USENIX Security 21)},
year = {2021},
isbn = {978-1-939133-24-3},
pages = {1523--1540},
}

@inproceedings{UGBA,
  title={Unnoticeable Backdoor Attacks on Graph Neural Networks},
  author={Dai, Enyan and Lin, Minhua and Zhang, Xiang and Wang, Suhang},
  booktitle={Proceedings of the ACM Web Conference 2023},
  pages={2263--2273},
  year={2023}
}

@inproceedings{SBA,
author = {Zhang, Zaixi and Jia, Jinyuan and Wang, Binghui and Gong, Neil Zhenqiang},
title = {Backdoor Attacks to Graph Neural Networks},
year = {2021},
isbn = {9781450383653},
publisher = {Association for Computing Machinery},
address = {New York, NY, USA},
doi = {10.1145/3450569.3463560},
pages = {15–26},
numpages = {12},
location = {Virtual Event, Spain},
series = {SACMAT '21}
}

@inproceedings{DPGBA,
  title={Rethinking graph backdoor attacks: A distribution-preserving perspective},
  author={Zhang, Zhiwei and Lin, Minhua and Dai, Enyan and Wang, Suhang},
  booktitle={Proceedings of the 30th ACM SIGKDD Conference on Knowledge Discovery and Data Mining},
  pages={4386--4397},
  year={2024}
}

@inproceedings{GNNGuard,
title     = {GNNGuard: Defending Graph Neural Networks against Adversarial Attacks},
author    = {Zhang, Xiang and Zitnik, Marinka},
booktitle = {NeurIPS},
year      = {2020}
}

@inproceedings{RobustGCN,
  title={Robust graph convolutional networks against adversarial attacks},
  author={Zhu, Dingyuan and Zhang, Ziwei and Cui, Peng and Zhu, Wenwu},
  booktitle={Proceedings of the 25th ACM SIGKDD International Conference on Knowledge Discovery \& Data Mining},
  pages={1399--1407},
  year={2019}
}

@inproceedings{
RIGBD,
title={Robustness Inspired Graph Backdoor Defense},
author={Zhiwei Zhang and Minhua Lin and Junjie Xu and Zongyu Wu and Enyan Dai and Suhang Wang},
booktitle={The Thirteenth International Conference on Learning Representations},
year={2025},
}

@inproceedings{explainabilitybased,
author = {Xu, Jing and Xue, Minhui (Jason) and Picek, Stjepan},
title = {Explainability-based Backdoor Attacks Against Graph Neural Networks},
year = {2021},
isbn = {9781450385619},
publisher = {Association for Computing Machinery},
address = {New York, NY, USA},
doi = {10.1145/3468218.3469046},
booktitle = {Proceedings of the 3rd ACM Workshop on Wireless Security and Machine Learning},
pages = {31–36},
numpages = {6},
location = {Abu Dhabi, United Arab Emirates},
series = {WiseML '21}
}

@misc{dyn-link,
      title={Dyn-Backdoor: Backdoor Attack on Dynamic Link Prediction}, 
      author={Jinyin Chen and Haiyang Xiong and Haibin Zheng and Jian Zhang and Guodong Jiang and Yi Liu},
      year={2021},
      eprint={2110.03875},
      archivePrefix={arXiv},
      primaryClass={cs.AI},
}

@misc{link-backdoor,
      title={Link-Backdoor: Backdoor Attack on Link Prediction via Node Injection}, 
      author={Haibin Zheng and Haiyang Xiong and Haonan Ma and Guohan Huang and Jinyin Chen},
      year={2022},
      eprint={2208.06776},
      archivePrefix={arXiv},
      primaryClass={cs.SI},
}

@inproceedings{RandomSmooth,
author = {Wang, Binghui and Jia, Jinyuan and Cao, Xiaoyu and Gong, Neil Zhenqiang},
title = {Certified Robustness of Graph Neural Networks against Adversarial Structural Perturbation},
year = {2021},
isbn = {9781450383325},
publisher = {Association for Computing Machinery},
address = {New York, NY, USA},
doi = {10.1145/3447548.3467295},
booktitle = {Proceedings of the 27th ACM SIGKDD Conference on Knowledge Discovery \& Data Mining},
pages = {1645–1653},
numpages = {9},
location = {Virtual Event, Singapore},
series = {KDD '21}
}

@misc{gcn,
      title={Semi-Supervised Classification with Graph Convolutional Networks}, 
      author={Thomas N. Kipf and Max Welling},
      year={2017},
      eprint={1609.02907},
      archivePrefix={arXiv},
      primaryClass={cs.LG},
}

@misc{gat,
      title={Graph Attention Networks}, 
      author={Petar Veličković and Guillem Cucurull and Arantxa Casanova and Adriana Romero and Pietro Liò and Yoshua Bengio},
      year={2018},
      eprint={1710.10903},
      archivePrefix={arXiv},
      primaryClass={stat.ML},
}

@misc{graphsage,
      title={Inductive Representation Learning on Large Graphs}, 
      author={William L. Hamilton and Rex Ying and Jure Leskovec},
      year={2018},
      eprint={1706.02216},
      archivePrefix={arXiv},
      primaryClass={cs.SI},
}

@misc{gin,
      title={How Powerful are Graph Neural Networks?}, 
      author={Keyulu Xu and Weihua Hu and Jure Leskovec and Stefanie Jegelka},
      year={2019},
      eprint={1810.00826},
      archivePrefix={arXiv},
      primaryClass={cs.LG},
}

@misc{linkprediction2018,
      title={Link Prediction Based on Graph Neural Networks}, 
      author={Muhan Zhang and Yixin Chen},
      year={2018},
      eprint={1802.09691},
      archivePrefix={arXiv},
      primaryClass={cs.LG},
}

@inproceedings{DGCGN,
author = {Zhang, Muhan and Cui, Zhicheng and Neumann, Marion and Chen, Yixin},
title = {An end-to-end deep learning architecture for graph classification},
year = {2018},
isbn = {978-1-57735-800-8},
publisher = {AAAI Press},
articleno = {544},
numpages = {8},
location = {New Orleans, Louisiana, USA},
series = {AAAI'18/IAAI'18/EAAI'18}
}

@misc{fan2019graphneuralnetworkssocial,
      title={Graph Neural Networks for Social Recommendation}, 
      author={Wenqi Fan and Yao Ma and Qing Li and Yuan He and Eric Zhao and Jiliang Tang and Dawei Yin},
      year={2019},
      eprint={1902.07243},
      archivePrefix={arXiv},
      primaryClass={cs.IR},
}

@article{moleculargnn,
   title={Molecular Geometry Prediction using a Deep Generative Graph Neural Network},
   volume={9},
   ISSN={2045-2322},
   DOI={10.1038/s41598-019-56773-5},
   number={1},
   journal={Scientific Reports},
   publisher={Springer Science and Business Media LLC},
   author={Mansimov, Elman and Mahmood, Omar and Kang, Seokho and Cho, Kyunghyun},
   year={2019},
   month=dec }

@ARTICLE{biologicalgnn,
AUTHOR={Zhang, Xiao-Meng  and Liang, Li  and Liu, Lin  and Tang, Ming-Jing },
TITLE={Graph Neural Networks and Their Current Applications in Bioinformatics},
JOURNAL={Frontiers in Genetics},
VOLUME={12},
YEAR={2021},
}

@inproceedings{recommendatationgnn,
author = {Gao, Chen and Wang, Xiang and He, Xiangnan and Li, Yong},
title = {Graph Neural Networks for Recommender System},
year = {2022},
isbn = {9781450391320},
publisher = {Association for Computing Machinery},
address = {New York, NY, USA},
doi = {10.1145/3488560.3501396},
pages = {1623–1625},
numpages = {3},
location = {Virtual Event, AZ, USA},
series = {WSDM '22}
}

@article{Adversarial-Attack-and-Defense,
   title={Adversarial Attack and Defense on Graph Data: A Survey},
   ISSN={2326-3865},
   DOI={10.1109/tkde.2022.3201243},
   journal={IEEE Transactions on Knowledge and Data Engineering},
   publisher={Institute of Electrical and Electronics Engineers (IEEE)},
   author={Sun, Lichao and Dou, Yingtong and Yang, Carl and Zhang, Kai and Wang, Ji and Yu, Philip S. and He, Lifang and Li, Bo},
   year={2022},
   pages={1–20} }

@misc{qian2023robusttraininggraphneural,
      title={Robust Training of Graph Neural Networks via Noise Governance}, 
      author={Siyi Qian and Haochao Ying and Renjun Hu and Jingbo Zhou and Jintai Chen and Danny Z. Chen and Jian Wu},
      year={2023},
      eprint={2211.06614},
      archivePrefix={arXiv},
      primaryClass={cs.LG},
}

@misc{dai2021nrgnnlearninglabelnoiseresistant,
      title={NRGNN: Learning a Label Noise-Resistant Graph Neural Network on Sparsely and Noisily Labeled Graphs}, 
      author={Enyan Dai and Charu Aggarwal and Suhang Wang},
      year={2021},
      eprint={2106.04714},
      archivePrefix={arXiv},
      primaryClass={cs.LG},
}

@misc{gosch2023adversarialtraininggraphneural,
      title={Adversarial Training for Graph Neural Networks: Pitfalls, Solutions, and New Directions}, 
      author={Lukas Gosch and Simon Geisler and Daniel Sturm and Bertrand Charpentier and Daniel Zügner and Stephan Günnemann},
      year={2023},
      eprint={2306.15427},
      archivePrefix={arXiv},
      primaryClass={cs.LG},
}

@misc{li2022spectraladversarialtrainingrobust,
      title={Spectral Adversarial Training for Robust Graph Neural Network}, 
      author={Jintang Li and Jiaying Peng and Liang Chen and Zibin Zheng and Tingting Liang and Qing Ling},
      year={2022},
      eprint={2211.10896},
      archivePrefix={arXiv},
      primaryClass={cs.LG},
}

@misc{li2022backdoorlearningsurvey,
      title={Backdoor Learning: A Survey}, 
      author={Yiming Li and Yong Jiang and Zhifeng Li and Shu-Tao Xia},
      year={2022},
      eprint={2007.08745},
      archivePrefix={arXiv},
      primaryClass={cs.CR},
}

@article{adversarial-survey,
   title={How Deep Learning Sees the World: A Survey on Adversarial Attacks and Defenses},
   volume={12},
   ISSN={2169-3536},
   DOI={10.1109/access.2024.3395118},
   journal={IEEE Access},
   publisher={Institute of Electrical and Electronics Engineers (IEEE)},
   author={Costa, Joana C. and Roxo, Tiago and Proença, Hugo and Inácio, Pedro Ricardo Morais},
   year={2024},
   pages={61113–61136} 
}

@misc{badnets,
      title={BadNets: Identifying Vulnerabilities in the Machine Learning Model Supply Chain}, 
      author={Tianyu Gu and Brendan Dolan-Gavitt and Siddharth Garg},
      year={2019},
      eprint={1708.06733},
      archivePrefix={arXiv},
      primaryClass={cs.CR},
}

@article{wang2019dgl,
    title={Deep Graph Library: A Graph-Centric, Highly-Performant Package for Graph Neural Networks},
    author={Minjie Wang and Da Zheng and Zihao Ye and Quan Gan and Mufei Li and Xiang Song and Jinjing Zhou and Chao Ma and Lingfan Yu and Yu Gai and Tianjun Xiao and Tong He and George Karypis and Jinyang Li and Zheng Zhang},
    year={2019},
}

@misc{abl,
      title={Anti-Backdoor Learning: Training Clean Models on Poisoned Data}, 
      author={Yige Li and Xixiang Lyu and Nodens Koren and Lingjuan Lyu and Bo Li and Xingjun Ma},
      year={2021},
      eprint={2110.11571},
      archivePrefix={arXiv},
      primaryClass={cs.LG},
}

@ARTICLE{gnn-healthcare,
  author={Paul, Showmick Guha and Saha, Arpa and Hasan, Md. Zahid and Noori, Sheak Rashed Haider and Moustafa, Ahmed},
  journal={IEEE Access}, 
  title={A Systematic Review of Graph Neural Network in Healthcare-Based Applications: Recent Advances, Trends, and Future Directions}, 
  year={2024},
  volume={12},
  number={},
  pages={15145-15170},
  doi={10.1109/ACCESS.2024.3354809}}

@misc{finance-gnn,
      title={A Review on Graph Neural Network Methods in Financial Applications}, 
      author={Jianian Wang and Sheng Zhang and Yanghua Xiao and Rui Song},
      year={2022},
      eprint={2111.15367},
      archivePrefix={arXiv},
      primaryClass={q-fin.ST},
}

@article{healthycate-gnn,
   title={Graph-Based Deep Learning for Medical Diagnosis and Analysis: Past, Present and Future},
   volume={21},
   ISSN={1424-8220},
   DOI={10.3390/s21144758},
   number={14},
   journal={Sensors},
   publisher={MDPI AG},
   author={Ahmedt-Aristizabal, David and Armin, Mohammad Ali and Denman, Simon and Fookes, Clinton and Petersson, Lars},
   year={2021},
   month=jul, pages={4758} }

@article{robust-graph_learning,
   title={Robust Graph Learning From Noisy Data},
   volume={50},
   ISSN={2168-2275},
   DOI={10.1109/tcyb.2018.2887094},
   number={5},
   journal={IEEE Transactions on Cybernetics},
   publisher={Institute of Electrical and Electronics Engineers (IEEE)},
   author={Kang, Zhao and Pan, Haiqi and Hoi, Steven C. H. and Xu, Zenglin},
   year={2020},
   month=may, pages={1833–1843} }

@misc{nt2019learninggraphneuralnetworks,
      title={Learning Graph Neural Networks with Noisy Labels}, 
      author={Hoang NT and Choong Jun Jin and Tsuyoshi Murata},
      year={2019},
      eprint={1905.01591},
      archivePrefix={arXiv},
      primaryClass={cs.LG},
}

@misc{survey-adversarial-learning-graph,
      title={A Survey of Adversarial Learning on Graphs}, 
      author={Liang Chen and Jintang Li and Jiaying Peng and Tao Xie and Zengxu Cao and Kun Xu and Xiangnan He and Zibin Zheng and Bingzhe Wu},
      year={2022},
      eprint={2003.05730},
      archivePrefix={arXiv},
      primaryClass={cs.LG},
      url={https://arxiv.org/abs/2003.05730}, 
}

@article{Blended,
	title        = {Targeted Backdoor Attacks on Deep Learning Systems Using Data Poisoning},
	author       = {Xinyun Chen and Chang Liu and Bo Li and Kimberly Lu and Dawn Song},
	journal      = {arXiv preprint arXiv:1712.05526},
	year         = {2017}
}

@misc{homophilynecessitygraphneural,
      title={Is Homophily a Necessity for Graph Neural Networks?}, 
      author={Yao Ma and Xiaorui Liu and Neil Shah and Jiliang Tang},
      year={2023},
      eprint={2106.06134},
      archivePrefix={arXiv},
      primaryClass={cs.LG},
}

@article{Gilbert1959RandomGraphs,
  author       = {Edgar N. Gilbert},
  title        = {Random Graphs},
  journal      = {The Annals of Mathematical Statistics},
  year         = {1959},
  publisher    = {Institute of Mathematical Statistics},
}

@ARTICLE{dense-detection,
  author={Chen, Jie and Saad, Yousef},
  journal={IEEE Transactions on Knowledge and Data Engineering}, 
  title={Dense Subgraph Extraction with Application to Community Detection}, 
  year={2012},
  volume={24},
  number={7},
  pages={1216-1230},
  doi={10.1109/TKDE.2010.271}}

@inproceedings{qi2022revisiting,
  title={Revisiting the assumption of latent separability for backdoor defenses},
  author={Qi, Xiangyu and Xie, Tinghao and Li, Yiming and Mahloujifar, Saeed and Mittal, Prateek},
  booktitle={The eleventh international conference on learning representations},
  year={2022}
}

@inproceedings{multi-target,
author = {Xu, Jing and Picek, Stjepan},
title = {Poster: Multi-target \& Multi-trigger Backdoor Attacks on Graph Neural Networks},
year = {2023},
isbn = {9798400700507},
publisher = {Association for Computing Machinery},
address = {New York, NY, USA},
doi = {10.1145/3576915.3624387},
booktitle = {Proceedings of the 2023 ACM SIGSAC Conference on Computer and Communications Security},
pages = {3570–3572},
numpages = {3},
location = {Copenhagen, Denmark},
series = {CCS '23}
}

@inproceedings{hou2022graphmae,
  title={GraphMAE: Self-Supervised Masked Graph Autoencoders},
  author={Hou, Zhenyu and Liu, Xiao and Cen, Yukuo and Dong, Yuxiao and Yang, Hongxia and Wang, Chunjie and Tang, Jie},
  booktitle={Proceedings of the 28th ACM SIGKDD Conference on Knowledge Discovery and Data Mining},
  pages={594--604},
  year={2022}
}

@article{disease-diagnosis,
author = {Zhang, Lin and Zhao, Yan and Che, Tongtong and Li, Shuyu and Wang, Xiuying},
title = {Graph neural networks for image-guided disease diagnosis: A review},
journal = {iRADIOLOGY},
volume = {1},
number = {2},
pages = {151-166},
doi = {https://doi.org/10.1002/ird3.20},
eprint = {https://onlinelibrary.wiley.com/doi/pdf/10.1002/ird3.20},
year = {2023}
}

@misc{xia2023waveattackasymmetricfrequencyobfuscationbased,
      title={WaveAttack: Asymmetric Frequency Obfuscation-based Backdoor Attacks Against Deep Neural Networks}, 
      author={Jun Xia and Zhihao Yue and Yingbo Zhou and Zhiwei Ling and Xian Wei and Mingsong Chen},
      year={2023},
      eprint={2310.11595},
      archivePrefix={arXiv},
      primaryClass={cs.CV},
}

@misc{lu2025anywheredoormultitargetbackdoorattacks,
      title={AnywhereDoor: Multi-Target Backdoor Attacks on Object Detection}, 
      author={Jialin Lu and Junjie Shan and Ziqi Zhao and Ka-Ho Chow},
      year={2025},
      eprint={2411.14243},
      archivePrefix={arXiv},
      primaryClass={cs.CR},
}

@misc{blind,
      title={Blind Backdoors in Deep Learning Models}, 
      author={Eugene Bagdasaryan and Vitaly Shmatikov},
      year={2021},
      eprint={2005.03823},
      archivePrefix={arXiv},
      primaryClass={cs.CR},
}

@Inbook{Joyce2011,
author="Joyce, James M.",
editor="Lovric, Miodrag",
title="Kullback-Leibler Divergence",
bookTitle="International Encyclopedia of Statistical Science",
year="2011",
publisher="Springer Berlin Heidelberg",
address="Berlin, Heidelberg",
pages="720--722",
}

@article{MENENDEZ1997307,
title = {The Jensen-Shannon divergence},
journal = {Journal of the Franklin Institute},
volume = {334},
number = {2},
pages = {307-318},
year = {1997},
}

@incollection{harsanyi1982simplified,
  title={A simplified bargaining model for the n-person cooperative game},
  author={Harsanyi, John C},
  booktitle={Papers in game theory},
  pages={44--70},
  year={1982},
  publisher={Springer}
}

@article{flickr,
  author       = {Hanqing Zeng and
                  Hongkuan Zhou and
                  Ajitesh Srivastava and
                  Rajgopal Kannan and
                  Viktor K. Prasanna},
  title        = {GraphSAINT: Graph Sampling Based Inductive Learning Method},
  journal      = {CoRR},
  volume       = {abs/1907.04931},
  year         = {2019},
}

@inproceedings{graphprot,
  title     = {GraphProt: Certified Black-Box Shielding Against Backdoored Graph Models},
  author    = {Yang, Xiao and Lai, Yuni and Zhou, Kai and Li, Gaolei and Li, Jianhua and Zhang, Hang},
  booktitle = {Proceedings of the Thirty-Fourth International Joint Conference on
               Artificial Intelligence, {IJCAI-25}},
  publisher = {International Joint Conferences on Artificial Intelligence Organization},
  editor    = {James Kwok},
  pages     = {619--627},
  year      = {2025}
}

@misc{yang2024distributedbackdoorattacksfederated,
      title={Distributed Backdoor Attacks on Federated Graph Learning and Certified Defenses}, 
      author={Yuxin Yang and Qiang Li and Jinyuan Jia and Yuan Hong and Binghui Wang},
      year={2024},
      eprint={2407.08935},
      archivePrefix={arXiv},
      primaryClass={cs.CR},
      url={https://arxiv.org/abs/2407.08935}, 
}

@inproceedings{cgba,
author = {Xia, Hui and Zhao, Xiangwei and Zhang, Rui and Xu, Shuo and Wang, Luming},
title = {Clean-label graph backdoor attack in the node classification task},
year = {2025},
isbn = {978-1-57735-897-8},
publisher = {AAAI Press},
series = {AAAI'25/IAAI'25/EAAI'25}
}

@inproceedings{cgba2,
author = {Fan, Xuanhao and Dai, Enyan},
title = {Effective Clean-Label Backdoor Attacks on Graph Neural Networks},
year = {2024},
isbn = {9798400704369},
publisher = {Association for Computing Machinery},
series = {CIKM '24}
}


\appendix
\clearpage

\begin{wrapfigure}{r}{0.5\linewidth}
\centering
\begin{subfigure}{0.47\linewidth}
\centering
\includegraphics[width=\linewidth]{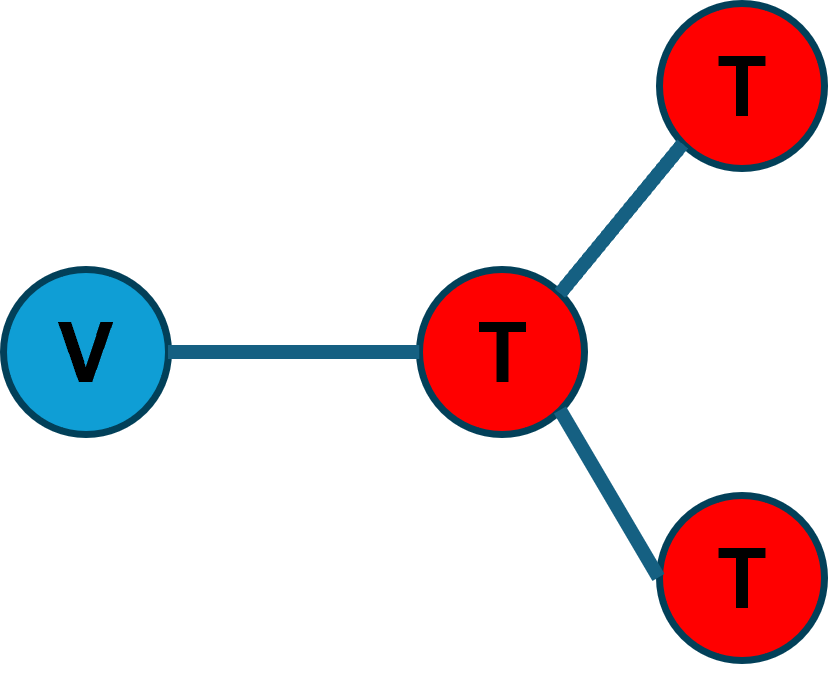}
\caption{Attach Trigger On Cora Dataset}
\label{fig:attach_a}
\end{subfigure}%
\hfill
\begin{subfigure}{0.47\linewidth}
\centering
\includegraphics[width=\linewidth]{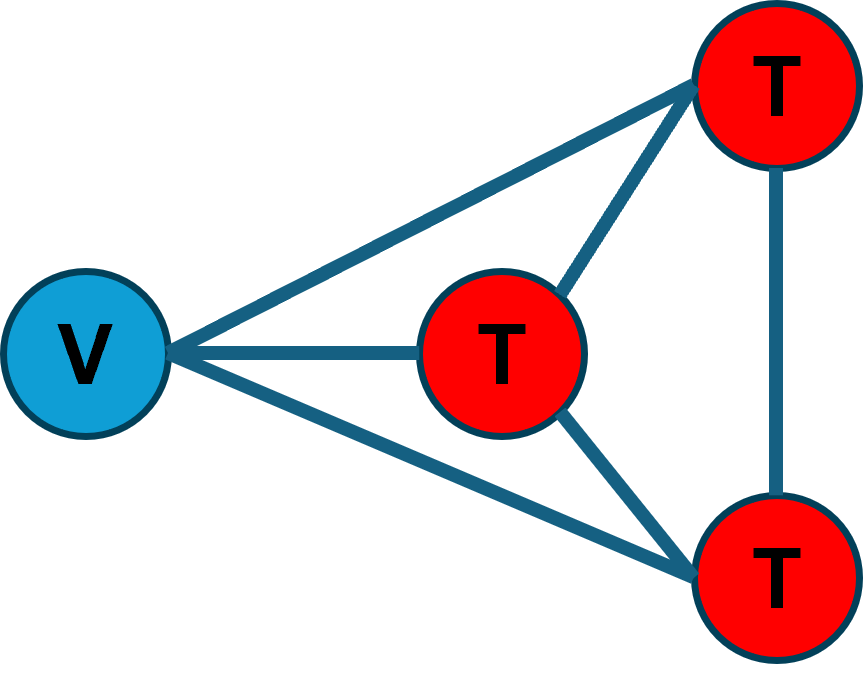}
\caption{Attach Trigger On OGB-arxiv Dataset}
\label{fig:attach_b}
\end{subfigure}
\caption{Example of Attachment Operation. "V" represents victim node, and "T" represents trigger node.}
\label{fig:attach}
\end{wrapfigure}
\section{Attachment Operation Learned by Attacker}
\label{apd:attach}
In Section~\ref{sec:threat_model}, we describe the threat model, where the attacker first trains an attack model to generate triggers based on the input graph and then attaches these triggers to each victim node. Notably, the attachment process is also learned during the training of the attack model. Furthermore, different attack methods learn distinct attachment strategies on the same dataset, and even the same attack method may adopt different attachment operations across datasets to optimize attack performance. For instance, with a trigger size of 3, the DPGBA attack, trained on the Cora dataset, learns to connect the victim node to the generated trigger via a single edge, as illustrated in Figure.~\ref{fig:attach_a}. But in the OGB-arxiv dataset, it establishes three edges to connect the victim node to each trigger node, as shown in Figure.~\ref{fig:attach_b}.

\section{Symbol Table}
\label{apd:symbol}

Here we provide the meaning of symbols used in this paper in Table~\ref{tab:symbol_table}.
\begin{table}[h!]
\centering
\caption{Main symbols used in the paper.}
\label{tab:symbol_table}
\scriptsize
\resizebox{0.5\linewidth}{!}{
\begin{tabular}{c|c}
\hline
Term & Symbol \\
\hline
Attributed Graph & $\mathcal{G}$ \\
Edges & $\mathcal{E}$ \\
Nodes & $\mathcal{V}$ \\
Node Attributes & $\mathcal{X}$ \\
Graph Matrix & $A$ \\
Backdoored Training Graph With Trigger & $\mathcal{G}_T$ \\
Edges in Backdoored Training Graph & $\mathcal{E}_T$ \\
Nodes in Backdoored Training Graph & $\mathcal{V}_T$ \\
Node Attributes in Backdoored Training Graph & $\mathcal{X}_T$ \\
Clean Node Set in Training Graph & $\mathcal{V}_C$ \\
Victim Node Set in Training Graph & $\mathcal{V}_B$ \\
Poison Node Set in Training Graph & $\mathcal{V}_P$ \\
Trigger Node Set in Training Graph & $\mathcal{V}_{Tri}$ \\
Trigger Size of Each Trigger & $TS$ \\
Victim Size of Total Victim Nodes & $VS$ \\
Unseen Backdoored Testing Graph With Trigger & $\mathcal{G}_U$ \\
Edges in Backdoored Testing Graph & $\mathcal{E}_U$ \\
Nodes in Backdoored Testing Graph & $\mathcal{V}_U$ \\
Node Attributes in Backdoored Testing Graph & $\mathcal{X}_U$ \\
Trigger Subgraph & $\mathcal{G}^{Tri}$ \\
Nodes in Trigger Subgraph & $\mathcal{V}^{Tri}$ \\
Edges in Trigger Subgraph & $\mathcal{E}^{Tri}$ \\
Node Attributes in Trigger Subgraph & $\mathcal{X}^{Tri}$ \\
Identified Victim Nodes & $\mathcal{V}_{I-B}$ \\
Identified Trigger Nodes & $\mathcal{V}^{I-Tri}$ \\
\hline
\end{tabular}
}
\vspace{-6pt}
\end{table}

\section{Statistics of Datasets}
\label{apd:dataset}

Here we provide the detailed statistics of our experimental datasets in Table~\ref{tab:dataset}.

\begin{table}[t]
\centering
\caption{Dataset statistics.}
\label{tab:dataset}
\scriptsize
\resizebox{0.7\linewidth}{!}{
\begin{tabular}{ccccc}
\hline
\textbf{Dataset} & \textbf{\#Nodes} & \textbf{\#Edges} & \textbf{\#Features} & \textbf{\#Classes} \\
\hline
Cora      & $2{,}708$   & $10{,}556$    & $1{,}433$ & $7$  \\
PubMed    & $19{,}717$  & $88{,}648$    & $500$     & $3$  \\
OGB-arxiv & $169{,}343$ & $2{,}315{,}598$ & $128$   & $40$ \\
Flickr    & $89{,}250$ &  $899{,}756$            & $500$       & $7$   \\
\hline
\end{tabular}
}
\vspace{-6pt}
\end{table}

\section{Challenges in Graph Backdoor Defense Compared to Image Domain}
\label{apd:challenges}

Defending against backdoor attacks in graph data presents unique and significant challenges compared to traditional domains such as image processing. Unlike image data, where malicious samples with incorrect labels can often be detected through visual inspection~\cite{badnets,blind,Blended}, graph data lacks straightforward indicators. In most graph datasets, discerning the true label of a node solely by observing its topological structure and node features is inherently difficult. This obscurity complicates the identification of poisoned nodes, as adversarial manipulations may not produce immediately obvious anomalies.

Moreover, triggers in graph backdoor attacks exhibit high variability in their characteristics. They can take on any shape or pattern, be located anywhere within the graph, and possess indeterminate sizes. Triggers may manifest as either sparse or dense subgraphs, further complicating their detection. The unpredictable nature of these triggers means that defenders cannot rely on fixed assumptions about their structure or properties. Consequently, methods that focus on detecting specific types of subgraphs, such as dense-subgraph detection techniques~\cite{dense-detection}, become ineffective against a broader range of trigger patterns. This necessitates the development of versatile defense mechanisms capable of handling a diverse array of trigger forms.

Additionally, the flexibility in trigger design allows attackers to craft triggers that are highly stealthy and resilient against existing defense strategies. The ability to embed triggers without significantly altering the overall graph distribution or introducing barely noticeable anomalies requires advanced detection techniques that go beyond superficial inspections of node features or local graph structures. Defenders must therefore prepare to cope with a variety of triggers, employing comprehensive strategies that consider both internal node correlations and external influences within the graph.

These challenges underscore the complexity of developing effective defenses against graph backdoor attacks. The inherent difficulty in identifying malicious manipulations within graph data, combined with the diverse and adaptable nature of potential triggers, necessitates innovative approaches that can robustly detect and mitigate backdoor threats while preserving the integrity and performance of legitimate GNN applications.

\section{Motivation}
\label{apd:motivation}

In our threat model, a backdoor attack is considered successful if it primarily alters a victim node’s prediction when a trigger subgraph is present, which requires the trigger to exert significant influence.
Attackers can achieve this by designing a trigger as either 1) a multi-node subgraph whose collective pattern is unique, or 2) a subgraph containing a few nodes with abnormally high influence, or 3) a strategic combination of both. Indeed, we theoretically prove (Theorem~\ref{thm:dichotomy}, Appendix~\ref{sec:theoretical_framework}) that any non-trivial attack's influence provably manifests in at least one of these two mechanisms. This suggests two detection angles: the first type of trigger can be revealed by its internal correlations, while the second can be identified by its external influence. This leads us to analyze graph backdoors from two perspectives:

\textbf{Internal perspective}. This perspective focuses on the unique patterns within triggers. To manipulate a victim node's classification without affecting others, trigger patterns are often designed to be easily memorized by the model, thereby exhibiting strong statistical correlation. A masked reconstruction task is well suited to capture this internal signal. Accordingly, we use masked reconstruction to quantify it. Consequently, for a well-trained model, masked portions of a trigger should be reconstructed with significantly lower error than parts of benign subgraphs.

\textbf{External perspective}. This perspective examines how triggers influence a victim's prediction. Even when triggers exhibit no obvious internal anomalies, effective triggers tend to induce disproportionate changes in the model’s decision process for the victim (and often trigger's neighborhood). To capture each node’s impact, we perform a counterfactual masking analysis and measure the variance of its neighbors’ logits before and after masking. For nodes that truly belong to a trigger, the resulting variance is significantly larger than that induced by masking benign nodes.

\begin{figure}[t]
    \centering
    \begin{subfigure}[b]{0.33\linewidth}
        \centering
        \includegraphics[width=\textwidth]{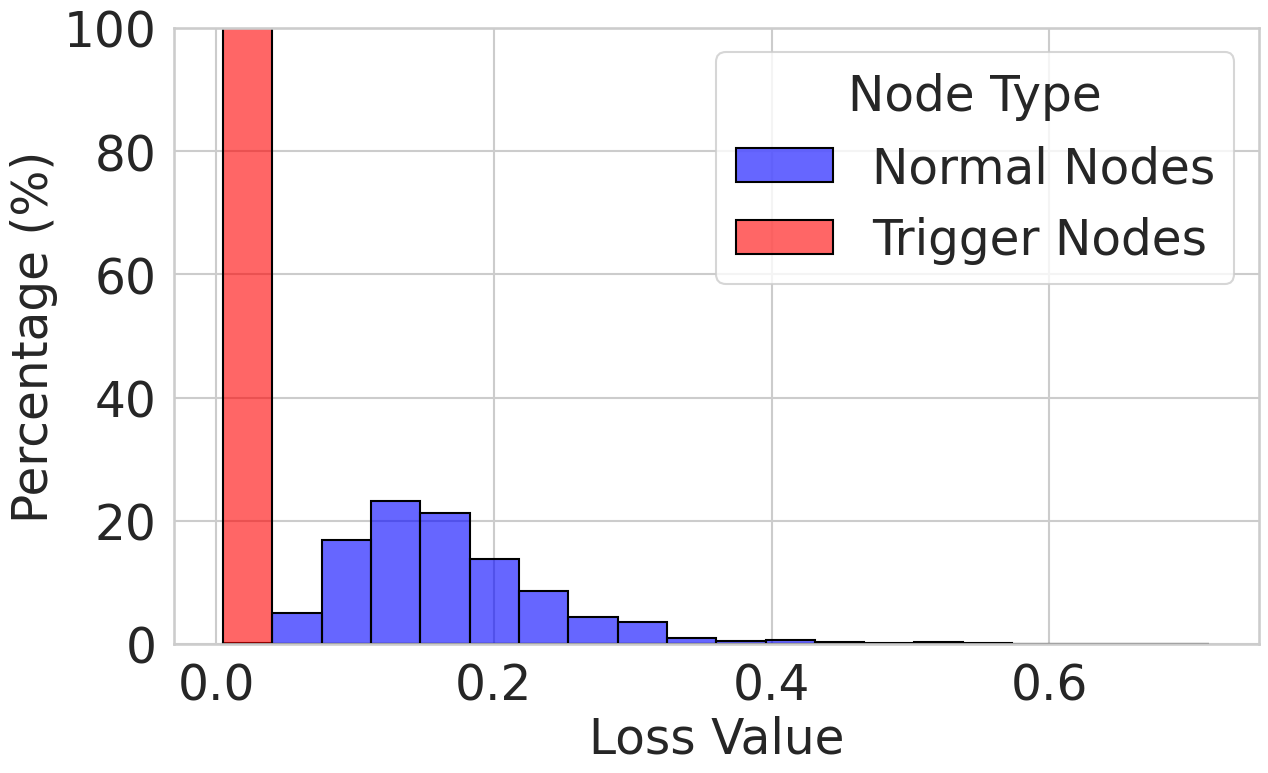}
        \caption{Internal correlation}
        \label{fig:internal}
    \end{subfigure}
    \hspace{0.02\linewidth}
    \begin{subfigure}[b]{0.33\linewidth}
        \centering
        \includegraphics[width=\textwidth]{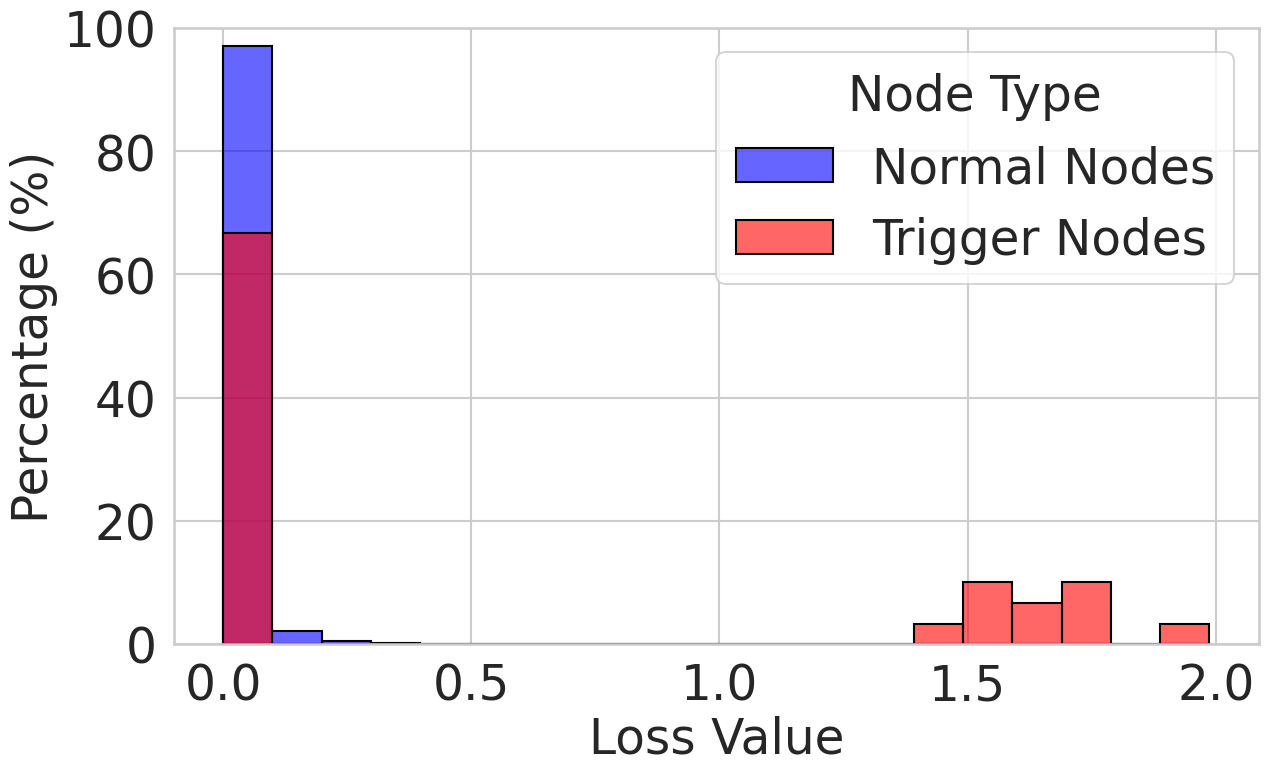}
        \caption{External influence}
        \label{fig:external}
    \end{subfigure}
    \caption{Comparison between normal nodes and trigger nodes.}
    \label{fig:combined_histograms_2}
    \vspace{-6pt}
\end{figure}

We empirically evaluated these two perspectives on the OGB-arXiv dataset using the DPGBA and GTA attacks, each injecting 565 triggers (3 nodes per trigger). To quantify internal correlation, we trained a 2-layer masked autoencoder and measured the reconstruction error on the backdoored graphs. To assess external influence, we trained a 2-layer backdoored GAT and computed the prediction variance induced by iteratively masking individual nodes. As shown in Figure~\ref{fig:combined_histograms}, the results indicate that our two viewpoints are complementary and capture different characteristics of the attacks.

Furthermore, we provide an additional example in Figure~\ref{fig:combined_histograms_2} demonstrating that backdoor triggers remain distinguishable from normal nodes under both perspectives. These results are obtained from the DPGBA attack on the Cora dataset, using 10 triggers (3 nodes each).

The generality of these findings is further supported by experiments across multiple datasets and attack types (see Sections~\ref{sec:defense_performance} and~\ref{sec:node_level}). We also analyze the conditions under which these approaches may fail through detailed ablation studies and security analyses (Sections~\ref{sec:ablation_study} and~\ref{sec:adaptive}).

\subsection{False positives on attack-free graphs}
\label{apd:reported_clean_audit}
Section~\ref{sec:node_level} reports the four-dataset audit in Table~\ref{tab:reported_clean_fp}.

Pruning flagged nodes alone changed CA by 0.00 pp on Cora and PubMed. Further analyses of these two datasets examine the following benign structures.

\noindent\textbf{Hubs and cross-view reinforcement.}
No top-100 degree node was flagged in any seed. Across all five seeds, PubMed had 11 false positives: 10 internal-view-only, one external-view-only, and none reinforced by both views. Three false positives belonged to the broader top-5\% degree set (0.6 per seed).

\noindent\textbf{Class imbalance.}
In PubMed, minority class 0 contains 20.8\% of nodes and had no false positives; classes 1 and 2 had six and five across the seeds, corresponding to FPRs of 0.0155\% and 0.0127\%. The datasets provide class labels but no demographic attributes.

\noindent\textbf{Shifted communities.}
For Louvain communities containing at least 20 nodes, Cora had zero false positives across 22 communities. PubMed had 11 across 34 communities, with only 3/11 in the top shift decile; the most-shifted community's anomaly score was near the graph average. In a controlled stress test, coherently shifting a natural Cora community by $\delta\in\{0,1,2,3,5,8\}$ feature standard deviations produced no overall or within-community false positives at any tested level.

\noindent\textbf{Calibration boundary.}
\system{} does not threshold classifier confidence. Its support threshold admits only label groups containing at least $\tau$ high-score nodes. A larger $\tau$ can reject noisy groups but can also exclude genuine low-budget attacks. One possible deployment rule is to choose the smallest $\tau$ meeting a specified FPR ceiling on clean calibration data and re-audit its FPR after a distribution shift.

\section{Score-to-Node Conversion Details}
\label{apd:score_to_node_details}
\label{apd:deviation}

\noindent\textbf{Overview.}
This section details how \system{} converts fused node-level deviation scores into suspicious victim groups and trigger sets. The conversion has five steps: (i) propagate node scores to labeled training nodes for victim scoring, (ii) place a label-free valley cutoff on the propagated training-node scores, (iii) infer supported label groups from the valley-selected anchor nodes, (iv) construct victim groups with a single-target refinement or group-wise multi-target rule, and (v) recover trigger nodes from the neighborhoods of the retained victim groups.

\noindent\textbf{Victim-score propagation.}
Let $s_i$ be the fused deviation score of node $v_i$. To emphasize large deviations, we set $L^0_i=s_i^r$ with $r\geq3$. We then propagate these scores on a modified adjacency matrix. Let $V_{\mathrm{train}}$ be the training nodes and define $A^{\mathrm{train}}_{ij}=1$ iff $v_i,v_j\in V_{\mathrm{train}}$ and $(v_i,v_j)\in\mathcal{E}_T$. We remove edges between training nodes and add self-loops:
\[
A' = A - A^{\mathrm{train}} + I,
\]
where $A$ is the original adjacency matrix and $I$ is the identity matrix. With $D'$ denoting the degree matrix of $A'$, we propagate
\[
L^{p+1}=D'^{-\frac{1}{2}}A'D'^{-\frac{1}{2}}L^p,\quad p=0,\ldots,P-1.
\]
The propagated training-node score is $\tilde{s}_i=L^P_i$ for $v_i\in V_{\mathrm{train}}$. Removing training-training edges reduces score diffusion among benign training nodes, while self-loops preserve high deviation on suspicious nodes. Propagating $P$ times matches the victim model's message-passing depth, allowing trigger evidence to reach the training victims it influences. 

\noindent\textbf{Largest-gap cutoff and support-based abstention.}
After deviation propagation, each training node receives a propagated score $\tilde{s}_i$. The valley cutoff is motivated by the empirical two-population structure of poisoned graphs: benign nodes tend to concentrate in a low-deviation mode, while poison-related nodes can form a separated upper-deviation mode.

We estimate the valley without labels by sorting the propagated scores in descending order,
\[
\tilde{s}_{(1)}\ge \tilde{s}_{(2)}\ge\cdots\ge\tilde{s}_{(m)}.
\]
To reduce sensitivity to benign long-tail scores, we apply the largest-gap criterion only within the top $10\%$ of the ranked nodes, where suspicious victim/trigger-related nodes are expected to concentrate.
Let $K$ denote the number of nodes in this high-deviation region. We report sensitivity to this search range in Appendix~\ref{app:additional_robustness}. We select
\[
k^\star=\arg\max_{1\le k<K}
\left(\tilde{s}_{(k)}-\tilde{s}_{(k+1)}\right),
\]
and place the cutoff at
\[
\theta_{\mathrm{val}}=
\frac{\tilde{s}_{(k^\star)}+\tilde{s}_{(k^\star+1)}}{2}.
\]
The upper-mode seed set is then
\[
\mathcal{H}=\{v_{(1)},\ldots,v_{(k^\star)}\}.
\]

The resulting seed set must further contain at least one label-supported group, i.e., $\exists c, |\mathcal{H}_c|\ge \tau$.
If no such group exists, \system{} abstains.
We emphasize that this is a practical score-to-node conversion rule rather than a theoretical requirement of the EI/IC decomposition.

\noindent\textbf{Minimum-support label grouping.}
Labels are used only after the label-free anchor set is obtained. For each observed label $c$, we form
\[
\mathcal{H}_c=\{v_i\in\mathcal{H}:y_i=c\}.
\]
We retain label groups with nontrivial anomalous support:
\[
\mathcal{C}_{\mathrm{sus}}=\{c:|\mathcal{H}_c|\ge\tau\},\quad \tau=5.
\]
The threshold $\tau$ is not the anomaly boundary; the anomaly boundary is determined by the valley cutoff above. Instead, $\tau$ is a small anti-noise support filter. A learnable node-level backdoor typically requires repeated trigger-target associations during training, while one or two isolated high-score nodes provide weak evidence for a distinct target-specific backdoor and are more likely to arise from benign nodes near the cutoff. We therefore set $\tau=5$ as a fixed minimum-support rule.

\noindent\textbf{Single-target refinement.}
If only one label group passes the support filter, i.e., $\mathcal{C}_{\mathrm{sus}}=\{\hat{y}_t\}$, we treat it as the standard single-target case and infer $\hat{y}_t$ as the target label. In this case, the valley-selected anchor set identifies the target label, while the final victim boundary further uses label consistency. Specifically, we scan the sorted training-node list $v_{(1)},\ldots,v_{(m)}$ in descending propagated score order and stop after observing $S=3$ consecutive nodes whose labels differ from $\hat{y}_t$. Let
\[
j^\star=\min\{j\ge S: y_{(j-S+1)}\ne\hat{y}_t,\ldots,y_{(j)}\ne\hat{y}_t\},
\]
with $j^\star=m+1$ if no such index exists. The identified victim group is
\[
\mathcal{V}^{I\!-\!Vic}_{\hat{y}_t}
=\{v_{(j)}:1\le j<j^\star,\ y_{(j)}=\hat{y}_t\}.
\]
This refinement does not assume the target label in advance; the label is inferred from the valley-selected anchor set. We report the sensitivity to $S$ in Appendix~\ref{app:additional_robustness}.

\noindent\textbf{Multi-target case.}
If multiple label groups pass the support filter, we do not apply the single-target stopping rule. Instead, each supported group is treated as one suspicious target group,
\[
\mathcal{V}^{I\!-\!Vic}_{c}=\mathcal{H}_c,\quad c\in\mathcal{C}_{\mathrm{sus}}.
\]
This avoids assuming a single target label and keeps the procedure applicable to multi-target or ALL2ALL-style settings.

\noindent\textbf{Trigger recovery.}
For each retained victim group $\mathcal{V}^{I\!-\!Vic}_c$, we search its $P$-hop neighborhood in the original poisoned graph $\mathcal{G}_T$; the edge-removed graph above is used only for score propagation. Let $\mathcal{N}_P(\cdot;\mathcal{G}_T)$ denote the $P$-hop neighborhood in $\mathcal{G}_T$. Candidate trigger nodes are
\[
\mathcal{B}_c=\mathcal{N}_P(\mathcal{V}^{I\!-\!Vic}_c;\mathcal{G}_T)\setminus
\bigcup_{c'\in\mathcal{C}_{\mathrm{sus}}}\mathcal{V}^{I\!-\!Vic}_{c'}.
\]
We retain high-deviation candidates using the victim-group boundary:
\[
\mathcal{V}^{I\!-\!Tri}_c=
\{v_i\in\mathcal{B}_c: L^P_i\geq
\min_{v_j\in\mathcal{V}^{I\!-\!Vic}_c}L^P_j\}.
\]
This recovers unlabeled trigger nodes by anchoring them to the suspicious labeled victims they influence, without assuming the trigger size.

\noindent\textbf{Group-wise unlearning.}
If only one label group remains, the procedure reduces to the standard single-target case. If multiple groups remain, \system{} applies the same trigger recovery and unlearning objective to each group. Thus multi-target and ALL2ALL-style settings are handled as multiple suspicious label groups; the defender does not need to know the number of target labels or the source-to-target mapping.

\section{The details of Attacks}
\label{apd:attack_details}

\noindent\textbf{SBA}~\cite{SBA}: SBA employs randomly generated graphs as triggers in graph backdoor attacks. It leverages the Erdős-Rényi (ER) model~\cite{Gilbert1959RandomGraphs} to create random subgraph structures with a limited number of nodes. Each trigger node is assigned random features, effectively embedding the trigger subgraph into the target graph without introducing discernible patterns.

\noindent\textbf{GTA}~\cite{GTA}: GTA utilizes a trigger generator that crafts subgraphs as triggers tailored to individual samples. The trigger generator is optimized exclusively based on the backdoor attack loss, disregarding any constraints related to trigger detectability. This allows GTA to produce highly effective, sample-specific triggers that enhance the attack success rate.

\noindent\textbf{UGBA}~\cite{UGBA}: UGBA selects representative and diverse nodes as poisoned nodes to fully utilize the attack budget. It employs an adaptive trigger generator that is optimized with a constraint loss, ensuring that the generated triggers closely resemble the target nodes. This similarity helps UGBA evade detection while maximizing the effectiveness of the backdoor attack.

\noindent\textbf{DPGBA}~\cite{DPGBA}: DPGBA introduces an adversarial learning strategy to generate in-distribution triggers. It proposes a novel loss function that guides the trigger generator to produce efficient in-distribution triggers, making them indistinguishable from legitimate subgraphs. This approach enhances the stealthiness and effectiveness of the backdoor attacks by ensuring that the triggers blend seamlessly with the graph’s natural structure.

\section{The details of Defenses}\label{apd:defense_details}

\noindent\textbf{Baseline tuning and reporting protocol.}
For a fair comparison, we carefully tune each baseline under the same training-time defense protocol. 
When official implementations and recommended hyperparameters are available, we follow the original papers and further tune their key hyperparameters on the validation set. 
When official configurations are not available, we perform a validation-based grid search over the main hyperparameters that control defense strength. 
For each baseline, we report the best validation-selected configuration, and the corresponding test performance is reported in the main results. 
No test labels, test-time trigger information, or attack-specific test statistics are used for tuning either \system{} or the baselines. 

\noindent\textbf{GNNGuard}~\cite{GNNGuard}: GNNGuard is a robust defense method for GNNs that protects against adversarial attacks by leveraging node similarity to filter out adversarial edges. It employs a multi-stage defense strategy that dynamically adjusts edge weights during training, thereby enhancing the model’s resilience to structural perturbations.

\noindent\textbf{SimilarityPrune (SP)}~\cite{UGBA}: SP is designed to prune edges that link nodes with low similarity, effectively removing triggers that significantly differ from clean nodes. By eliminating these dissimilar connections, SP reduces the impact of malicious subgraphs on the GNN’s predictions.

\noindent\textbf{Outlier Detection (OD)}~\cite{DPGBA}: OD introduces a graph auto-encoder to identify and filter out nodes with high reconstruction loss, which are considered outliers compared to clean nodes. By removing these outlier nodes, OD effectively eliminates triggers that deviate from the normal node distribution.

Unlike OD's high reconstruction error on unmasked inputs, our internal score uses low reconstruction error after hiding the candidate feature and reconstructing it from context.

\noindent\textbf{RIGBD}~\cite{RIGBD}: RIGBD is a robust defense method for GNNs that identifies victim nodes by detecting high prediction variance through random edge dropping. Once identified, RIGBD applies robust learning techniques to these victim nodes to counteract the influence of trigger patterns, thereby mitigating the impact of backdoor attacks.

\section{The Calculation of ASR and CA}\label{apd:asr_ca_calculations}

\subsection{Attack Success Rate (ASR)}
ASR measures the proportion of backdoored victim nodes in the testing set that are incorrectly classified into the target class when the trigger is present.

\[
\text{ASR} = \left( \frac{\text{\# of victim nodes classified as target class}}{\text{Total \# of filtered victim nodes}} \right) 
\]
where "filtered victim nodes" means the victim nodes that are not in the target class.

\subsection{Clean Accuracy (CA)}

CA assesses the GNN’s performance on clean, non-backdoored samples, indicating its ability to maintain accuracy on legitimate data.
\[
\text{CA} = \left( \frac{\text{Number of clean samples correctly classified}}{\text{Total number of clean samples}} \right) 
\]

\section{Comparison Results Analysis}
\label{app:detailed_analysis}
Compared with GNNGuard, which down-weights dissimilar neighbors via cosine similarity during training, \system{} learns normal patterns on the filtered graph and suppresses the target label on identified victims and triggers (Section~\ref{sec:repair_objective}). This leads to lower average ASR and higher average CA (Table~\ref{tab:results}).

OD filters nodes with high reconstruction error, while SP prunes edges between nodes with low feature similarity. Compared to OD and SP, \system{} accurately identifies both victim and trigger nodes. Moreover, instead of merely detecting these nodes, \system{} leverages the identified ones for robust learning, resulting in superior performance.

In comparison with RIGBD, which identifies only victim nodes, \system{} detects both victim and trigger nodes. This dual identification allows us to act against backdoor attacks directly (rather than only via victim proxies), face the attacker pattern more directly, and provide more samples covering the attack patterns as supervision for unlearning during robust learning. In addition, explicitly handling triggers helps stabilize neighborhoods when a trigger connects to multiple nodes, reducing residual influence at test time. Together, these factors reduce adversarial pattern insertion and enhance defense performance (lower ASR without sacrificing CA).

\noindent\textbf{The possible reason for the performance gap of RIGBD:} We also observe a gap relative to the results reported in the RIGBD paper, which appears to stem from RIGBD’s sensitivity to attack strength: it is effective under higher budgets but brittle under subtler settings. For example, under UGBA, changing the configuration from \textit{victim\_size=40}, \textit{prune\_thr=0.5} to \textit{victim\_size=10}, \textit{prune\_thr=0.1} leads to detection failure (ASR changes from 0.44\% to 92.44\%). This failure under high-budget settings likely stems from RIGBD's reliance on prediction variance under edge perturbation. RIGBD assumes that removing adversarial edges will cause significant fluctuations in model output. However, in high-intensity attacks (e.g., dense trigger subgraphs), the adversarial connectivity is highly redundant. Even if a subset of edges is pruned during RIGBD's sensitivity check, the remaining edges maintain sufficient information flow to preserve the attack pattern. Consequently, the model's prediction remains low variance, missing the detection of victim nodes.

\vspace{-10pt}
\section{Detailed Ablation Analysis}\label{app:ablation_ic_ex_analysis}
In this section, we provide a granular analysis of the ablation study results presented in Section~\ref{sec:ablation_study}, focusing on how specific attack characteristics affect each component.

\textbf{Analysis of \system$\backslash$E (Reliance on Internal Correlation):}
The capability of \system$\backslash$E is significantly impaired under attacks that deploy a single trigger node per victim (e.g., GTA and UGBA) on the OGB-arxiv and PubMed datasets. When a lone trigger node is removed or masked, the structural integrity of the potential trigger subgraph collapses, effectively obscuring the backdoor signature. This reveals a critical limitation: \textit{internal correlation relies on the presence of multi-node substructures and becomes less effective against sparse, single-node triggers.}

\textbf{Analysis of \system$\backslash$I (Reliance on External Influence):}
\system$\backslash$I faces challenges in distinguishing trigger nodes from benign nodes under attacks that generate highly robust, dense trigger patterns (e.g., DPGBA and SBA). For instance, DPGBA constructs fully connected trigger subgraphs (cliques). These patterns remain intact through message passing even when individual nodes are masked, maintaining their adversarial influence. \textit{This resilience highlights the inadequacy of relying solely on external influence against structurally robust triggers.}

\textbf{Synergy in \system{}:}As shown in Table~\ref{tab:node_level}, the full \system{} model overcomes these individual limitations. By combining external checks (effective against sparse triggers) with internal correlation (effective against structural anomalies), \system{} achieves the necessary synergy to identify nearly all trigger nodes across diverse attack strategies.

\section{Addtional Ablation Study}
\label{apd:ablation_details}
\noindent\textbf{Weighted deviation score.}
Next, we examine our weighted-sum approach to combining internal-correlation and external-influence scores. Concretely, we run the victim/trigger identification procedure separately using $\mathcal{S}_{\mathrm{int}}$ and $\mathcal{S}_{\mathrm{ext}}$, obtaining $(N_v^{\mathrm{int}},N_t^{\mathrm{int}})$ and $(N_v^{\mathrm{ext}},N_t^{\mathrm{ext}})$ as the numbers of detected potential victims and triggers. We use $N_vN_t$ as a reliability proxy and compute:
\begin{equation}
\mathcal{S}
=
\big(N_v^{\mathrm{int}} N_t^{\mathrm{int}}\big)\,\mathcal{S}_{\mathrm{int}}
\;+\;
\big(N_v^{\mathrm{ext}} N_t^{\mathrm{ext}}\big)\,\mathcal{S}_{\mathrm{ext}}.
\end{equation}
The raw scores are used without normalization. We introduce \system$\sim$1:1, which assigns equal weights to both scores. We focus on the PubMed dataset for all attack types to highlight differences.
\begin{table}[t]
\centering
\caption{Comparison of weighting schemes.}
\label{tab:weight}
\scriptsize
\resizebox{0.82\linewidth}{!}{
\begin{tabular}{c|c|cc|cc|cc} 
\hline
\multirow{3}{*}{Method} & \multirow{3}{*}{Attack} 
& \multicolumn{2}{c|}{\multirow{2}{*}{Defense}} 
& \multicolumn{4}{c}{Node Detection} \\
\cline{5-8}
& & \multicolumn{2}{c|}{} 
& \multicolumn{2}{c|}{Victim Nodes} 
& \multicolumn{2}{c}{Trigger Nodes} \\
& & ASR(\%) & CA(\%) & P(\%) & R(\%) & P(\%) & R(\%) \\
\hline
\multirow{4}{*}{\makecell{\system\\$1:1$}}   
& SBA   & 8.24 & 85.11 & 80.00  & 31.00  & 65.60  & 20.58 \\
& GTA   & 0.85 & 83.26 & 100.00 & 61.50  & 94.85  & 61.50 \\
& UGBA  & 2.25 & 84.41 & 100.00 & 98.00  & 100.00 & 98.00 \\
& DPGBA & 9.55 & 84.18 & 84.56  & 50.50  & 93.61  & 44.83 \\
\hline
\multirow{4}{*}{\makecell{\system\\(Ours)}}   
& SBA   & 0.91 & 85.21 & 100.00 & 91.50  & 97.60  & 70.33 \\
& GTA   & 0.27 & 83.18 & 96.79  & 75.55  & 96.68  & 75.55 \\
& UGBA  & 3.67 & 84.47 & 100.00 & 94.50  & 100.00 & 94.50 \\
& DPGBA & 2.24 & 84.06 & 97.10  & 100.00 & 100.00 & 100.00 \\
\hline
\end{tabular}
}
\vspace{-6pt}
\end{table}
\begin{table}[t]
\centering
\caption{Comparison of different selection methods.}
\label{tab:top_K}
\scriptsize
\setlength{\tabcolsep}{3pt}
\renewcommand{\arraystretch}{0.9}
\resizebox{0.82\linewidth}{!}{
\begin{tabular}{c|cc|cc|cc} 
\hline
\multirow{3}{*}{Selection Percentage} 
& \multicolumn{2}{c|}{\multirow{2}{*}{Defense}} 
& \multicolumn{4}{c}{Node Detection} \\ 
\cline{4-7}
& \multicolumn{2}{c|}{} 
& \multicolumn{2}{c|}{Victim Nodes} 
& \multicolumn{2}{c}{Trigger Nodes} \\ 
\cline{2-7}
& ASR(\%) & CA(\%) & P(\%) & R(\%) & P(\%) & R(\%) \\ 
\hline
Top 1\%  & $65.84$ & $80.93$ & $19.43$ & $97.25$ & $35.15$ & $91.29$ \\
Top 3\%  & $66.11$ & $79.01$ & $5.53$  & $98.38$ & $12.25$ & $94.50$ \\
Top 5\%  & $42.29$ & $79.86$ & $3.45$  & $99.00$ & $7.59$  & $96.83$ \\
Top 7\%  & $46.24$ & $79.68$ & $2.65$  & $99.13$ & $5.43$  & $97.08$ \\
Top 10\% & $40.63$ & $80.26$ & $2.27$  & $99.38$ & $3.90$  & $98.88$ \\
Top 15\% & $34.95$ & $80.83$ & $2.05$  & $99.38$ & $2.53$  & $98.88$ \\ 
\hline
\system{} (Ours) & $1.03$ & $84.24$ & $95.39$ & $91.63$ & $95.92$ & $85.88$ \\
\hline
\end{tabular}
}
\vspace{-6pt}
\end{table}

Table~\ref{tab:weight} shows that: (1) Except for the UGBA attack, where \system$\sim$1:1 performs slightly better than \system{}, our primary model \system{} generally achieves superior performance by maintaining lower ASR and higher CA. (2) \system{} consistently exhibits higher precision and recall in identifying both victim and trigger nodes. These results indicate that the merit of dynamically assigning weights to internal-correlation and external-influence scores can intrinsically suppress noise, helping pinpoint poisoned nodes more accurately and enhancing overall defense.



\noindent\textbf{Node selection.}
\label{sec:top_k}
We also evaluate our method for selecting victim and trigger nodes via a variant, \system$\backslash$S, which regards the top $K\%$ most suspicious trigger nodes and designates their connected nodes as victim nodes. Due to space limits, we only show results on the PubMed dataset.

We observe in Table~\ref{tab:top_K} that as $K$ grows, recall for both victim and trigger nodes remains above $90\%$. However, precision drops due to increased false positives among benign nodes. ASR remains at least $34.95\%$, compared with $1.03\%$ for \system{}, and CA is over 3.3 percentage points lower at every $K$.

\subsection{Repair with incomplete localization}
\label{apd:reported_repair}
The repair ablation is rerun independently of the main end-to-end experiments under the same configurations. Within this ablation, localization is held fixed when comparing removal-only retraining, additional victim-side unlearning, and additional victim-and-trigger unlearning (Table~\ref{tab:reported_repair}). Each entry gives the mean ASR and CA across attacks.

\begin{table}[htbp]
\centering
\caption{Fixed-localization repair ablation. Entries are ASR / CA (\%).}
\label{tab:reported_repair}
\begin{tabular}{lccc}
\toprule
Dataset & Removal only & + Victim & + Victim and trigger\\
\midrule
Cora & 62.4 / 79.8 & 4.7 / 81.7 & 0.8 / 81.9\\
PubMed & 70.4 / 83.2 & 2.8 / 84.1 & 0.8 / 84.1\\
OGB-arxiv & 65.5 / 64.5 & 0.0 / 65.8 & 0.0 / 65.9\\
\bottomrule
\end{tabular}
\end{table}

Detected victim-trigger groups provide target-conditioned negative supervision on the poisoned graph; their use in repair does not require removing every injected node.

For the two PubMed cases with trigger recall below 80\%, SBA has trigger precision/recall 97.60\%/70.33\%, ASR 28.57\% before and 0.91\% after defense, and defended CA 85.21\%. GTA has trigger precision/recall 96.68\%/75.55\%, ASR 94.45\% before and 0.27\% after, and defended CA 83.18\%. Disconnected trigger components with only one node interacting directly with the victim can weaken internal evidence; victims whose original labels already match the target can show modest prediction shifts.

\section{Hyperparameter Study}\label{app:additional_robustness}

\begin{table}[t]
\caption{The Performance of Different Architectures}
\centering
\resizebox{\linewidth}{!}{
\begin{tabular}{c|c|cc|cc|cc|cc} 
\hline
\multirow{3}{*}{Attack} & \multirow{3}{*}{Architecture} & \multicolumn{2}{c|}{\multirow{2}{*}{No Defense}} & \multicolumn{2}{c|}{\multirow{2}{*}{Defense}} & \multicolumn{4}{c}{Node Detection}                  \\ 
\cline{7-10}
                        &                                    & \multicolumn{2}{c|}{}                            & \multicolumn{2}{c|}{}                                     & \multicolumn{2}{c|}{Victim Nodes} & \multicolumn{2}{c}{Trigger Nodes}  \\ 
\cline{3-6}
                        &                                    & ASR(\%)  & CA(\%)                                & ASR(\%) & CA(\%)                                          & P(\%) & R(\%)        & P(\%) & R(\%)         \\ 
\hline
\multirow{3}{*}{DPGBA}  & GAT                                & $94.35$ & $80.67$                              & $0.17$ & $81.85$                                        & $91.21$      & $100.00$             & $100.00$         & $100.00$              \\
                        & GCN                                & $94.52$ & $81.26$                              & $0.52$ & $83.56$                                        & $91.99$      & $92.00$              & $100.00$         & $92.00$               \\
                        & GIN                                & $94.52$ & $81.41$                              & $0.70$ & $79.41$                                        & $68.61$      & $100.00$             & $69.35$      & $97.33$           \\ 
\hline
\multirow{3}{*}{UGBA}   & GAT                                & $94.13$ & $79.19$                              & $0.89$ & $81.93$                                        & $98.18$      & $100.00$             & $93.33$      & $100.00$              \\
                        & GCN                                & $89.60$ & $79.93$                              & $4.44$ & $83.63$                                         & $98.18$      & $100.00$             & $94.85$      & $100.00$              \\
                        & GIN                                & $93.24$ & $79.93$                              & $5.42$ & $83.41$                                        & $96.36$      & $98.00$              & $94.85$      & $98.00$               \\ 
\hline
\multirow{3}{*}{GTA}    & GAT                                & $97.87$ & $75.85$                              & $0.44$ & $76.89$                                        & $98.00$          & $86.00$              & $94.67$      & $86.00$               \\
                        & GCN                                & $98.44$ & $75.19$                              & $0.00$     & $76.15$                                        & $98.18$      & $90.00$              & $96.36$      & $90.00$               \\
                        & GIN                                & $98.40$   & $75.41$                              & $0.00$     & $71.93$                                        & $94.55$      & $100.00$             & $90.00$          & $100.00$              \\ 
\hline
\multirow{3}{*}{SBA}    & GAT                                & $49.78$ & $83.11$                              & $0.09$ & $84.30$                                        & $100.00$         & $100.00$          & $100.00$         & $83.72$           \\
                        & GCN                                & $51.91$ & $83.85$                              & $9.33$ & $84.18$                                        & $60.90$      & $52.00$              & $69.09$      & $34.00$               \\
                        & GIN                                & $46.84$ & $83.63$                               & $0.00$     & $83.78$                                        & $92.32$      & $96.00$              & $86.36$      & $72.00$               \\
\hline
\end{tabular}
}
\label{tab:architecture}
\end{table}

\begin{table}[t]
\caption{Highest-ranked non-victim node position}
\centering
\small
\begin{tabular}{lccc}
\hline
\textbf{Metric\textbackslash Dataset} & \textbf{Cora} & \textbf{PubMed} & \textbf{OGB-Arxiv} \\
\hline
Average highest rank & 10.55 & 37.45 & 520.25 \\
Total victim nodes                      & 10    & 40    & 565    \\
\hline
\end{tabular}
\label{tab:highest_non_target_label_rank}
\end{table}
\begin{table}[t]
\centering
\caption{Sensitivity to the stop criterion.}
\label{tab:stop_criterion_sensitivity}
\scriptsize
\setlength{\tabcolsep}{4pt}
\renewcommand{\arraystretch}{1}
\resizebox{0.70\linewidth}{!}{
\begin{tabular}{c|cc|cc|cc}
\hline
\multirow{2}{*}{\#Different Labels ($S$)} 
& \multicolumn{2}{c|}{Defense} 
& \multicolumn{2}{c|}{Victim Nodes} 
& \multicolumn{2}{c}{Trigger Nodes} \\
& ASR(\%) & CA(\%) & P(\%) & R(\%) & P(\%) & R(\%) \\
\hline
1 & 1.60 & 77.24 & 95.71 & 91.50 & 100.00 & 85.51 \\
2 & 0.73 & 77.16 & 93.72 & 93.00 & 98.57 & 87.66 \\
3 & 0.93 & 77.46 & 91.36 & 94.50 & 98.57 & 89.17 \\
4 & 1.22 & 76.89 & 91.36 & 94.50 & 98.57 & 89.17 \\
5 & 0.79 & 77.10 & 90.61 & 94.50 & 98.57 & 88.67 \\
\hline
\end{tabular}
}
\vspace{-6pt}
\end{table}
\begin{table}[t]
\centering
\caption{Sensitivity to the cutoff search range \(p\).}
\label{tab:p_sensitivity}
\scriptsize
\setlength{\tabcolsep}{4pt}
\renewcommand{\arraystretch}{1}
\resizebox{0.70\linewidth}{!}{
\begin{tabular}{c|cc|cc|cc}
\hline
\multirow{2}{*}{Cutoff Range \(p\)} 
& \multicolumn{2}{c|}{Defense} 
& \multicolumn{2}{c|}{Victim Nodes} 
& \multicolumn{2}{c}{Trigger Nodes} \\
& ASR(\%) & CA(\%) & P(\%) & R(\%) & P(\%) & R(\%) \\
\hline
3\%--10\% & 0.00 & 78.12 & 97.0 & 100.0 & 96.8 & 100.0 \\
\hline
\end{tabular}
}
\vspace{-6pt}
\end{table}
\noindent\textbf{Model Architecture.}
We observe that our defense system, \system{}, utilizes a 2-layer GAT architecture, while both the attacker model and the benign model are based on a 2-layer GCN architecture. 

The decision to use a GAT architecture in our defense is motivated by its ability to assign different weights to each neighbor, thereby giving higher importance to nodes that significantly contribute to the mask prediction task. This capability allows GAT to better capture relevant information within the context of our defense. However, the architectural difference between the defense model and the attacker’s model may raise concerns about whether this inconsistency could influence the performance of \system.

To address this concern, we conducted two additional experiments for each attack on the Cora dataset to evaluate whether the architecture mismatch affects the performance.

The results, presented in Table~\ref{tab:architecture}, indicate that \system{} consistently achieves strong performance in accurately identifying both victim and trigger nodes, with averaged precision and recall values exceeding 80\%. Moreover, \system{} effectively defends against all attacks across model architectures, with post-defense ASR at most 9.33\% in all 12 settings.


\noindent\textbf{Sensitivity to Stop Criterion}
We stop the victim node selection when we encounter $S$ consecutive nodes with labels different from the target label. As shown in Table~\ref{tab:highest_non_target_label_rank} (which shows the rank of the highest non-victim node), the highest rank of a non-victim node is typically far from the majority of victim nodes. This separation already indicates that a small $S$ (e.g., $k{=}2$) is unlikely to cause the selection to stop prematurely and miss true victim nodes.

To further provide direct evidence for this robustness, we vary $S$ from 1 to 5 on Cora datasets. Table~\ref{tab:stop_criterion_sensitivity} presents the averaged results over all attacks, which show that \system{}'s performance is highly stable and not sensitive to this hyperparameter.

As $S$ varies, all key metrics remain remarkably stable: ASR stays within $0.73$--$1.60\%$ and CA within $76.89$--$77.46\%$. Victim precision and recall remain high (ranging from $90.61$--$95.71\%$ and $91.5$--$94.5\%$, respectively), while trigger precision and recall are also stable ($98.57$--$100\%$ and $85.51$--$89.17\%$). This stability indicates that \system{} is robust to the specific choice of $S$.

By default, we set $S{=}3$. Its victim precision and recall are $91.36\%$ and $94.50\%$, respectively.

\noindent\textbf{Sensitivity to cutoff-related hyperparameters.}
We evaluate the sensitivity of the score-to-node conversion on Cora dataset to the cutoff search range \(p\).
As shown in Table~\ref{tab:p_sensitivity}, varying \(p\) from 3\% to 10\% does not change the defense outcome.
The support threshold \(\tau\) is used only after the label-free valley cutoff to remove isolated noisy seeds; it does not determine the anomaly boundary.
The direct support-count sweep below identifies a stable range and the low-budget failure boundary. We fix \(\tau=5\) across datasets and attacks.
These results suggest that the cutoff is driven by a stable score separation rather than by carefully tuned post-processing parameters.

\subsection{Direct sweeps of the reconstruction exponent and support count}
\label{apd:reported_parameters}
The defaults $\gamma=3$ and $\tau=5$ follow their different roles: $\gamma$ controls nonlinear contrast among masked-reconstruction residuals, while $\tau$ removes isolated label groups after the anomaly boundary has already been selected. Table~\ref{tab:reported_parameter_sweeps} reports mean results across DPGBA, GTA, SBA, and UGBA on Cora. Undefended ASR is 84.10\%. We include $\gamma=1$ as a boundary test of the default $\gamma>1$ setting.

\begin{table}[htbp]
\centering
\caption{Cora parameter sweeps: four-attack means (\%). P/R denote trigger precision/recall.}
\label{tab:reported_parameter_sweeps}
\begin{tabular}{lrrr|lrrr}
\toprule
$\gamma$ & ASR & P & R & $\tau$ & ASR & P & R\\
\midrule
1 & 10.10 & 100.00 & 35.00 & 1 & 1.22 & 84.83 & 80.00\\
3 (default) & 0.40 & 97.00 & 92.43 & 3 & 0.40 & 97.00 & 92.43\\
5 & 0.40 & 97.00 & 92.43 & 5 (default) & 0.40 & 97.00 & 92.43\\
7 & 0.40 & 97.00 & 92.43 & 7 & 0.40 & 97.00 & 92.43\\
9 & 0.40 & 97.00 & 92.43 & 9 & 0.40 & 97.00 & 92.43\\
11 & 4.44 & 90.64 & 80.20 & 10 & 22.43 & 82.25 & 81.68\\
15 & 19.61 & 46.43 & 45.00 & 11 & 46.19 & 49.33 & 50.00\\
\bottomrule
\end{tabular}
\end{table}

Both sweeps show a plateau at $3,5,7,9$. At $\gamma=1$, weak nonlinear contrast reduces separation and trigger recall; large $\gamma$ can destabilize inverse-score ranking through numerical underflow. At $\tau=1$, isolated noisy or hub-associated label groups can survive, invoke multi-target handling, and bypass single-target victim expansion. Conversely, $\tau\ge10$ can exceed genuine support in low-budget Cora GTA/UGBA settings and skip unlearning. This boundary depends on scale: the reported OGB-arxiv runs retain sufficient support to avoid this failure even at $\tau=100$.

\section{Detailed Analysis and Additional Adaptive Attacks}
\label{apd:additional_security_analysis}

\subsection{White-Box Attacks}
\label{sec:baseline}

In this section, we first consider a scenario where the attacker only knows partial information about our defense (e.g., model architecture). Therefore, we introduce a sophisticated variant of the DPGBA attack, $DPGBA_{Partial}$, which aims to hide final output representations of trigger nodes. This is achieved by training a \textit{supervised} GAT model—sharing the same architecture as our defense—to distinguish between trigger nodes and benign nodes based on their final representations.

We further consider a scenario where the attacker fully replicates our defense model and uses it to craft triggers, resulting in $DPGBA_{Full}$. Here, the supervised GNN from $DPGBA_{Partial}$ is replaced by our defense model, allowing the attacker to directly adapt the trigger generator to the defense model’s behavior. This provides a stricter test of our defense under the assumption that attackers possess complete knowledge of our defense details. The trigger-node set is fixed at initialization; the adversarial model optimizes node features and intra-trigger edges during training. The Full search repeatedly runs localization and repair and selects candidates using post-repair ASR.

We present the results in Table~\ref{tab:adaptive}. For both attacks, \system{} identifies trigger nodes and victim nodes with precision and recall above 90\%, except for trigger recall under $DPGBA_{Full}$ (86\%). Post-defense ASR is 2.96\% for $DPGBA_{Partial}$ and 0.00\% for $DPGBA_{Full}$. For $DPGBA_{Full}$, CA drops from 71.56\% to 68.07\% after defense. A model trained on a clean graph and tested on the poisoned graph also has low CA (66.15\%).

The effectiveness of \system{} can be attributed to two key factors:  
(1)~\textit{Targeting Internal Correlations and External Influences.} \system{} emphasizes both internal correlations among trigger nodes and their external influence, making mere manipulation of final representations insufficient to hide trigger structures (even in $DPGBA_{Partial}$).  
(2)~\textit{Deviation Score Propagation.} In $DPGBA_{Full}$, each trigger subgraph contains three trigger nodes. Although two nodes actually evade initial detection, the third node consistently exhibits strong correlations, allowing deviation propagation to recover associated trigger candidates. Together with repair, this reduces ASR to 0.00\%.

\noindent\textbf{Conclusion:} \system{} is still effective even if the attacker knows the details of our defense models.

\begin{table}[t]
\caption{White-box Attacks}
\centering
\resizebox{\linewidth}{!}{
\begin{tabular}{c|cc|cc|cccc} 
\hline
\multirow{3}{*}{Attack} & \multicolumn{2}{c|}{\multirow{2}{*}{No Defense}} & \multicolumn{2}{c|}{\multirow{2}{*}{Defense}} & \multicolumn{4}{c}{Node Detection}                 \\ 
\cline{6-9}
                        & \multicolumn{2}{l|}{}                            & \multicolumn{2}{l|}{}                                     & \multicolumn{2}{l}{Victim Nodes} & \multicolumn{2}{l}{Trigger Nodes}  \\ 
\cline{2-5}
                        & ASR(\%)  & CA(\%)                                & ASR(\%) & CA(\%)                                          & P(\%) & R(\%)       & P(\%) & R(\%)         \\ 
\hline
Partial              & $72.26$ & $80.37$                               & $2.96$ & $83.19$                                        & $91.21$      & $100.00$         & $98.11$      & $100.00$           \\
Full                & $98.09$ & $71.56$                               & $0.00$ & $68.07$                                        & $92.55$      & $94.00$          & $100.00$     & $86.00$            \\
\hline
\end{tabular}
}

\label{tab:adaptive}
\end{table}

\subsection{One-Node Clean Label Attack}
\label{sec:clean_label}

We are also aware of recent natively clean-label graph backdoor methods~\cite{cgba,cgba2}. However, to the best of our knowledge, no public implementations were available to us when these experiments were conducted, making direct reproduction infeasible. We therefore construct clean-label variants of UGBA and DPGBA to stress-test the same threat factor under reproducible settings.

To explore this scenario, we consider three variants of UGBA and DPGBA attacks: $UGBA_{CO}$, $DPGBA_{CO}$, and $UGBA_{CO}+DPGBA_{CO}$. In $UGBA_{CO}$, we restrict the trigger size to 1 and select victim nodes whose labels already match the target class. Likewise, in $DPGBA_{CO}$, the trigger size is 1, and victim nodes come solely from the target class. The combined version, $UGBA_{CO}+DPGBA_{CO}$, takes the same setting but simultaneously applies UGBA’s objective of increasing similarity among connected nodes and DPGBA’s goal of obfuscating the distribution of trigger nodes.

As shown in Table~\ref{tab:clean_vs}, our findings are threefold. First, the initial ASR for $DPGBA_{CO}$ and $UGBA_{CO}+DPGBA_{CO}$ drops sharply from $\geq90\%$ to 53.7\% and 21.6\%, respectively. Second, for these two attacks, \system{}'s detection recall for trigger and victim nodes is low (below 30\%), and the residual ASRs are 18.09\% and 13.91\%, respectively. Third, while $UGBA_{CO}$ maintains a high ASR ($>$90\%), \system{} detects most of its trigger nodes, reducing the post-defense ASR to 6.76\%. This is because even though UGBA’s triggers are similar to victims, their overall distribution remains distinct from benign nodes, making them detectable.

\begin{table}[htbp]
\caption{Varying Trigger Size of Clean Label Attack}
\centering
\resizebox{\linewidth}{!}{
\begin{tabular}{c|c|cc|cc|cc|cc} 
\hline
\multirow{3}{*}{Attack}     & \multirow{3}{*}{TS} & \multicolumn{2}{c|}{\multirow{2}{*}{No Defense}} & \multicolumn{2}{c|}{\multirow{2}{*}{Defense}} & \multicolumn{4}{c}{Node Detection}                               \\ 
\cline{7-10}
                            &                     & \multicolumn{2}{c|}{}                            & \multicolumn{2}{c|}{}                                     & \multicolumn{2}{c}{Victim Nodes}               & \multicolumn{2}{c}{Trigger Nodes}  \\ 
\cline{3-6}
                            &                     & ASR(\%)  & CA(\%)                                & ASR(\%)  & CA(\%)                                         & P(\%) & \multicolumn{1}{l}{R(\%)} & P(\%) & R(\%)         \\ 
\hline
\multirow{5}{*}{DPGBA}      & $1$                 & $53.74$  & $82.82$                               & $18.09$  & $83.11$                                        & $50.83$       & $28.00$                        & $80.00$       & $26.00$            \\
                            & $2$                 & $80.26$  & $83.04$                               & $0.09$   & $81.85$                                        & $92.55$       & $98.00$                        & $100.00$      & $98.00$            \\
                            & $3$                 & $94.26$  & $80.74$                               & $0.00$   & $80.74$                                        & $94.55$       & $100.00$                       & $100.00$      & $98.67$            \\
                            & $4$                 & $96.17$  & $79.93$                               & $0.00$   & $80.30$                                        & $95.45$       & $100.00$                       & $100.00$      & $98.50$            \\
                            & $5$                 & $97.65$  & $78.89$                               & $0.00$   & $79.63$                                        & $94.55$       & $100.00$                       & $100.00$      & $98.00$            \\ 
\hline
\multirow{5}{*}{UGBA}       & $1$                 & $92.18$  & $81.48$                               & $6.76$   & $82.89$                                        & $84.00$       & $84.00$                        & $100.00$      & $84.00$            \\
                            & $2$                 & $92.44$  & $82.30$                               & $4.36$   & $83.11$                                        & $93.33$       & $94.00$                        & $100.00$      & $94.00$            \\
                            & $3$                 & $94.76$  & $80.15$                               & $1.24$   & $83.26$                                        & $90.00$       & $96.00$                        & $100.00$      & $96.00$            \\
                            & $4$                 & $92.09$  & $80.82$                               & $1.51$   & $82.37$                                        & $83.03$       & $98.00$                        & $100.00$      & $98.00$            \\
                            & $5$                 & $93.69$  & $80.67$                               & $0.44$   & $80.89$                                        & $89.70$       & $96.00$                        & $100.00$      & $95.00$            \\ 
\hline
\multirow{5}{*}{\makecell{$DPGBA$\\$+$\\$UGBA_{CO}$}} & $1$                 & $21.57$  & $83.33$                               & $13.91$  & $83.56$                                        & $48.33$       & $20.00$                        & $80.00$       & $18.00$            \\
                            & $2$                 & $52.35$  & $83.30$                               & $2.17$   & $83.19$                                        & $89.24$       & $84.00$                        & $98.82$       & $82.00$            \\
                            & $3$                 & $67.74$  & $83.41$                               & $0.09$   & $82.96$                                        & $86.36$       & $100.00$                       & $100.00$      & $95.33$            \\
                            & $4$                 & $68.70$  & $82.37$                               & $0.44$   & $82.59$                                        & $82.67$       & $86.00$                        & $100.00$      & $84.50$            \\
                            & $5$                 & $70.52$  & $82.89$                               & $2.09$   & $82.44$                                        & $83.80$       & $68.00$                        & $100.00$      & $64.00$            \\
\hline
\end{tabular}
}
\label{tab:clean_ts}
\end{table}

However, attackers might try enlarging the victim set to increase ASR. To evaluate this, we vary the victim size among \{10, 20, 40, 80, 160\} for three variants. Table~\ref{tab:clean_vs} shows that increasing victim size does not notably improve ASR, while \system{}’s detection performance actually rises. Once the victim set exceeds 40, \system{} detects more than 60\% of both victim and trigger nodes with above 90\% precision. This occurs because adding more victim nodes inherently introduces more trigger nodes and more attack/trigger patterns; Thus, the influence of the trigger nodes is enhanced and the correlations between trigger and victim nodes are strengthened, so that trigger nodes are more easily detected.

We also vary trigger sizes on three variants. The results are shown in Table~\ref{tab:clean_ts}. From the table, we observe that as the trigger size increases, the attacker achieves a higher ASR when no defense is applied. However, our defense also shows improved performance, with a higher detection rate for trigger nodes, leading to a lower ASR. This improvement occurs because, as the trigger size increases, the internal correlation among the trigger nodes becomes stronger, enabling us to detect them more effectively based on this higher correlation. These results suggest that our defense is not easily bypassed by simply increasing the trigger size. We also provide the curve (see Figure.~\ref{fig:clean_vs_ts} of how ASR changes with varying victim size and trigger size.

\begin{figure}[h] 
\centering
\begin{subfigure}[b]{0.4\linewidth} 
\centering
\includegraphics[width=\textwidth]{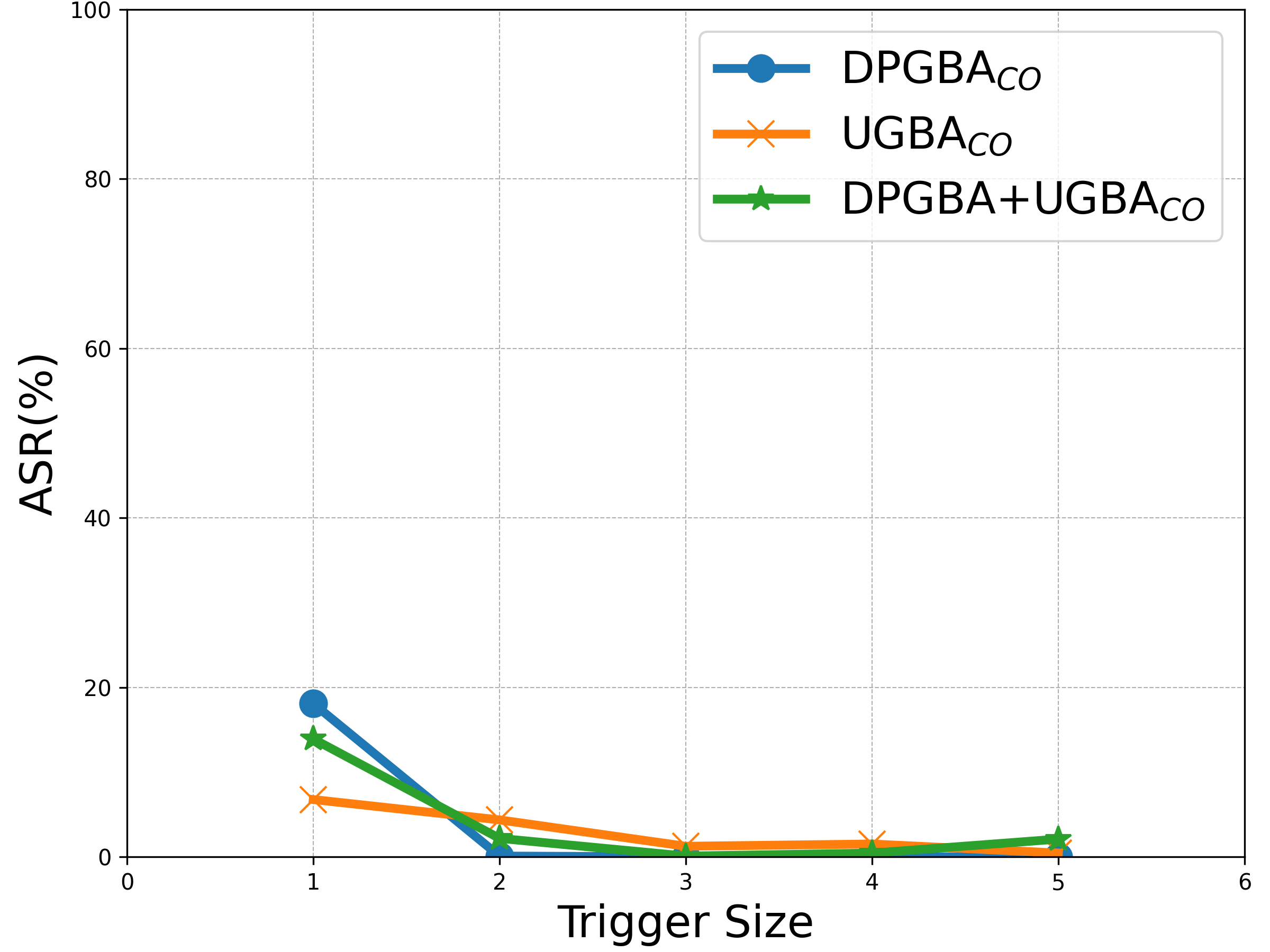}
\caption{Varying Trigger Size}
\label{fig:asr_ts}
\end{subfigure}
\hfill 
\begin{subfigure}[b]{0.4\linewidth}
\centering
\includegraphics[width=\textwidth]{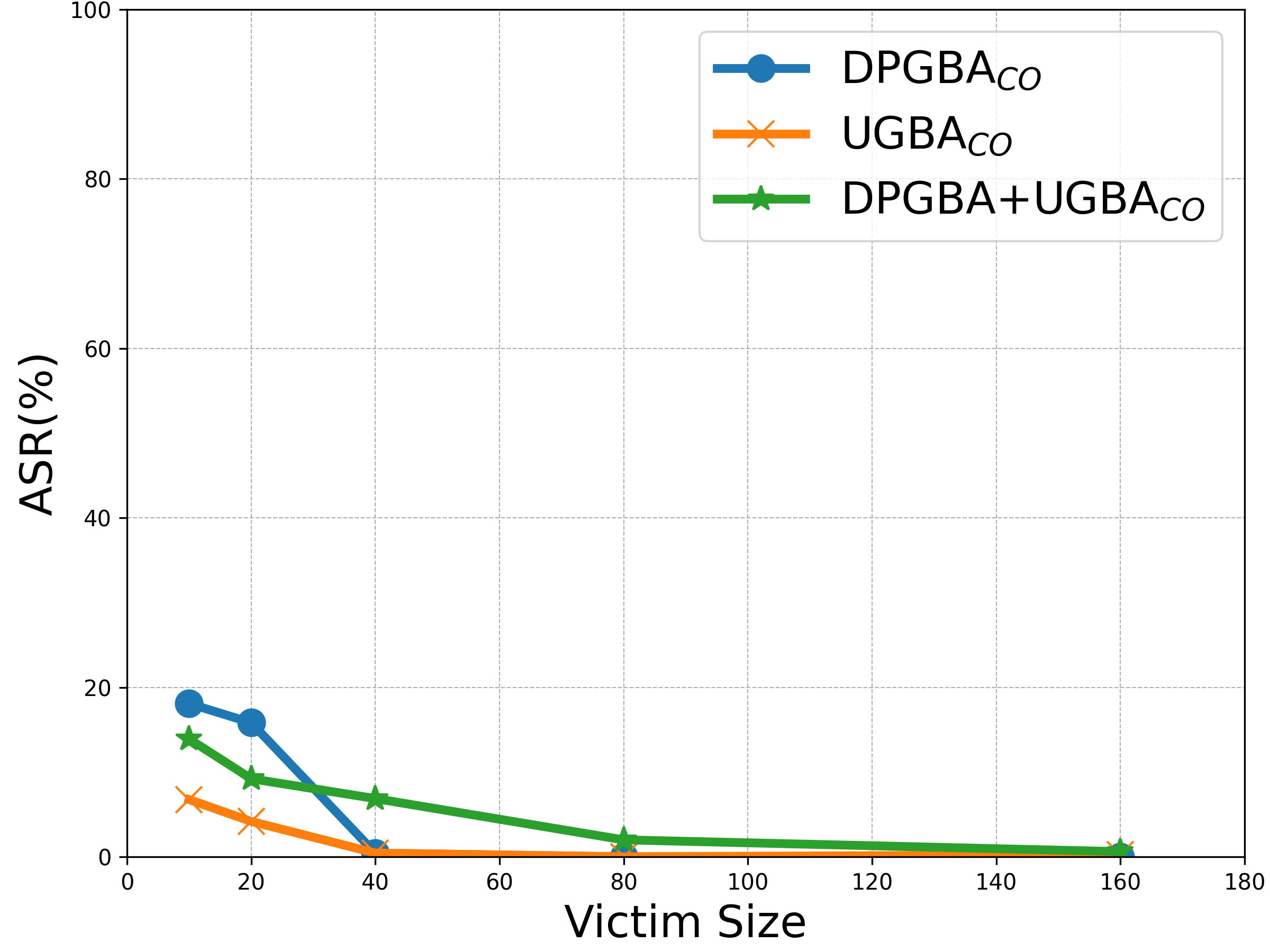} %
\caption{Varying Victim Size} 
\label{fig:asr_vs}
\end{subfigure}
\caption{Clean Label Attack}
\label{fig:clean_vs_ts}
\end{figure}

\noindent\textbf{Conclusion:} Although attackers may try one-node clean-label triggers to reduce the effectiveness of \system{}, the success of such approaches is limited by significant trade-offs. Even with variations in victim size or trigger size, the tested clean-label attacks achieve no more than 18.1\% ASR after defense. Future work could audit one-node triggers at inference, when they can change a non-target node's prediction, or use class-specific confidence thresholds calibrated on clean validation data to identify additional victim candidates.

\subsection{Random Noise Disruption Attack}
\label{sec:random_noise_add}

Beyond one-node clean-label attacks, adversaries may try to remove internal correlations among trigger nodes by injecting random noise. This disrupts internal correlation checks and allows triggers larger than one node. To further evade external influence checks, attackers can pair clean-label attack strategies with establishing edges between victim nodes and all trigger nodes. We implement this idea using an SBA attack, where a Gaussian noise with a standard deviation 30 times that of the original features is added to every trigger node.

We vary the trigger sizes \{1, 3, 5, 10, 20, 30, 40, 50, 60, 70, 80, 90, 100\} and victim sizes \{10, 20, 40, 80, 160\} on the Cora dataset. We plot ASR and CA curves before and after defense to illustrate how ASR and CA change with different victim and trigger sizes.

As shown in Figure~\ref{fig:random_noise_asr_ca}, increasing either the trigger or the victim size can raise the ASR. Some high-ASR configurations retain clean accuracy close to the undefended model.

What's more, we present additional results of our random noise disruption attack on the \textbf{PubMed dataset}. 

\begin{figure}[h]
\centering
\includegraphics[width=0.7\linewidth]{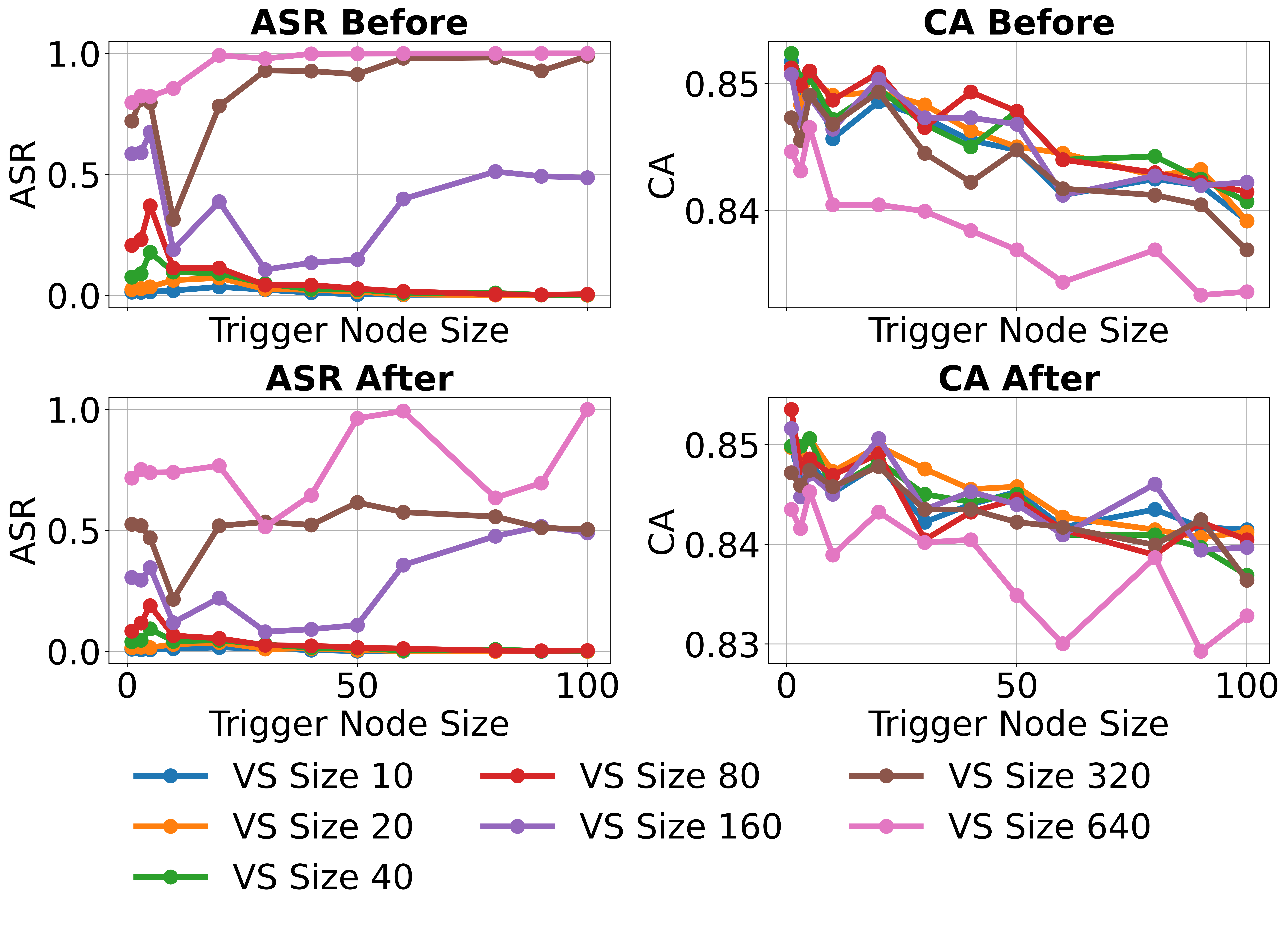}
\caption{ASR and CA of Random Disruption Attack}
\label{fig:pubmed_random_noise_asr_ca}
\end{figure}

As shown in Figure~\ref{fig:pubmed_random_noise_asr_ca}, the results align with those in Section~\ref{sec:adaptive}. On PubMed, the attacker can achieve high ASR with a large victim size (e.g., 640) and large noisy triggers. Such substantial pattern insertions motivate testing simple degree or local-volume checks.

 \begin{table}[t]
\caption{Detection results of Asymmetric Attacks}
\centering
\resizebox{\linewidth}{!}{
\begin{tabular}{c|c|cc|cc|cc|cc} 
\hline
\multirow{3}{*}{$p$} & \multirow{3}{*}{$q$} & \multicolumn{2}{c|}{\multirow{2}{*}{No Defense}} & \multicolumn{2}{c|}{\multirow{2}{*}{Defense Performance}} & \multicolumn{4}{c}{Victim and Trigger Node Detection}                  \\ 
\cline{7-10}
                   &                    & \multicolumn{2}{c|}{}                            & \multicolumn{2}{c|}{}                                     & \multicolumn{2}{c|}{Victim Nodes} & \multicolumn{2}{c}{Trigger Nodes}  \\ 
\cline{3-6}
                   &                    & ASR(\%) & (CA\%)                                 & ASR(\%) & (CA\%)                                          & P(\%) & R(\%)        & P(\%) & R(\%)         \\ 
\hline
0.5                & 10                 & 86.05   & 82.22                                  & 0       & 82.59                                           & 100           & 95                & 100           & 95.76              \\ 
\hline
0.5                & 3                  & 91.24  & 81.11                                  & 0       & 82.96                                           & 100             & 95                & 100           & 97.35              \\ 
\hline
0.8                & 10                 & 87.64   & 81.11                                  & 0       & 82.22                                           & 100           & 85                & 88.24         & 86.96              \\ 
\hline
0.8                & 3                  & 88.45   & 80.74                                  & 0       & 81.85                                           & 100           & 100               & 97.37         & 100                \\ 
\hline
0.9                & 10                 & 72.91   & 82.59                                  & 0       & 82.59                                           & 100           & 65                & 100           & 66.67              \\
\hline
\hline
\multicolumn{2}{c|}{Average} &85.26	&81.55&	0	&82.44	&100	&88	&97.12&	89.34\\
\hline
\end{tabular}
}
\label{tab:asymmetric}
\end{table}

\subsection{Asymmetric Trigger Insertion Strategy}
\label{sec:asymmetric}
Recent studies in the image backdoor domain~\cite{qi2022revisiting, xia2023waveattackasymmetricfrequencyobfuscationbased} have suggested that \textit{asymmetric triggers} can effectively reduce the latent separability of backdoor attacks, thereby evading defenses. Inspired by this insight, we propose an \textit {Asymmetric Graph Backdoor} that leverages asymmetry to make triggers in the training set less detectable.

Our strategy focuses on two primary aspects to reduce trigger visibility: trigger sizes and trigger features.  
1. \textit{Trigger sizes.} For each added trigger node, we assume a removal probability \(p\) during the training phase, which also removes the edges connected to the node. This reduces the effective strength of the trigger at the node level. When \(p=0\), the trigger node remains intact, representing a full-strength poison. Conversely, when \(p=1\), all trigger nodes are removed, resulting in no poisoning effect.  
2. \textit{Trigger Features.} To diminish feature strength, we introduce a random reduction factor \(\alpha\), such that the training feature value is given by \(h_\text{train} = \alpha \times h\). Here, \(\alpha\) is a random variable defined as \(\alpha = \text{Rand}(0, 1)^q\), where \(q\) is a strength factor controlling the magnitude of the reduction. When \(q=0\), \(\alpha = 1\), resulting in full-strength features. As \(q\) increases, \(\alpha\) approaches 0, leading to negligible feature strength.  For experiments, we implement this strategy with the SBA attack. We use a victim size 40 and a trigger size of 5; then we vary $p$ and $q$ to acquire different strengths of the attack.

By tuning \(p\) and \(q\), we can precisely control the strength of the training triggers in terms of both nodes and features. During the testing stage, trigger nodes and features are set to full strength to ensure attack success.

\noindent \textbf{Conclusion.}  The results are shown in Table~\ref{tab:asymmetric}. Based on Table~\ref{tab:asymmetric}, \system{} remains robust against the asymmetric trigger insertion strategy, achieving at least 97\% mean precision and 88\% mean recall for both victims and triggers, and a 0\% ASR while maintaining higher mean CA than the no-defense baseline. These findings underscore the effectiveness and resilience of our \system{} in defending against diverse backdoor attacks.

\subsection{Multi-target Backdoor Attacks}\label{apd:multi-target}

This setting stress-tests the downstream score-to-node localization stage rather than the core two-view scoring mechanism. In standard single-target attacks, poisoned training nodes usually share one target label, so the single-target refinement in Appendix~\ref{apd:score_to_node_details} can exploit label consistency along the sorted score list. Multi-target attacks~\cite{multi-target} instead inject simultaneous backdoors with different target labels, so high-deviation victim seeds may be spread across several observed label groups. This does not change the internal-correlation or external-influence scores: trigger nodes remain structurally and functionally abnormal regardless of how many target labels are used. It only changes how the high-deviation seeds are converted into victim groups. Following Appendix~\ref{apd:score_to_node_details}, \system{} first applies the label-free valley cutoff, then retains multiple supported label groups and performs trigger recovery and unlearning group-wise, without enforcing the single-target stopping rule.

This same reasoning also covers ALL2ALL-style attacks~\cite{lu2025anywheredoormultitargetbackdoorattacks}
. Such attacks instantiate multiple simultaneous trigger-target mappings, but their localization challenge is still to recover several suspicious target groups from the same score distribution. Since \system{} groups high-deviation seeds after the label-free cutoff, it does not require knowing the number of target labels or the source-to-target mapping.

\noindent\textbf{Results.}
Table~\ref{tab:multi_target} reports the defense results under multi-target attacks. DPGBA is reduced from 90.4\% ASR to 0.0\%, and UGBA from 96.0\% ASR to 2.2\%, while CA is maintained after defense. Localization precision and recall both exceed 92\% for victims and triggers under both attacks. As shown in Figure.~\ref{fig:multi_target}, \system{} assigns high deviation scores to trigger nodes under both internal correlation and external influence across target labels, confirming that the core two-view detector is agnostic to the number of target labels; only the lightweight score-to-node conversion needs to be read group-wise.

\begin{table}[htbp]
\centering
\caption{Defense and localization results under multi-target attacks.}
\label{tab:multi_target}
\small
\setlength{\tabcolsep}{4pt}
\begin{tabular}{lcccccc}
\toprule
Attack & ASR Before & ASR After & CA Before & CA After & Victim Nodes. P/R & Trigger Nodes. P/R \\
\midrule
DPGBA & 90.4\% & 0.0\% & 81.8\% & 82.9\% & 92.1\%/100.0\% & 92.5\%/94.5\% \\
UGBA  & 96.0\% & 2.2\% & 81.8\% & 82.9\% & 94.2\%/96.0\% & 93.3\%/92.5\% \\
\bottomrule
\end{tabular}
\vspace{-5pt}
\end{table}

\begin{figure}[htbp] 
\centering
\begin{subfigure}[b]{0.48\linewidth} 
\centering
\includegraphics[width=\textwidth]{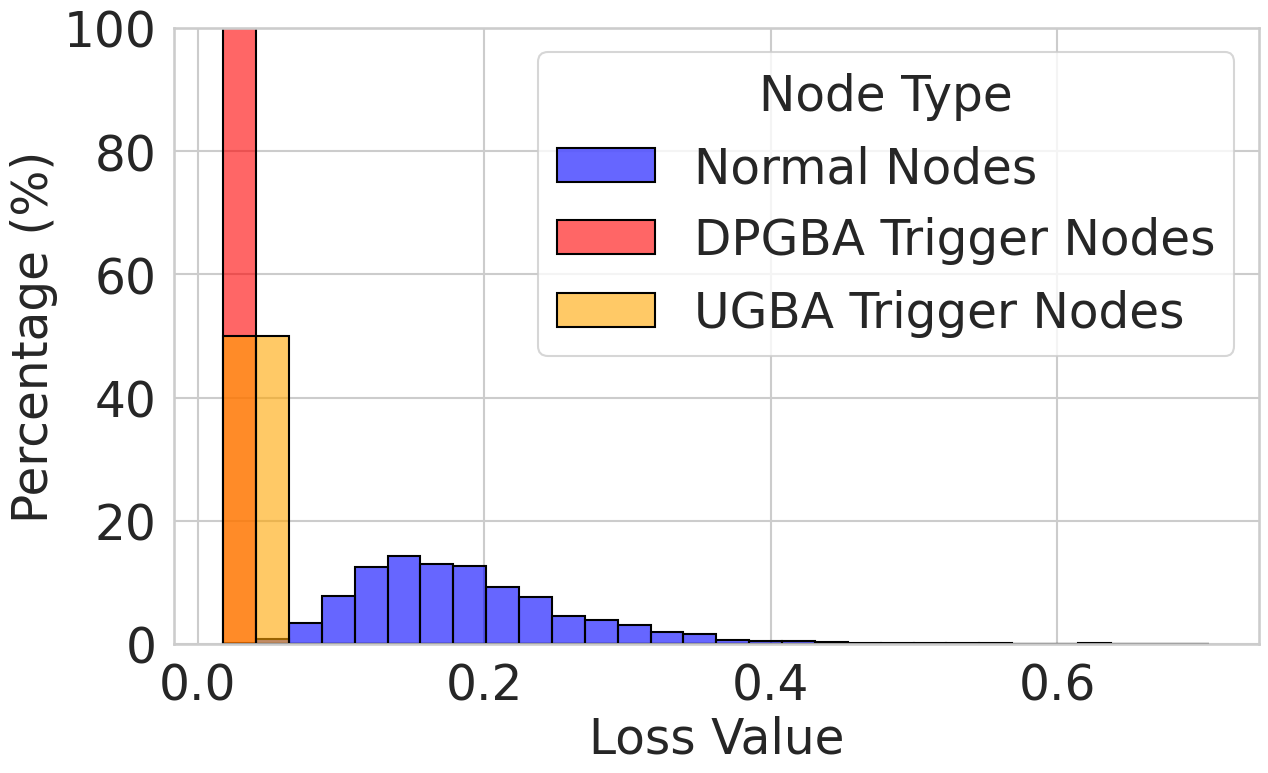}
\caption{Internal Correlation}
\label{fig:internal_target}
\end{subfigure}
\hfill 
\begin{subfigure}[b]{0.48\linewidth}
\centering
\includegraphics[width=\textwidth]{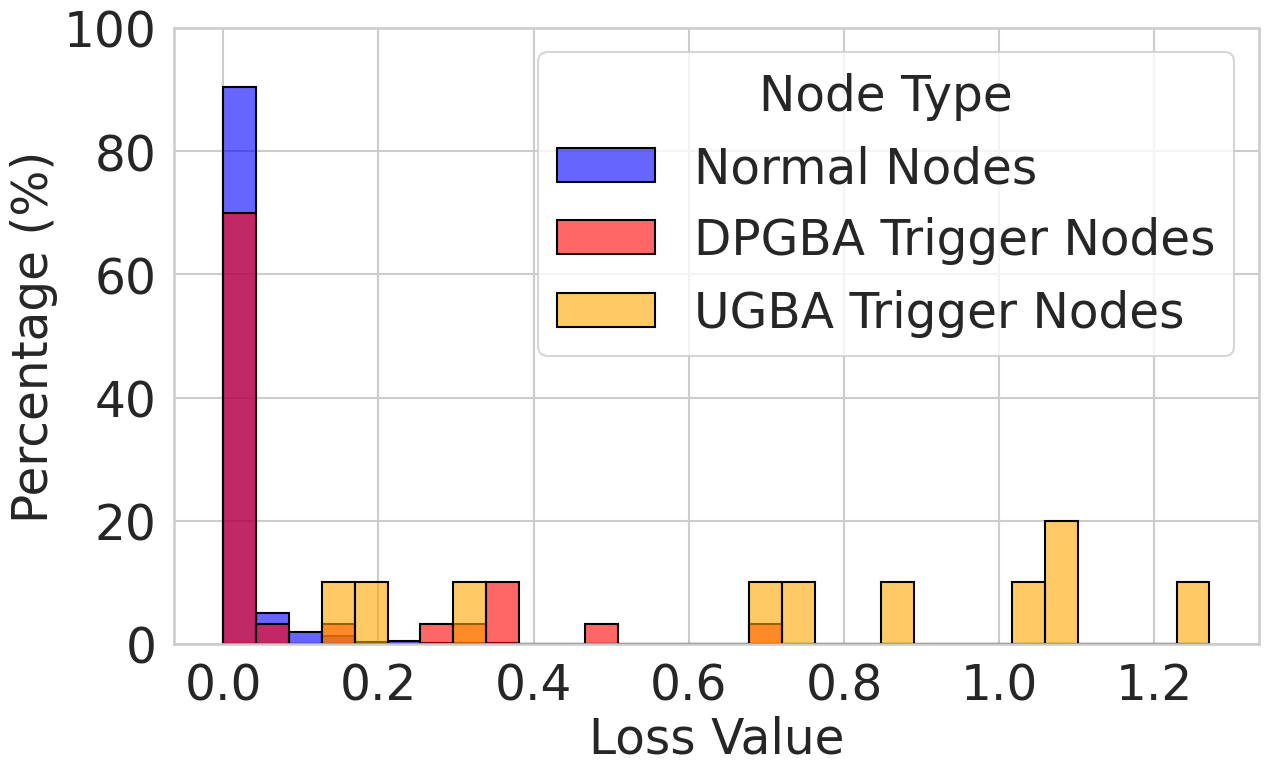} %

\caption{External Influence} 
\label{fig:external_target}
\end{subfigure}
\caption{Two deviation scores under multi-target attacks.}
\label{fig:multi_target}
\end{figure}

\subsection{Trade-offs for the Attacker}
\label{sec:trade-offs-adaptive-attack}

We further examine the trade-offs faced by adaptive attackers under defense-aware settings. The evaluated attacks directly target the main components of \system{}, including internal-correlation scoring, external-influence scoring, and score-to-node localization. While these experiments do not exhaust all possible adaptive strategies, they cover natural ways to weaken the signals used by our defense.

Our results indicate that evading \system{} can involve a trade-off between attack effectiveness and stealthiness. Post-defense ASR reaches at most 18.1\% for the clean-label variants we test. Under random disruption, increasing the attack scale, such as using larger trigger or victim sets, can yield substantially higher ASR even when clean accuracy remains close to that of the undefended model.

\section{Discussion}
\label{apd:discussion}

\noindent\textbf{Runtime cost and complexity optimization.}
Theoretically, our approach utilizes the Deep Graph Library (DGL)~\cite{wang2019dgl} for GNN training, which performs message passing on an edge-by-edge basis. For a GCN layer with hidden dimension \(H\), \(N\) nodes, and \(E\) edges, the per-layer time complexity is \(\mathcal{O}(N H^2 + E H)\). During our initial defense implementation, node masking necessitates inference for each candidate node, resulting in an overall complexity of \(\mathcal{O}(N \times (N H^2 + E H))\). To enhance efficiency, we optimized the computation by introducing \system{-Faster}. Firstly, this method updates \emph{only} the masked nodes and their neighbors, circumventing the recalculation of the entire graph. Consequently, the complexity for processing a single node at the \(L\)-th layer is reduced from \(\mathcal{O}(N H^2 + E H)\) to \(\mathcal{O}(N_s H^2 + E_s H)\), where \(N_s\) and \(E_s\) denote the number of nodes and edges, respectively, within the \(L\)-hop subgraph centered on the masked node. This optimization significantly reduces the total computational cost (e.g., reduces detection time to 5\% of its original on OGB-arxiv). Detailed runtime costs are provided in Table~\ref{tab:runtime}.

Secondly, by focusing on suspicious nodes (i.e., training nodes and neighbors within the model's depth layer), \system-Faster can further limit the analysis scope, enhancing efficiency (e.g.,~25\% fewer nodes processed on OGB-arxiv). For scenarios where timely updates are needed, trigger detection can be performed only on updated training-related nodes, avoiding full graph reprocessing.

Nevertheless, additional methods, such as multi-process inference and sampling methods, could further accelerate our system. We leave these for future work.

\begin{table}
\caption{Runtime Cost}
\centering
\resizebox{0.75\linewidth}{!}{
\begin{tabular}{c|c|c|c} 
\hline
Dataset   & Score Calculation & Node Selection & Relearning  \\ 
\hline
Cora  & 43.21s & 0.04s   & 1.55s\\ 
\hline
PubMed& 700.28s& 0.80s& 2.00s   \\ 
\hline
OGB-arxiv &  1560.11s & 40.02s   &  15.83s \\
\hline
\end{tabular}
}
\label{tab:runtime}
\end{table}

\begin{table}
\caption{Peak GPU memory under the standard configuration.}
\centering
\begin{tabular}{c|c|c} 
\hline
Dataset & \#Nodes & Peak GPU Memory \\ 
\hline
Cora & 2,708 & 0.83 GB \\ 
\hline
PubMed & 19,717 & 4.47 GB \\ 
\hline
OGB-arxiv & 169,343 & 10.7 GB \\
\hline
\end{tabular}
\label{tab:gpu_memory}
\end{table}

\noindent\textbf{Memory overhead.}

We further report peak GPU memory under our standard configuration in Table~\ref{tab:gpu_memory}. A \(62.5\times\) increase in node count from Cora to OGB-arxiv yields only a \(12.9\times\) increase in peak memory, remaining within the capacity of a single GPU.

\noindent\textbf{Application scope.}
Although \system{} is designed for defending node-level graph backdoor attacks, its internal correlation—which captures relationships among trigger nodes within a subgraph—is not limited to node-level tasks. Even for graph-level backdoor attacks (e.g., SBA and GTA can also be launched to graph-level attacks) that replace a subgraph of the original clean graph with a trigger subgraph, \system{} can still quantify the internal correlations among the injected trigger nodes. Meanwhile, the external influence could be adapted to measure the impact of each subgraph rather than individual nodes. Consequently, \system{} has the potential to defend against graph-level backdoor attacks, and we plan to explore this in future work.

\section{Impact Statements}
\label{apd:impactstatement}

This paper presents \system, a new defense paradigm against backdoor attacks on Graph Neural Networks (GNNs). By jointly auditing two complementary signals—\emph{internal correlation} within potential trigger subgraphs and \emph{external influence} of suspicious nodes on neighborhood predictions—our method aims to localize and neutralize backdoor behavior via filtered training and targeted unlearning. If adopted in practice, \system{} could improve the safe deployment of GNNs in security- and safety-relevant applications (e.g., fraud detection, recommendation, and provenance/security graphs), where stealthy poisoning can cause systematic misclassification while remaining hard to diagnose.

However, \system{} also has limitations and potential negative impacts. As a defense study, our analysis and adaptive-attack discussion may inadvertently inform adversaries about which trigger designs are harder to audit, accelerating evasion research. In addition, because our auditing relies on correlation/influence deviations, naturally high-influence nodes (e.g., hubs, minority-pattern nodes, or distribution-shifted regions) could be flagged, leading to false positives and possible unfair performance degradation if unlearning is applied to benign nodes. Furthermore, as demonstrated in our analysis of adaptive clean-label attacks (Section~\ref{sec:adaptive}), extremely stealthy settings such as a \emph{one-node clean-label} trigger can weaken the auditing signals and partially bypass \system, raising the Attack Success Rate (e.g., up to 18.1\%). While our work focuses on defense, we encourage future research to investigate how to evaluate and improve robustness under stronger adaptive clean-label strategies, to ensure equitable and reliable security across diverse threat landscapes.

\section{Theoretical Framework}\label{sec:theoretical_framework}
Let $\mathcal{G}=(\mathcal{V},\mathcal{E})$ be a graph with features $\mathcal{X}$, and $f$ be a GNN model with $r$ layers that computes a logits $H_v(\mathcal{G})\in\mathbb{R}^C$ for node $v$ in graph $\mathcal{G}=(\mathcal{G}, \mathcal{X})$, where $C$ is the number of classes. A backdoor attack inserts a trigger $\mathcal{T}$—a set of structural or feature perturbations—into a clean graph $\mathcal{G}_C$ to create $\mathcal{G}_T$, aiming to force $v$ into a target class. Our goal is to decompose this influence and motivate complementary detection scores.

\subsection{Decomposition of Anomalous Influence}
\label{sec:decomposition}

We define the influence of any subgraph components $S$ on $v$'s \textit{predicted logits}, where $S$ within $v$'s $r$-hop neighborhood. We use an interventional operation $\mathcal{G}_C\oplus\mathsf{do}(S)$, which represents intervening on $v$'s $r$-hop neighborhood by applying the perturbation $S$ (adds $S$ if absent; removes $S$ if present). The resulting change is:
\[
\Delta_H(S)\;=\;H_v\!\big(\mathcal{G}_C\oplus\mathsf{do}(S)\big)\;-\;H_v(\mathcal{G}_C)
\]
By this, $\Delta_H(\emptyset)=\mathbf{0}$, and the total trigger influence is $\Delta_H(\mathcal{T})$.

Our theoretical goal is to prove that the trigger's influence cannot be simultaneously negligible in both its individual components ($EI$) and its synergistic structure ($IC$).

We build on the established assumption \cite{RIGBD} that a high-ASR attack tends to induce a "large prediction variance," meaning its influence is non-trivial. We formalize this minimal assumption as our premise.

\begin{assumption}[Non-Trivial Influence]
\label{premise:sasr}
A high-ASR trigger $\mathcal{T}$ typically induces a perturbation $\Delta_H(\mathcal{T})$ whose magnitude is non-trivial, i.e., it is bounded below by a detectable, non-trivial constant $\Theta_{\text{detectable}} \gg 0$:
\begin{equation}
\big\|\Delta_H(\mathcal{T})\big\|_2\;\geq\;\Theta_{\text{detectable}}
\label{eq:premise-sasr}
\end{equation}
\end{assumption}

Then we attribute the total influence $\Delta_H(\mathcal{T})$ using Harsanyi dividends \cite{harsanyi1982simplified}, an application of Möbius inversion to set functions. We model $\Delta_H(\mathcal{T})$ as a set function $v:2^{\mathcal{T}}\!\to\mathbb{R}^C$. The total influence uniquely decomposes into dividends $\Phi_S = \Delta_H(S) - \sum_{R \subset S} \Phi_R$, such that $\Delta_H(\mathcal{T}) = \sum_{\emptyset\neq S\subseteq\mathcal{T}}\!\Phi_S$.

We attribute this sum to two disjoint sources:
\textbf{1. External Influence (EI):} The sum of first-order dividends. This captures the \emph{additive} effect of individual trigger components.
\textbf{2. Internal Correlation (IC):} The sum of all higher-order ($|S|\ge 2$) dividends. This captures the \emph{non-additive, synergistic} effects emerging from the components' structure.
\[
EI(\mathcal{T}) = \sum_{t\in\mathcal{T}} \Phi_{\{t\}}, \quad
IC(\mathcal{T}) = \sum_{\substack{S\subseteq\mathcal{T} \\ |S|\ge 2}} \Phi_S
\]
By definition, these two components sum to the total perturbation:
\begin{equation}
\Delta_H(\mathcal{T})\;=\;EI(\mathcal{T})\;+\;IC(\mathcal{T})
\label{eq:sum}
\end{equation}

This decomposition leads to a key theoretical guarantee.
\begin{theorem}[Influence Dichotomy]
\label{thm:dichotomy}
If a trigger $\mathcal{T}$ satisfies the Non-Trivial Influence premise (Eq.~\ref{eq:premise-sasr}), then its influence \emph{must} be present in at least one of its components at a non-trivial level:
\[
\big\|EI(\mathcal{T})\big\|_2\;\geq\;\tfrac{1}{2}\Theta_{\text{detectable}}
\quad\text{or}\quad
\big\|IC(\mathcal{T})\big\|_2\;\geq\;\tfrac{1}{2}\Theta_{\text{detectable}}
\]
\end{theorem}

\begin{proof}
Assume for contradiction that both $\|EI(\mathcal{T})\|_2 < \tfrac{1}{2}\Theta_{\text{detectable}}$ and $\|IC(\mathcal{T})\|_2 < \tfrac{1}{2}\Theta_{\text{detectable}}$.
By the triangle inequality on the decomposition (Eq.~\ref{eq:sum}):
\begin{align*}
\big\|\Delta_H(\mathcal{T})\big\|_2 &= \big\|EI(\mathcal{T}) + IC(\mathcal{T})\big\|_2 \\
&\le \big\|EI(\mathcal{T})\big\|_2 + \big\|IC(\mathcal{T})\big\|_2
\end{align*}
Substituting the assumptions yields:
\[
\big\|\Delta_H(\mathcal{T})\big\|_2 \;<\; \tfrac{1}{2}\Theta_{\text{detectable}} + \tfrac{1}{2}\Theta_{\text{detectable}} \;=\; \Theta_{\text{detectable}}
\]
This result, $\big\|\Delta_H(\mathcal{T})\big\|_2 < \Theta_{\text{detectable}}$, directly contradicts Premise~\ref{premise:sasr}. Thus, the assumption is false.
\end{proof}

\noindent\textbf{Implication:} Theorem~\ref{thm:dichotomy} provides the theoretical justification for our defense. It proves that the non-trivial influence of a trigger \emph{cannot} simultaneously vanish from both its individual ($EI$) and synergistic ($IC$) components. This motivates measuring both channels. Appendix~\ref{apd:structured_additions} gives conditions under which the corresponding practical scores reveal this influence.

\subsection{From Theoretical Components to Practical Proxy Scores}
\label{sec:soundness-of-proxies}

Our theoretical framework, culminating in Theorem~\ref{thm:dichotomy}, shows that a trigger $\mathcal{T}$ satisfying Premise~\ref{premise:sasr} must induce non-negligible influence through at least one of two channels: the first-order external-influence component $EI(\mathcal{T})$ or the higher-order internal-correlation component $IC(\mathcal{T})$.
For a known victim and an $m$-component trigger, the aggregate components require only the empty, singleton, and full coalitions (at most $m+2$ evaluations), whereas all dividends require $2^m$. At deployment the trigger is unknown, so searching candidate component sets is combinatorial.

\system{} therefore does not attempt to recover the exact theoretical quantities.
Instead, it employs two efficient node-level proxy scores, $\mathcal{S}_{ext}$ and $\mathcal{S}_{int}$, to empirically audit the two corresponding influence channels.
Specifically, $\mathcal{S}_{ext}$ probes whether a node has unusually large individual influence on its local neighborhood, while $\mathcal{S}_{int}$ probes whether a node is embedded in a strongly predictable local pattern.
Below, we explain how these two practical scores are aligned with the theoretical components.

\subsubsection{\texorpdfstring{$\mathcal{S}_{ext}$}{Sext} as a Proxy for External Influence}

The theoretical $EI$ component captures the first-order, individual contribution of trigger nodes to the victim's prediction.
If $EI$ is nonzero, at least one trigger node has nonzero individual influence. Appendix~\ref{apd:structured_additions} gives conditions under which masking such a node changes local predictions.

\noindent \textbf{Definition.}
$\mathcal{S}_{ext}$ is computed by masking a single node $v$ and measuring the resulting shift in the softmax probability distributions of its 1-hop neighbors using KL and JS divergences.

\noindent \textbf{Proxy intuition.}
This mask-and-measure intervention provides an efficient empirical estimate of how much node $v$ individually contributes to nearby predictions.
Although this score is not identical to the Harsanyi first-order dividend, it targets the same failure mode: if an attack concentrates backdoor influence on a small number of trigger nodes, masking such a node should induce an unusually large prediction shift in its neighborhood.
Thus, $\mathcal{S}_{ext}$ operationalizes the external-influence channel in a scalable node-level form.

\noindent \textbf{Practical caveat.}
Masking node $v$ may also perturb higher-order interactions involving $v$.
Appendix~\ref{apd:reported_coalition_audit} compares the deployed score with the exact singleton influence.

\vspace{-5pt}
\subsubsection{\texorpdfstring{$\mathcal{S}_{int}$}{Sint} as a Proxy for Internal Correlation}

The theoretical $IC$ component captures higher-order, synergistic influence arising from the joint pattern of trigger components.
An $IC$-dominant attack spreads the backdoor effect across multiple trigger nodes, so the attack relies less on any single highly influential node and more on the collective structure or feature pattern formed by the trigger.

\noindent \textbf{Definition.}
$\mathcal{S}_{int}$ is computed as the inverse of the masked-node reconstruction loss from a masked graph autoencoder:
nodes that are easier to reconstruct from their surrounding context receive higher $\mathcal{S}_{int}$ scores.

\noindent \textbf{Proxy intuition.}
A masked graph autoencoder is trained to capture statistical dependencies and neighborhood predictability.
If a trigger subgraph forms a tightly coupled and distinctive pattern, then masking one trigger node still leaves strong contextual evidence from the remaining trigger-related nodes.
As a result, the masked trigger node tends to have a low reconstruction loss.
By taking the inverse of this loss in Eq.~\ref{euq:sint}, \system{} assigns a high $\mathcal{S}_{int}$ score to nodes embedded in such strongly correlated local patterns.

\noindent \textbf{Connection to theory.}
$\mathcal{S}_{int}$ does not explicitly compute the higher-order Harsanyi dividends in $IC(\mathcal{T})$.
Rather, it serves as a scalable proxy for local dependency. Appendix~\ref{apd:structured_additions} gives sufficient conditions for joint influence to appear in this score.
Together with $\mathcal{S}_{ext}$, this gives \system{} two complementary probes: one for concentrated individual influence and one for distributed pattern-level influence.

\subsection{Conditional Score Guarantees in Structured Settings}
\label{apd:structured_additions}

This section supplements the logit-vector framework in Appendix~\ref{sec:theoretical_framework} with explicit sufficient conditions for score observability. We distinguish four objects: the original logit-vector $(EI,IC)$, probability-vector $(EI_v^p,IC_v^p)$, target-coordinate audit quantities $(\Phi_A^F,e_F,IC_F)$, and the deployed scores $(\mathcal S_{\mathrm{ext}},\mathcal S_{\mathrm{int}})$. Appendix~\ref{apd:reported_coalition_audit} tests the deployed scores beyond these structured settings.

\subsubsection{Probability-space decomposition and singleton carriers}
Fix a poisoned classifier $f$, victim $v$, trigger set $\mathcal T$, $m=|\mathcal T|$, and $C$ classes. Use a fixed clean baseline $\mathcal G_C$ and the intervention of Appendix~\ref{sec:decomposition}:
\begin{align}
\mathcal G_A&:=\mathcal G_C\oplus\operatorname{do}(A), &
\mathcal G_\varnothing&:=\mathcal G_C,\\
p_v^f(\mathcal G)&:=\operatorname{softmax}(H_v^f(\mathcal G)), &
\Delta P_v(A)&:=p_v^f(\mathcal G_A)-p_v^f(\mathcal G_\varnothing).
\end{align}
Every intervention applies exactly the components in $A$ to the same clean context, with non-trigger content fixed. Define the coordinate-wise Harsanyi dividends by
\begin{equation}
\Phi_v^p(A)=\Delta P_v(A)-\sum_{B\subsetneq A}\Phi_v^p(B),
\qquad \Phi_v^p(\varnothing)=0.
\end{equation}
Then
\begin{align}
EI_v^p(\mathcal T)&=\sum_{t\in\mathcal T}\Phi_v^p(\{t\}), &
IC_v^p(\mathcal T)&=\sum_{\substack{A\subseteq\mathcal T\\|A|\ge2}}\Phi_v^p(A),\\
\Delta P_v(\mathcal T)&=EI_v^p(\mathcal T)+IC_v^p(\mathcal T).
\end{align}
The triangle inequality yields
\begin{equation}
\max\{\|EI_v^p(\mathcal T)\|_2,\|IC_v^p(\mathcal T)\|_2\}
\ge \tfrac12\|\Delta P_v(\mathcal T)\|_2.
\end{equation}

\begin{lemma}[Singleton carrier]
\label{lem:reported_carrier}
If $\|EI_v^p(\mathcal T)\|_2\ge E>0$, there exists $t^\star\in\mathcal T$ such that $\|\Phi_v^p(\{t^\star\})\|_2\ge E/m$.
\end{lemma}
\begin{proof}
By the triangle inequality,
\[
E\le\left\|\sum_{t\in\mathcal T}\Phi_v^p(\{t\})\right\|_2
\le\sum_{t\in\mathcal T}\|\Phi_v^p(\{t\})\|_2
\le m\max_{t\in\mathcal T}\|\Phi_v^p(\{t\})\|_2.
\]
\end{proof}

\subsubsection{External score observability}
For a reference version of the external probe, let $f'$ be the scoring predictor, and let $p_j^{f'}$ and $q_{j,t}^{f'}$ denote its probability distributions before and after masking $t$. Define
\begin{align}
\ell_{j,t}^{f'}&=\frac{1}{C}\left[\operatorname{KL}(p_j^{f'}\Vert q_{j,t}^{f'})+R(p_j^{f'},q_{j,t}^{f'})\right],\\
\mathcal S_{\mathrm{ext}}^{(0)}(t)&=\frac{1}{|\mathcal N(t)|}\sum_{j\in\mathcal N(t)}\ell_{j,t}^{f'},
\end{align}
where $R\ge0$ (the manuscript uses JS). KL uses natural logarithms. The factor $1/C$ conservatively covers class averaging; summed KL gives a stronger floor.

\begin{proposition}[Additive observable carrier]
\label{prop:reported_external}
Assume $\|EI_v^p(\mathcal T)\|_2\ge E>0$ and let $t^\star$ be a carrier from Lemma~\ref{lem:reported_carrier}. Suppose (i) every higher-order probability dividend vanishes; (ii) $f'=f$; (iii) masking $t^\star$ realizes coalition deactivation; and (iv) $v\in\mathcal N(t^\star)$ with $|\mathcal N(t^\star)|\le D$. Then
\begin{equation}
\mathcal S_{\mathrm{ext}}^{(0)}(t^\star)\ge\frac{E^2}{CDm^2}.
\label{eq:reported_external_bound}
\end{equation}
\end{proposition}
\begin{proof}
Additivity makes the full-context removal effect equal to $\Phi_v^p(\{t^\star\})$. Predictor and intervention matching give $\delta=p_v^{f'}-q_{v,t^\star}^{f'}$ with $\|\delta\|_2\ge E/m$. Let $a=\sum_{\delta_c>0}\delta_c$. Since $\sum_c\delta_c=0$, we have $\|\delta\|_1=2a$ and $\|\delta\|_2^2\le2a^2$, so $\|\delta\|_1^2\ge2\|\delta\|_2^2$. Pinsker's inequality and $R\ge0$ imply
\[
\ell_{v,t^\star}^{f'}\ge\frac{1}{2C}\|\delta\|_1^2
\ge\frac{1}{C}\|\delta\|_2^2\ge\frac{E^2}{Cm^2}.
\]
Averaging nonnegative contributions over at most $D$ scored neighbors proves the result.
\end{proof}

The proposition bounds the carrier's score. Separating it from benign nodes would also require a bound on their scores.

\begin{corollary}[One-node, single-attachment case]
Under the predictor and masking conditions above, if $\mathcal T=\{t\}$ and $v$ is the only scored neighbor, then $m=D=1$, $IC_v^p(\mathcal T)=0$, and
\[
\mathcal S_{\mathrm{ext}}^{(0)}(t)\ge\frac{1}{C}\|EI_v^p(\mathcal T)\|_2^2\ge\frac{E^2}{C}.
\]
\end{corollary}

\noindent\textbf{Relaxing exact matching.}
Let $r_t^f=p_v^f(\mathcal G_{\mathcal T})-p_v^f(\mathcal G_{\mathcal T\setminus\{t\}})$. Suppose each instance has $\|EI_v^p(\mathcal T)\|_2\ge E$, an observable carrier $t^\star$ with $|\mathcal N(t^\star)|\le D$, and uniform $\kappa_0,\alpha_0>0$ satisfying
\begin{align}
\|r_{t^\star}^f\|_2&\ge\kappa_0\|\Phi_v^p(\{t^\star\})\|_2,\\
\|p_v^{f'}-q_{v,t^\star}^{f'}\|_2&\ge\alpha_0\|r_{t^\star}^f\|_2.
\end{align}
The same proof gives $\mathcal S_{\mathrm{ext}}^{(0)}(t^\star)\ge\alpha_0^2\kappa_0^2E^2/(CDm^2)$. Here $\kappa_0$ quantifies retention under full-context removal and $\alpha_0$ quantifies scoring/intervention fidelity.

\subsubsection{Internal score observability}
Define an equal-share allocation of higher-order dividends,
\begin{equation}
\psi_{v,t}^{p,IC}=\sum_{\substack{A\subseteq\mathcal T:\\t\in A,\ |A|\ge2}}\frac{\Phi_v^p(A)}{|A|}.
\end{equation}
Summing over $t$ gives $\sum_t\psi_{v,t}^{p,IC}=IC_v^p(\mathcal T)$. Therefore $\|IC_v^p(\mathcal T)\|_2\ge I>0$ implies a carrier $t^\dagger$ with $\|\psi_{v,t^\dagger}^{p,IC}\|_2\ge I/m$. This auxiliary allocation is not a unique node-level definition of IC.

Let $g=f_D\circ f_E$ be a GraphMAE map, $Z_t$ the visible same-graph context after masking $t$, and $\widehat g$ the fitted map. Define
\begin{equation}
L_t(g)=\big(1-\cos(x_t,g(Z_t))\big)^\gamma,\quad \gamma>1,
\qquad \mathcal S_{\mathrm{int}}^{(0)}(t)=\frac{1}{L_t(\widehat g)}.
\end{equation}

\begin{proposition}[Pure-joint, context-identifiable carrier]
\label{prop:reported_internal}
Let $m\ge2$. Assume (i) $\Phi_v^p(A)=0$ for every nonempty proper subset $A\subsetneq\mathcal T$, while $\|\Phi_v^p(\mathcal T)\|_2\ge I>0$; (ii) the masked feature and ID do not leak; (iii) the remaining trigger context lies within the receptive field; (iv) a map $g^\star$ in the GraphMAE class satisfies $L_{t^\dagger}(g^\star)\le r_T$; and (v) the same carrier satisfies $L_{t^\dagger}(\widehat g)-L_{t^\dagger}(g^\star)\le\varepsilon_{\mathrm{fit}}$, with $r_T,\varepsilon_{\mathrm{fit}}\ge0$. Then every node is an IC carrier,
\[
\psi_{v,t}^{p,IC}=\Phi_v^p(\mathcal T)/m,\qquad
\|\psi_{v,t}^{p,IC}\|_2\ge I/m,
\]
and the context-identifiable carrier obeys
\begin{equation}
L_{t^\dagger}(\widehat g)\le r_T+\varepsilon_{\mathrm{fit}},\qquad
\mathcal S_{\mathrm{int}}^{(0)}(t^\dagger)\ge\frac{1}{r_T+\varepsilon_{\mathrm{fit}}}.
\label{eq:reported_internal_bound}
\end{equation}
The reference score is $+\infty$ for a zero denominator.
\end{proposition}
\begin{proof}
Only the full-coalition dividend remains, giving $\psi_{v,t}^{p,IC}=\Phi_v^p(\mathcal T)/m$. The pointwise approximation and fit conditions give the loss bound; reciprocation gives the score bound.
\end{proof}
IC supplies a carrier, while context identifiability and pointwise fitting make that same carrier reconstructible.

\noindent\textbf{Concrete two-layer witness.}
Let $\mathcal T$ be an $m$-node clique ($m\ge2$) with a shared nonzero feature $z$. Only attachment node $t_0$ is adjacent to victim $v$. For a coalition $A$, let $b_A=\mathbf 1\{t_0\in A\}$ and define
\begin{align}
g(A)&=\max\{\beta b_A|A|-\theta,0\},\\
H_v^f(\mathcal G_A)&=H_v^f(\mathcal G_\varnothing)+g(A)e_{y_{\mathrm{tar}}},
\end{align}
where $\beta>0$, $(m-1)\beta\le\theta<m\beta$, and $e_{y_{\mathrm{tar}}}$ selects the target coordinate. If $t_0\notin A$, then $b_A=0$; if $t_0\in A\subsetneq\mathcal T$, then $|A|\le m-1$. In either case $g(A)=0$. Only the full coalition has $g(\mathcal T)=m\beta-\theta>0$. Thus $\Delta P_v(A)=0$ for every proper coalition but $\Delta P_v(\mathcal T)\ne0$, yielding $EI_v^p(\mathcal T)=0$ and $IC_v^p(\mathcal T)=\Delta P_v(\mathcal T)$ by M\"obius inversion.

A two-layer sum-aggregation ReLU GNN realizes the gate by counting active trigger nodes at $t_0$ and sending the gated increment to $v$. For a non-attachment carrier $t^\dagger\ne t_0$, the peer-mean map reconstructs $z$ exactly. Hence $r_T=0$ if this map belongs to the GraphMAE class, and the proposition applies under the stated pointwise fitting condition. The threshold creates joint influence; identical peers create context identifiability. The map trains on the same graph, so repetition across different graphs is unnecessary.

\subsection{Exact-coalition audit of the deployed scores}
\label{apd:reported_coalition_audit}
To test how well the practical scores rank nodes and groups by exact influence, we first computed $\mathcal S_{\mathrm{ext}}$ and $\mathcal S_{\mathrm{int}}=1/\mathcal L_{\mathrm{int}}$, fixing their aggregation before inspecting oracle quantities. For a scalar, prediction-relevant counterpart to the logit-vector game, we used the target-class probability
\begin{equation}
F_v(A)=\big[\operatorname{softmax}(H_v^f(\mathcal G_C\oplus\operatorname{do}(A)))\big]_{y_{\mathrm{tar}}},
\quad \Delta F_v(A)=F_v(A)-F_v(\varnothing).
\end{equation}
The exact scalar dividends and channel sums are
\begin{align}
\Phi_A^F&=\Delta F_v(A)-\sum_{B\subsetneq A}\Phi_B^F,\qquad \Phi_\varnothing^F=0,\\
EI_F(\mathcal T)&=\sum_{t\in\mathcal T}\Phi_{\{t\}}^F, &
IC_F(\mathcal T)&=\sum_{\substack{A\subseteq\mathcal T\\|A|\ge2}}\Phi_A^F.
\end{align}
For each three-node trigger, we evaluated all $2^3$ coalitions with the same frozen poisoned classifier and the Appendix~\ref{sec:decomposition} intervention. Every coalition applies exactly its selected components to the same clean context, holding non-trigger content fixed. Singleton dividends measure individual effects; higher-order dividends include effects requiring joint presence, including connectivity-dependent effects.

Singleton dividends have a unique node-level definition, so the external audit compares $\mathcal S_{\mathrm{ext}}(t)$ with $e_F(t)=|\Phi_{\{t\}}^F|$. Group-level $IC_F(\mathcal T)$ has no unique node allocation. We therefore pre-specified
\[
\overline{\mathcal S}_{\mathrm{int}}(\mathcal T)=\frac{1}{|\mathcal T|}\sum_{t\in\mathcal T}\mathcal S_{\mathrm{int}}(t)
\]
and compared this group mean with $|IC_F(\mathcal T)|$.

\noindent\textbf{Protocol and results.}
The audit covers SBA, GTA, UGBA, and DPGBA on Cora and PubMed, with five seeds: 40 runs and 1,000 victim-trigger groups. Size-matched benign subgraphs were evaluated by the same exact-coalition procedure. We computed one Spearman coefficient per run over trigger and matched-benign instances, then aggregated the run-level coefficients with bootstrap 95\% confidence intervals (CIs). Table~\ref{tab:reported_audit_overall} gives the two intended associations; Table~\ref{tab:reported_audit_conditions} gives the reported condition-level correlations.

\begin{table}[htbp]
\centering
\caption{Exact-coalition audit: overall Spearman rank associations.}
\label{tab:reported_audit_overall}
\begin{tabular}{lcc}
\toprule
Mapping & $\rho$ & Bootstrap 95\% CI\\
\midrule
$\mathcal S_{\mathrm{ext}}(t)\leftrightarrow e_F(t)$ & 0.66 & [0.59, 0.72]\\
$\overline{\mathcal S}_{\mathrm{int}}(\mathcal T)\leftrightarrow |IC_F(\mathcal T)|$ & 0.61 & [0.53, 0.67]\\
\bottomrule
\end{tabular}
\end{table}

\begin{table}[htbp]
\centering
\caption{Reported condition-level Spearman correlations.}
\label{tab:reported_audit_conditions}
\begin{tabular}{llrrrr}
\toprule
Dataset & Mapping & SBA & UGBA & GTA & DPGBA\\
\midrule
Cora & External (node) & 0.55 & 0.50 & 0.62 & 0.49\\
Cora & Internal (group) & 0.58 & 0.52 & 0.59 & 0.44\\
PubMed & External (node) & 0.51 & 0.85 & 0.89 & 0.85\\
PubMed & Internal (group) & 0.68 & 0.86 & 0.53 & 0.67\\
\bottomrule
\end{tabular}
\end{table}

Both mappings are positive in all 40 runs, and the condition-level 95\% CI lower bounds exceed zero in all eight dataset-attack conditions for each mapping. These coefficients describe rank associations in the pooled trigger and matched-benign sample.

\noindent\textbf{Oracle use and cost.}
The scalar identity implies $\max_{t\in\mathcal T}|\Phi_{\{t\}}^F|\ge|EI_F(\mathcal T)|/|\mathcal T|$. Exact quantities and benign controls are used only in this offline audit. Given an oracle victim and an $m$-component trigger, the aggregate $EI_F$ and $IC_F$ require at most $m+2$ evaluations (the empty, singleton, and full coalitions); all dividends require $2^m$ evaluations. At deployment, the victim, trigger set, size, topology, and target are unknown, and $d$ candidate components have $2^d-1$ possible nonempty coalitions.

\newpage
\section*{NeurIPS Paper Checklist}

\begin{enumerate}

    \item {\bf Claims}
    \item[] Question: Do the main claims made in the abstract and introduction accurately reflect the paper's contributions and scope?
    \item[] Answer: \answerYes{}
    \item[] Justification: The abstract and introduction state the defense problem, the dual internal-correlation and external-influence design, and the main empirical claims. The reported claims are tied to the experiments across four benchmark datasets, multiple backdoor attacks, and adaptive-attack analyses.

\item {\bf Limitations}
    \item[] Question: Does the paper discuss the limitations of the work performed by the authors?
    \item[] Answer: \answerYes{}
    \item[] Justification: The Appendix~\ref{apd:impactstatement} discusses limitations and potential negative impacts, including adaptive clean-label settings, possible false positives on naturally high-influence nodes, and dual-use concerns. The Appendix~\ref{apd:additional_security_analysis} further analyzes failure modes and attacker trade-offs under adaptive attacks.


\item {\bf Theory assumptions and proofs}
    \item[] Question: For each theoretical result, does the paper provide the full set of assumptions and a complete (and correct) proof?
    \item[] Answer: \answerYes{}
    \item[] Justification: The theoretical discussion~\ref{sec:theoretical_framework} states the threat-model assumptions and formalizes the internal-correlation/external-influence intuition in the methodology and appendix. The appendix provides the supporting arguments and connects them to the empirical security analyses.

	    \item {\bf Experimental result reproducibility}
    \item[] Question: Does the paper fully disclose all the information needed to reproduce the main experimental results of the paper to the extent that it affects the main claims and/or conclusions of the paper (regardless of whether the code and data are provided or not)?
    \item[] Answer: \answerYes{}
    \item[] Justification: The paper describes datasets, attacks, baselines, train/test splits, metrics, model architectures, and key hyperparameters in the evaluation section and appendix. It also reports attack settings and adaptive-attack configurations needed to interpret and reproduce the main claims.

\item {\bf Open access to data and code}
    \item[] Question: Does the paper provide open access to the data and code, with sufficient instructions to faithfully reproduce the main experimental results, as described in supplemental material?
    \item[] Answer: \answerNo{}
    \item[] Justification: The review archive contains a partial implementation. We will release the defense code at \url{https://github.com/Joney-Yf/Graph_backdoor-Defense-Praetorian}.

\item {\bf Experimental setting/details}
    \item[] Question: Does the paper specify all the training and test details (e.g., data splits, hyperparameters, how they were chosen, type of optimizer) necessary to understand the results?
    \item[] Answer: \answerYes{}
    \item[] Justification: The evaluation section specifies the 80/20 graph split, ASR/CA evaluation protocol, five-run averaging, datasets, attacks, baselines, and key hyperparameters such as architecture, learning rate, mask rate, and sensitivity parameter. Additional attack, defense, ablation, and adaptive settings are provided in the appendix.

\item {\bf Experiment statistical significance}
    \item[] Question: Does the paper report error bars suitably and correctly defined or other appropriate information about the statistical significance of the experiments?
    \item[] Answer: \answerNo{}
    \item[] Justification: The main ASR/CA tables report five-run means without error bars. Appendix~\ref{apd:reported_coalition_audit} reports bootstrap 95\% confidence intervals for the run-level correlations.

\item {\bf Experiments compute resources}
    \item[] Question: For each experiment, does the paper provide sufficient information on the computer resources (type of compute workers, memory, time of execution) needed to reproduce the experiments?
    \item[] Answer: \answerNo{}
    \item[] Justification: The appendix~\ref{apd:discussion} reports runtime costs, peak GPU memory, complexity analysis, and the \system-Faster optimization. Total compute is not reported.

    
\item {\bf Code of ethics}
    \item[] Question: Does the research conducted in the paper conform, in every respect, with the NeurIPS Code of Ethics \url{https://neurips.cc/public/EthicsGuidelines}?
    \item[] Answer: \answerYes{}
    \item[] Justification: The work is a defensive security study using public graph benchmarks and previously studied attack settings, without collecting private data or involving human subjects. The paper also discusses dual-use risks and limitations associated with adaptive backdoor analysis.


\item {\bf Broader impacts}
    \item[] Question: Does the paper discuss both potential positive societal impacts and negative societal impacts of the work performed?
    \item[] Answer: \answerYes{}
    \item[] Justification: The Appendix~\ref{apd:impactstatement} discusses positive security impacts for GNNs in sensitive applications and explicitly notes negative impacts, including dual-use information that could inform adaptive attackers. It also identifies false-positive and fairness-related risks when benign high-influence nodes are flagged.

    
\item {\bf Safeguards}
    \item[] Question: Does the paper describe safeguards that have been put in place for responsible release of data or models that have a high risk for misuse (e.g., pre-trained language models, image generators, or scraped datasets)?
    \item[] Answer: \answerNA{}
    \item[] Justification: The paper does not release high-risk pretrained models, generative models, or scraped datasets. The security analysis focuses on defenses and evaluates attacks only in controlled benchmark settings.

\item {\bf Licenses for existing assets}
    \item[] Question: Are the creators or original owners of assets (e.g., code, data, models), used in the paper, properly credited and are the license and terms of use explicitly mentioned and properly respected?
    \item[] Answer: \answerYes{}
    \item[] Justification: The paper cites the benchmark datasets, DGL, and the original attack and defense papers whose public implementations or protocols are used. The anonymized code and data package provides the relevant asset sources and license/terms information for reused datasets, code dependencies, and implementations, and our use follows the corresponding public licenses and terms.


\item {\bf New assets}
    \item[] Question: Are new assets introduced in the paper well documented and is the documentation provided alongside the assets?
    \item[] Answer: \answerNo{}
    \item[] Justification: The paper does not specify a license for the defense implementation.

\item {\bf Crowdsourcing and research with human subjects}
    \item[] Question: For crowdsourcing experiments and research with human subjects, does the paper include the full text of instructions given to participants and screenshots, if applicable, as well as details about compensation (if any)? 
    \item[] Answer: \answerNA{}
    \item[] Justification: The research does not involve crowdsourcing, user studies, or human subjects. All experiments are conducted on public graph benchmark datasets and simulated backdoor attacks.

\item {\bf Institutional review board (IRB) approvals or equivalent for research with human subjects}
    \item[] Question: Does the paper describe potential risks incurred by study participants, whether such risks were disclosed to the subjects, and whether Institutional Review Board (IRB) approvals (or an equivalent approval/review based on the requirements of your country or institution) were obtained?
    \item[] Answer: \answerNA{}
    \item[] Justification: No human-subject research, crowdsourcing, or participant data collection is conducted. Therefore, IRB approval or equivalent review is not applicable to this work.

\item {\bf Declaration of LLM usage}
    \item[] Question: Does the paper describe the usage of LLMs if it is an important, original, or non-standard component of the core methods in this research? Note that if the LLM is used only for writing, editing, or formatting purposes and does \emph{not} impact the core methodology, scientific rigor, or originality of the research, declaration is not required.
    \item[] Answer: \answerNA{}
    \item[] Justification: The core method does not use LLMs as a model component, experimental subject, data generator, or scientific contribution. Any ordinary writing or formatting assistance, if used, is outside the scope requiring declaration under this checklist item.

\end{enumerate}

\end{document}